\documentclass[10pt,twocolumn,journal]{IEEEtran}
\usepackage{cite}
\usepackage[colorlinks=true,linkcolor=black,citecolor=black,urlcolor=black]{hyperref}

\usepackage[mathscr]{eucal}
\usepackage{epsfig,epsf,psfrag}
\usepackage{amssymb,amsmath,amsfonts,latexsym}
\usepackage{graphicx}
\usepackage{epstopdf}
\usepackage{bm,xcolor,url}
\usepackage[caption=false]{subfig} 
\usepackage{fixltx2e}
\usepackage{array}
\usepackage{verbatim}
\usepackage{bm}
\usepackage{algpseudocode, cite}
\usepackage{algorithm}
\usepackage{verbatim}
\usepackage{textcomp}
\usepackage{mathrsfs}
\usepackage{relsize}
\usepackage{subfig}
\usepackage{amsthm}
\usepackage{enumerate}
\usepackage{setspace}
\usepackage{cases}
\usepackage{footnote}
\usepackage{tabularx,tabu}
\usepackage{tablefootnote}
\usepackage{threeparttable}
\usepackage{tikz}
\usetikzlibrary{datavisualization}
\usetikzlibrary{datavisualization.formats.functions}
\usetikzlibrary{decorations.pathreplacing}
\usepackage{textcomp}
\usepackage{multicol}
\usepackage{lipsum}
 
\catcode`~=11 \def\UrlSpecials{\do\~{\kern -.15em\lower .7ex\hbox{~}\kern .04em}} \catcode`~=13 

\allowdisplaybreaks[3]

\newcommand{\iidsim}{\widesim{iid}}
\newcommand{\vecz}{\mathbf{0}}

\newcommand{\norm}[1]{\left\Vert#1\right\Vert}
\newcommand{\normt}[1]{\Vert#1\Vert}

\newcommand{\abs}[1]{\left\lvert#1\right\rvert}
\newcommand{\lrpar}[1]{\left(#1\right)}
\newcommand{\lrang}[1]{\left\langle#1\right\rangle}

\newcommand{\nn}{\nonumber}

\newcommand{\defeq}{\triangleq}
\newcommand{\eqcst}{\stackrel{{\rm c}}{=}}
\newcommand\numberthis{\addtocounter{equation}{1}\tag{\theequation}} 
\newcommand{\prox}{\mathbf{prox}}

\newcommand{\tilell}{\widetilde{\ell}}
\newcommand{\widesim}[2][1.5]{\mathrel{\overset{#2}{\scalebox{#1}[1]{$\sim$}}}}
\newcommand{\convu}{\xrightarrow{{\rm u}}}

\newcommand{\scHN}{\mathscr{H\hspace{-.15cm}N\hspace{-0cm}}}

-lmtk10

\newcommand{\barbA}{\overline{\bf A}}
\newcommand{\barbB}{\overline{\bf B}}

\newcommand{\barbW}{\overline{\bf W}}

\newcommand{\barkappa}{\overline{\kappa}}

\newcommand{\calA}{\mathcal{A}}
\newcommand{\calB}{\mathcal{B}}
\newcommand{\calC}{\mathcal{C}}

\newcommand{\calF}{\mathcal{F}}
\newcommand{\calG}{\mathcal{G}}
\newcommand{\calH}{\mathcal{H}}
\newcommand{\calI}{\mathcal{I}}

\newcommand{\calK}{\mathcal{K}}
\newcommand{\calL}{\mathcal{L}}

\newcommand{\calN}{\mathcal{N}}

\newcommand{\calP}{\mathcal{P}}
\newcommand{\calQ}{\mathcal{Q}}
\newcommand{\calR}{\mathcal{R}}
\newcommand{\calS}{\mathcal{S}}
\newcommand{\calT}{\mathcal{T}}
\newcommand{\calU}{\mathcal{U}}
\newcommand{\calV}{\mathcal{V}}

\newcommand{\calX}{\mathcal{X}}

\newcommand{\calZ}{\mathcal{Z}}

\newcommand{\tilcalF}{\widetilde{\calF}}

\newcommand{\tilcalL}{\widetilde{\calL}}

\newcommand{\tilcalR}{\widetilde{\calR}} 
\newcommand{\tilcalS}{\widetilde{\calS}}

\newcommand{\tilcalV}{\widetilde{\calV}}

\newcommand{\bA}{\mathbf{A}}
\newcommand{\bb}{\mathbf{b}}
\newcommand{\bB}{\mathbf{B}}

\newcommand{\bD}{\mathbf{D}}

\newcommand{\bG}{\mathbf{G}}
\newcommand{\bh}{\mathbf{h}}
\newcommand{\bH}{\mathbf{H}}

\newcommand{\bI}{\mathbf{I}}

\newcommand{\bL}{\mathbf{L}}

\newcommand{\bM}{\mathbf{M}}

\newcommand{\bN}{\mathbf{N}}

\newcommand{\bq}{\mathbf{q}}
\newcommand{\bQ}{\mathbf{Q}}
\newcommand{\br}{\mathbf{r}}
\newcommand{\bR}{\mathbf{R}}

\newcommand{\bS}{\mathbf{S}}

\newcommand{\bu}{\mathbf{u}}
\newcommand{\bU}{\mathbf{U}}
\newcommand{\bv}{\mathbf{v}}
\newcommand{\bV}{\mathbf{V}}

\newcommand{\bW}{\mathbf{W}}
\newcommand{\bx}{\mathbf{x}}
\newcommand{\bX}{\mathbf{X}}
\newcommand{\by}{\mathbf{y}}
\newcommand{\bY}{\mathbf{Y}}
\newcommand{\bz}{\mathbf{z}}

\newcommand{\bbE}{\mathbb{E}}

\newcommand{\bbN}{\mathbb{N}}

\newcommand{\bbP}{\mathbb{P}}

\newcommand{\bbR}{\mathbb{R}}

\newcommand{\scU}{\mathscr{U}}

\DeclareMathAlphabet{\mathbsf}{OT1}{cmss}{bx}{n}
\DeclareMathAlphabet{\mathssf}{OT1}{cmss}{m}{sl}

\DeclareSymbolFont{bsfletters}{OT1}{cmss}{bx}{n}  
\DeclareSymbolFont{ssfletters}{OT1}{cmss}{m}{n}
\DeclareMathSymbol{\bsfGamma}{0}{bsfletters}{'000}
\DeclareMathSymbol{\ssfGamma}{0}{ssfletters}{'000}
\DeclareMathSymbol{\bsfDelta}{0}{bsfletters}{'001}
\DeclareMathSymbol{\ssfDelta}{0}{ssfletters}{'001}
\DeclareMathSymbol{\bsfTheta}{0}{bsfletters}{'002}
\DeclareMathSymbol{\ssfTheta}{0}{ssfletters}{'002}
\DeclareMathSymbol{\bsfLambda}{0}{bsfletters}{'003}
\DeclareMathSymbol{\ssfLambda}{0}{ssfletters}{'003}
\DeclareMathSymbol{\bsfXi}{0}{bsfletters}{'004}
\DeclareMathSymbol{\ssfXi}{0}{ssfletters}{'004}
\DeclareMathSymbol{\bsfPi}{0}{bsfletters}{'005}
\DeclareMathSymbol{\ssfPi}{0}{ssfletters}{'005}
\DeclareMathSymbol{\bsfSigma}{0}{bsfletters}{'006}
\DeclareMathSymbol{\ssfSigma}{0}{ssfletters}{'006}
\DeclareMathSymbol{\bsfUpsilon}{0}{bsfletters}{'007}
\DeclareMathSymbol{\ssfUpsilon}{0}{ssfletters}{'007}
\DeclareMathSymbol{\bsfPhi}{0}{bsfletters}{'010}
\DeclareMathSymbol{\ssfPhi}{0}{ssfletters}{'010}
\DeclareMathSymbol{\bsfPsi}{0}{bsfletters}{'011}
\DeclareMathSymbol{\ssfPsi}{0}{ssfletters}{'011}
\DeclareMathSymbol{\bsfOmega}{0}{bsfletters}{'012}
\DeclareMathSymbol{\ssfOmega}{0}{ssfletters}{'012}

\newcommand{\tild}{\widetilde{d}}

\newcommand{\hatf}{\widehat{f}}

\newcommand{\tilf}{\widetilde{f}}

\newcommand{\tilg}{\widetilde{g}}

\newcommand{\hatbH}{\widehat{\bH}}

\newcommand{\tilL}{\widetilde{L}}

\newcommand{\hatbR}{\widehat{\bR}}

\newcommand{\hatbW}{\widehat{\bW}}

\newcommand{\barf}{\overline{f}}
\newcommand{\barg}{\overline{g}}
\newcommand{\barh}{\overline{h}}

\newcommand{\balpha}{\bm{\alpha}}
\newcommand{\bbeta}{\bm{\beta}}

\newcommand{\bPsi}{\bm{\Psi}}

\newcommand{\tmeta}{\widetilde{\eta}}
\newcommand{\tkappa}{\widetilde{\kappa}}

\newcommand{\trho}{\widetilde{\rho}}
\newcommand{\tnu}{\widetilde{\nu}}

\newcommand{\iid}{i.i.d.\ }

\newcommand{\convas}{\xrightarrow{\mathrm{a.s.}}}

\newcommand{\floor}[1]{\lfloor{#1}\rfloor}
\newcommand{\lrangle}[2]{\left\langle{#1},{#2}\right\rangle}

\DeclareMathOperator*{\argmax}{arg\,max}
\DeclareMathOperator*{\argmin}{arg\,min}

\DeclareMathOperator{\st}{s.t.\;}

\DeclareMathOperator{\sgn}{sgn}
\DeclareMathOperator{\tr}{tr}

\newcommand{\bone}{\mathbf{1}}

\newtheorem{theorem}{Theorem} 
\newtheorem*{theorem*}{Theorem}
\newtheorem{lemma}{Lemma}

\newtheorem*{assump}{Assumptions}

\theoremstyle{definition}
\newtheorem{definition}{Definition} 

\theoremstyle{remark}
\newtheorem{remark}{Remark}

\newenvironment{psketch}{\noindent{\em Proof Sketch.}\hspace*{1em}}{\qed\bigskip\\}

\newcommand{\qednew}{\nobreak \ifvmode \relax \else
      \ifdim\lastskip<1.5em \hskip-\lastskip
      \hskip1.5em plus0em minus0.5em \fi \nobreak
      \vrule height0.75em width0.5em depth0.25em\fi}

\graphicspath{{../figures/}}
\newcommand{\rmth}{\mathrm{th}}

\begin{document}

\title{Online Nonnegative Matrix Factorization with Outliers}

\author{Renbo~Zhao,~\IEEEmembership{Member,~IEEE,}
        and~Vincent~Y.~F.~Tan,~\IEEEmembership{Senior~Member,~IEEE}
        
\thanks{A preliminary work has been published in ICASSP 2016 \cite{Zhao_16}.} 

\thanks{The authors  are with the Department of Electrical and Computer Engineering and the Department of Mathematics, National University of Singapore. They are supported in part by the NUS Young Investigator Award (grant number R-263-000-B37-133).}
}


\maketitle

\begin{abstract}
We propose a unified and systematic framework for performing online nonnegative matrix factorization in the presence of outliers. Our framework is particularly suited to large-scale data.  We propose two solvers  based on   projected gradient descent and  the  alternating direction method of multipliers. We prove that the sequence of objective values converges almost surely by appealing to the quasi-martingale convergence theorem. We also show the sequence of learned dictionaries converges to the set of stationary points of the expected loss function almost surely.  
In addition, we extend our basic problem formulation to various settings with different constraints and regularizers. We also adapt the solvers and analyses to each setting. 
We perform extensive experiments on both synthetic and real datasets. These experiments demonstrate the computational efficiency and efficacy of our algorithms  on tasks such as (parts-based) basis learning, image denoising, shadow removal and foreground-background separation. 

\end{abstract}


\section{Introduction}\label{sec:intro}
In recent years, Nonnegative Matrix Factorization (NMF) has become a popular dimensionality reduction \cite{Tsuge_01} technique, due to its parts-based, non-subtractive  interpretation of the learned basis \cite{Lee_99}. Given a nonnegative data matrix $\bV$, it seeks to approximately decompose $\bV$ into two nonnegative matrices, $\bW$ and $\bH$, such that $\bV\approx\bW\bH$. In the literature, the fidelity of such a approximation is most commonly measured by $\norm{\bV-\bW\bH}_F^2$ \cite{Lee_00,Kim_08a,Lin_07a,Kim_08b,Xu_12}. 
To obtain this approximation, many algorithms have been proposed, including multiplicative updates \cite{Lee_00}, block principal pivoting \cite{Kim_08a}, projected gradient descent \cite{Lin_07a}, active set method\cite{Kim_08b}, and the alternating direction method of multipliers \cite{Xu_12}. These algorithms have promising performances in numerous applications, including document clustering \cite{Ding_05}, hyperspectral unmixing \cite{Yuan_15} and audio source separation \cite{Durrieu_11}. However, there are also many studies \cite{Cao_07,Zhang_11} showing that their performances deteriorate under two common and practical scenarios. The first scenario is when the data matrix $\bV$ has a large number of columns (data samples). This situation arises in today's data-rich environment. {\em Batch} data processing methods used in  the aforementioned algorithms become  highly inefficient  in terms of  the computational time and storage space.  
The second scenario is the existence of outliers in some of the data samples. For example (e.g.), there are  glares and shadows in images   due to bad illumination conditions. Another example is the presence of  impulse noises in time series,  including speech recordings in natural language  processing or temperature recordings in weather forecasting. The outliers, if not handled properly, can significantly corrupt the learned basis matrix, thus the underlying low-dimensional data structure cannot be learned reliably. Moreover, outlier detection and pruning become much more difficult in large-scale datasets~\cite{RAD_15}. As such, it is imperative to design   algorithms that can learn interpretable parts-based basis representations from large-scale datasets whilst being robust to possible outliers.

\subsection{Previous Works} \label{sec:prev_work}
Many efforts have been devoted to address each challenge separately. To handle large datasets,  researchers have pursued solutions in three main directions. The first class of algorithms proposed is know as {\em online NMF} algorithms \cite{Cao_07,Bucak_08,Guan_12,Lef_11,WangFei_11,WuShen_14}. These algorithms aim to refine the basis matrix each time a new data sample is acquired without storing the past data samples. The second class of algorithms is known as {\em distributed NMF}\cite{Liu_10,Gemulla_11,Du_14,Chen_15}. The basic idea behind these algorithms is to distribute the data samples over a network of agents so that several small-scale optimization problems can be performed concurrently. 
The final class of algorithms is called the {\em compressed NMF} algorithms \cite{Halko_11,Tepper_16}. These algorithms perform structured random compression to 
project the data onto the lower-dimensional manifolds. As such, the size of the dataset can be reduced. 
These three approaches have successfully reduced the computation and storage complexities---either provably or through numerical experiments.  
For the existence of outliers, a class of algorithms called the (batch) {\em robust NMF} \cite{Gao_15, Huang_14, Yang_13, Du_12, Ding_12, Ding_11, Cichoc_11, Kasi_12,Fev_14,Shen_14,Zhang_11} has been proposed to reliably learn the basis matrix by minimizing the effects of the outliers. The robustness against the outliers is achieved via different approaches. These are detailed in Section~\ref{sec:RNMF}. However, to the best of our knowledge, so far there are no NMF-based algorithms that are able to systematically handle  outliers in large-scale datasets.

\subsection{Main Contributions}
In this paper, we propose an algorithm called the {\em online NMF with outliers} that fills this void. Specifically, our algorithm aims to learn the basis matrix $\bW$ in an online manner whilst being robust to outliers.   
The development of the proposed algorithm involves much more than the straightforward combination or adaptation of online NMF and robust NMF algorithms. Indeed, since there are many ways to ``robustify'' the NMF algorithms, it is crucial to find an appropriate way to incorporate such robustness guarantees into the online algorithms. Our algorithm proceeds as follows. At each time instant, we solve two optimization problems. The first enables us to learn the coefficient and outlier vectors while the second enables us to update the basis matrix. 
We propose two solvers  based on projected gradient descent (PGD) and alternating direction method of multipliers (ADMM) to solve both optimization problems. 
 Moreover, the presence of outliers also results in more difficulty when we analyze the convergence properties of  our  algorithms. See Section~\ref{sec:conv_dis} for a detailed discussion. We remark that in recent years, some algorithms of similar flavors have been proposed, e.g., online robust PCA \cite{Feng_13,Shen_14b} and online robust dictionary learning \cite{Wang_13}. However, due to different problem formulations, our algorithm  has many distinctive features, which then calls for different techniques to develop the solvers and analyze the  convergence properties. Furthermore, in Section~\ref{sec:numericals}, we also observe its superior performance on real-world applications, including (parts-based) basis learning, image denoising, shadow removal and foreground-background separation, over the similar algorithms.  In sum, our contributions are threefold:
\begin{enumerate}
\item We develop two different  solvers  based on PGD and ADMM  to solve the optimization problems  at each time instant.  These two solvers can be easily extended to two novel solvers for the batch robust NMF problem. The theoretical and empirical performances of both solvers are compared and contrasted. 
\item Assuming the data are independently drawn from some common distribution $\bbP$, we prove the almost sure convergence of the sequence of objective values  as well as the almost sure convergence of 
the sequence of learned basis matrices to the set of stationary points of the expected loss function.
The proof techniques involve the use of tools from convex analysis \cite{Rock_70} and empirical process theory\cite{Vaart_00}, as well as the quasi-martingale convergence theorem \cite[Theorem 9.4 \& Proposition 9.5]{Met_82}.
\item   We extend the basic problem setting to various other general settings, by altering the constraint sets and adding regularizers. We also indicate how to adapt our solvers and analyses to each case. By doing so, the applicability of our algorithms is greatly generalized. 
\end{enumerate}

\subsection{Paper Organization}\label{sec:org}
This paper is organized as follows. We first provide a  more detailed literature survey in Section~\ref{sec:priorart}. Next we state a formal formulation of our problem in Section~\ref{sec:prob}. The algorithms are derived in Section~\ref{sec:algo} and their convergence properties are analyzed in Section~\ref{sec:conv_analysis}. In Section~\ref{sec:extension}, we extend our basic problem formulations to a wide variety of settings, and indicate how the solvers and analyses can be adapted to each setting. Finally in Section~\ref{sec:numericals}, we provide extensive experiment results on both synthetic and real data. The results are compared to  those of their batch counterparts and other online matrix factorization algorithms.   We conclude the paper in Section \ref{sec:con} stating some promising avenues for further investigations.

In this paper, all the lemmas and sections with indices beginning with `S' will appear in the supplemental material. 


\subsection{Notations}
In the following, we use capital boldface letters to denote matrices. For example, the (updated) dictionary/basis matrix at time $t$ is denoted by $\bW_t$.  
We use $F$ and $K$ to denote the ambient dimension and the (known) latent dimension of data respectively. 
We use lower-case boldface letters to denote vectors. Specifically, at time instant $t$, we denote the acquired sample vector, learned coefficient vector and outlier vector as $\bv_t$, $\bh_t$ and $\br_t$ respectively. For a vector $\bx$, its $i$-th entry is denoted by $x_i$. 
Given a matrix $\bX$, we denote its $i$-th row as $\bX_{i:}$, $j$-th column as $\bX_{:j}$ and $(i,j)$-th entry by $x_{i,j}$. Moreover, we denote its Frobenius norm by $\norm{\bX}_F$, spectral norm by $\norm{\bX}_2$, $\ell_{1,1}$ norm by $\norm{\bX}_{1,1} \defeq \sum_{i,j} \abs{x_{i,j}}$ 
and trace by $\tr(\bX)$. 
Inequality $\bx\geq 0$ or $\bX \geq 0$ denotes entry-wise nonnegativity. We use $\lrangle{\cdot}{\cdot}$ to denote the Frobenius inner product between two matrices and $\bone$ the vector with all entries equal to one. For a closed convex nonempty set $\calA$, 
we denote $\calP_\calA$ as the Euclidean projector onto $\calA$. In particular, $\calP_+$ denotes the Euclidean projector onto the nonnegative orthant.  Also, the $\infty$-indicator function of $\calA$, $I_\calA$ is defined as
\begin{equation}
I_\calA(x) \defeq \begin{cases}
0, &x\in\calA\\
\infty, & x\not\in\calA
\end{cases}.
\end{equation}
In particular, $I_+$ denotes the $\infty$-indicator function of the nonnegative orthant. 
For $n\in\bbN$, $[n]:=\{1,2,\ldots,n\}$. Also, $\bbR_+$ denotes the set of nonnegative real numbers. 

\section{Related Works}\label{sec:priorart}
\subsection{Robust NMF}\label{sec:RNMF}
The canonical NMF problem can be stated as the following minimization problem
\begin{equation}
\min_{\bW\in\calC, \{\bh_i\}_{i=1}^N\ge 0} \quad\frac{1}{N}\sum_{i=1}^N \frac{1}{2}\norm{\bv_i-\bW\bh_i}_2^2, \label{eq:canon_NMF}  
\end{equation}
where $\calC\subseteq\bbR_+^{F\times K}$ denotes the constraint set for $\bW$ and $N$ denotes the number of data samples. In many works~\cite{Paa_94,Lee_00,Hoyer_04}, $\calC$ is set to $\bbR_+^{F\times K}$. For simplicity, we omit   regularizers on $\bW$ and $\{\bh_i\}_{i=1}^N$ at this point. 
Since algorithms for the canonical NMF perform unsatisfactorily when the data contain outliers, robust NMF algorithms have been proposed.  Previous algorithms for   robust NMF fall into two categories. The first category~\cite{Gao_15, Huang_14, Yang_13, Du_12, Ding_12, Ding_11, Cichoc_11, Kasi_12} replaces the (half) squared $\ell_2$ loss $1/2\norm{\bv_i-\bW\bh_i}_2^2$ in \eqref{eq:canon_NMF} with some other robust loss measure $\psi(\cdot\Vert\cdot)$
\begin{equation}
\min_{\bW\in\calC, \{\bh_i\}_{i=1}^N\ge 0} \quad \frac{1}{N}\sum_{i=1}^N  \psi(\bv_i\Vert\bW\bh_i). \label{eq:robust1}
\end{equation}
For example, $\psi(\bv_i\Vert\bW\bh_i)$ can be the $\ell_2$ norm $\norm{\bv_i-\bW\bh_i}_2$ \cite{Ding_11} or the $\ell_1$ norm $\norm{\bv_i-\bW\bh_i}_1$ \cite{Kasi_12}.   
The second category \cite{Fev_14,Shen_14,Zhang_11} retains the squared $\ell_2$ loss but explicitly models the outlier vectors $\{\br_i\}_{i=1}^N$. Specifically, \eqref{eq:canon_NMF} is reformulated as  
\begin{align}
&\min \quad \frac{1}{N}\sum_{i=1}^N \frac{1}{2}\norm{\bv_i-\bW\bh_i-\br_i}_2^2 + \lambda\, \phi(\bR) \nn\\
&\st \quad \bW\in\calC, \{\bh_i\}_{i=1}^N\ge 0, \bR\in\calQ\label{eq:robust2}
\end{align}
where $\lambda\ge 0$ is the regularization parameter, $\bR = [\br_1,\ldots,\br_N]$ is the outlier matrix, $\phi(\bR)$ is the regularizer on $\bR$ and 
$\calQ$ is the feasible set of $\bR$. Depending on the assumed sparsity structure of $\bR$, $\phi(\bR)$ can be the $\ell_{2,1}$ norm \cite{Fev_14}, $\ell_{1,2}$ norm\cite{Shen_14} or $\ell_{1,1}$ norm \cite{Zhang_11} of $\bR$. 
Robust NMF algorithms typically do not admit strong recovery guarantees of the original data matrix 
since neither \eqref{eq:robust1} nor \eqref{eq:robust2} are convex programs. However, as shown empirically, the estimated basis matrix $\hatbW$ represents meaningful parts of the data and the residues of the outliers in the reconstructed matrix $\hatbW\hatbH$ are very small \cite{Shen_14,Fev_14,Zhao_16}. 


\subsection{Online Matrix Factorization}
Existing algorithms on online matrix factorization belong to two distinct categories.  The first category of algorithms \cite{Guan_12,Mairal_10,Feng_13,Shen_14b,WangFei_11} assumes 
the data samples $\{\bv_t\}_{t\ge 1}$ are generated independently from a time-invariant distribution $\bbP$. Under this assumption, it is possible to provide theoretical guarantees on the convergence of the online stochastic algorithms by leveraging the empirical process theory, 
as was done  in \cite{Mairal_10,Feng_13,Guan_12,Shen_14b}. These methods have extensive applications, including document clustering \cite{WangFei_11}, image inpainting~\cite{Mairal_10}, face recognition~\cite{Guan_12}, and image annotation \cite{Guan_12}. The second category of algorithms  \cite{WuShen_14,Cao_07,Bucak_08,Lef_11,WangDong_13,Wang_13,Xing_13,Zhang_15,Zhang_15b}, with major applications in visual tracking, assumes 
that the data generation distribution $\bbP$ is time-varying. Although these assumptions are weaker than those in the first class of algorithms, it is very difficult to provide theoretical guarantees on the convergence of the online algorithms. 

\subsection{Online Low-rank and Sparse Modeling}\label{sec:OL_lowrank}
Another related line of works \cite{Zhan_16,Lois_15,Guo_15,Spre_12,Spre_12b} aims to recover the low-rank data matrix $\bL$ and the sparse outlier matrix $\bS$ from their additive mixture $\bM$ in an online fashion. Among these works, \cite{Zhan_16,Lois_15,Guo_15} assume that the sequence of mixture vectors (columns of $\bM$) arrives in a streaming fashion. The authors derive the recovery algorithms based on their proposed models of the sequence of ground-truth data vectors (columns of $\bL$) and outlier vectors (columns of $\bS$). 
The authors of \cite{Spre_12,Spre_12b} adopt a different approach. They allow the number of mixture vectors to be large but finite and use an alternating minimization approach to solve a variant of \eqref{eq:robust2} (with additional Tikhonov regularizers on $\bW$ and $\{\bh_i\}_{i=1}^N$). In learning the coefficient vectors $\{\bh_i\}_{i=1}^N$ (with fixed $\bW$), the authors employ the stochastic gradient descent method, thus ensuring that their algorithms are scalable to large-scale data.

\section{Problem Formulation}\label{sec:prob}
Following Section~\ref{sec:RNMF}, in this work we explicitly model the outlier vectors as $\{\br_t\}_{t\ge 1}$. 
Also, we assume the data generation distribution $\bbP$ is time-invariant, for ease of the convergence analysis. 
First, for a fixed data sample $\bv$ and a fixed basis matrix $\bW$, define the loss function with respect to (w.r.t.) $\bv$ and $\bW$, $\ell(\bv,\bW)$ as
\begin{align}
\ell(\bv,\bW) \defeq \min_{\bh\ge 0,\br\in\calR} \tilell(\bv,\bW, \bh, \br),\label{eq:ith_loss}
\end{align}
where 
\begin{equation}
\tilell(\bv,\bW, \bh, \br)\defeq\frac{1}{2}\norm{\bv-\bW\bh-\br}_2^2 +\lambda\norm{\br}_1, \label{eq:ith_loss_tilde}
\end{equation}
and $\calR \defeq \{\br\in\bbR^F\,|\,\norm{\br}_\infty\le  M\}$ is the constraint set of the outlier vector $\br$. Here we use the $\ell_1$ regularization to promote entrywise sparsity on $\br$. 

Next, given a finite set of data samples $\{\bv_i\}_{i\in[t]}\iidsim\bbP$,  we define the {\em empirical loss} associated with $\{\bv_i\}_{i\in[t]}$, $f_t(\bW)$ as 
\begin{align}
f_t(\bW) \defeq  \frac{1}{t}\sum_{i=1}^t\ell(\bv_i,\bW) \label{eq:emp_loss}
\end{align}
where $\calC \defeq \left\{\bW\in\bbR_+^{F\times K}\,|\,\norm{\bW_{:i}}_2\le 1, \forall\,i\in[K]\right\}$. 

Following the convention of the online learning literature \cite{Tsypkin_71,Bottou_98,Bottou_08}, instead of minimizing the empirical loss $f_t(\bW)$ in \eqref{eq:emp_loss}, we aim to minimize the {\em expected loss} $f(\bW)$, i.e., 
\begin{equation}
\min_{\bW\in\calC} \left[f(\bW) \defeq \bbE_{\bv\sim\bbP}[\ell(\bv,\bW)]\right].  \label{eq:exp_loss}
\end{equation}
In other words, we aim to solve a (non-convex) stochastic program \cite{Shap_07}. Note that by the strong law of large numbers, given any  $\bW\in\calC$, we have
$f_t(\bW) \convas f(\bW)$ as $t\to\infty$.

\begin{remark}
We make three remarks here. First we explain the reasonings behind the choice of the constraint sets $\calC$ and $\calR$.  The set $\calC$ constrains the columns of $\bW$ in the unit (nonnegative) $\ell_2$ ball. This is to prevent the entries of $\bW$ from being unbounded, following the conventions in \cite{Lin_07a,Lin_07b}. The set $\calR$ uniformly bounds the entries of $\br$. This is because in practice, both the underlying data (without outliers) and the observed data are uniformly bounded entrywise. Since we do not require $\br$ to be exactly recovered, this prior information can often improve the estimation of $\br$. For real data, the bound $M>0$ can often be easily chosen. For example, for gray-scale images, $M$ can be chosen as $2^m-1$, where $m$ is the number of bits per pixel. In the case of matrices containing ratings from users with a maximum rating of $\upsilon$, $M$ can be chosen as $\upsilon$. In some scenarios where $M$ is difficult to estimate, we simply set $M=\infty$. As will be shown in Sections~\ref{sec:algo} and \ref{sec:conv_analysis}, the algorithms and analyses developed for finite $M$ can be easily adapted to infinite $M$. 
Second, for the sake of brevity, we omit regularizing $\bW$ and $\bh$. Such regularizations, together with other possible constraint sets of $\bW$ and $\br$ will be discussed in Section~\ref{sec:extension}.
Third, we assume the ambient data dimension $F$, the latent data dimension $K$ and the penalty parameter $\lambda$ in \eqref{eq:ith_loss} are time-invariant parameters.\footnote{In this work, we do not simultaneously consider the data with high ambient dimensions. An attempt on this problem in the context of dictionary learning with the squared-$\ell_2$ loss has been made in \cite{Mensch_16}.}  
\end{remark}

\section{Algorithms}\label{sec:algo}
To tackle the problem proposed in Section~\ref{sec:prob}, we leverage the  {\em stochastic majorization-minimization (MM)} framework\cite{Mairal_13,Raza_16}, which has been widely used in previous works on online matrix factorization \cite{Mairal_10,Feng_13,Shen_14b,WuShen_14,Cao_07,Lef_11,Bucak_08,WangDong_13,Wang_13}. In essence, such framework decomposes the optimization problem in \eqref{eq:exp_loss} into two steps, namely {\em nonnegative encoding} and {\em dictionary update}. Concretely, at a time instant $t$ $(t\ge 1)$, we first learn the coefficient vector $\bh_t$ and the outlier vector $\br_t$ based on the newly acquired data sample $\bv_t$ and the previous dictionary matrix $\bW_{t-1}$.  Specifically, we solve the following convex optimization problem
\begin{equation}
(\bh_t,\br_t) = \argmin_{\bh\geq 0,\br\in\calR}\;\; \tilell(\bv_t,\bW_{t-1}, \bh, \br).\label{eq:min_hr}
\end{equation}
Here the initial basis matrix $\bW_0$ is randomly chosen in $\calC$. 
Next, based on the past statistics $\{\bv_i,\bh_i,\br_i\}_{i\in[t]}$, the basis matrix is updated to 
\begin{align}
\bW_t = \argmin_{\bW\in\calC} \tilf_t(\bW),\label{eq:min_W}
\end{align}
where 
\begin{equation}
\tilf_t(\bW)\defeq\frac{1}{t}\sum_{i=1}^t \frac{1}{2}\norm{\bv_i-\bW\bh_i-\br_i}_2^2+\lambda\norm{\br_i}_1. \label{eq:tilde_f_t}
\end{equation}
We note that \eqref{eq:min_W} can be rewritten as
\begin{align}
\bW_t = \argmin_{\bW\in\calC} \frac{1}{2}\tr\left(\bW^T\bW\bA_t\right) - \tr\left(\bW^T\bB_t\right), \label{eq:min_W_tr}
\end{align}
where $\bA_t\defeq1/t\sum_{i=1}^t \bh_i\bh_i^T$ and $\bB_t \defeq 1/t\sum_{i=1}^t (\bv_i-\br_i)\bh_i^T$ are the sufficient statistics. From \eqref{eq:min_W_tr}, we observe that our algorithm has a storage complexity independent of $t$ since only $\bA_t$, $\bB_t$ and $\bW_t$ need to be stored and updated.


To solve \eqref{eq:min_hr} and \eqref{eq:min_W_tr} ({\em at a fixed time instant $t$}), we propose two solvers based on PGD and ADMM respectively. For ease of reference, we refer to the former algorithm as OPGD and the latter as OADMM.
We now explain the motivations behind proposing these two solvers. Since both \eqref{eq:min_hr} and \eqref{eq:min_W_tr} are constrained optimization problems, the most straightforward solver would be based on PGD. Although such a solver has a linear computational complexity per iteration, it typically needs a large number of iterations to converge. Moreover, it is also easily trapped in bad local minima \cite{Sun_14}. Thus, it would be meaningful to contrast its performance with a solver with very different properties. This leads us to propose another solver based on ADMM. Such a solver has a higher computational complexity per iteration but typically needs fewer number of iterations to converge \cite{Xu_12}. It is also less susceptible to bad local minima since it solves optimization problems in the dual space. In Section~\ref{sec:numericals}, we will show that the practical performances of these two solvers are comparable despite the different properties they possess. Thus either solver can be used for most practical purposes. 

In the sequel, we omit the time subscript $t$ to keep notations uncluttered. 
In the iterations, the updated value of a variable is denoted with the superscript `$+$'. 
Pseudo-codes of the entire algorithm (for $N$ data samples) are provided in Algorithm~\ref{algo:generic}. 

\subsection{Online Algorithm Based on PGD (OPGD)}\label{sec:OPGD}
\subsubsection{PGD solver for \eqref{eq:min_hr}} \label{sec:algoMM_minhr}

For a fixed $\bW$, we solve \eqref{eq:min_hr} by alternating between the following two steps
\begin{align}
&\bh^+ := \argmin_{\bh'\ge 0} Q_\eta(\bh'\vert \bh),\label{eq:min_h}\\
&\br^+ := \argmin_{\br'\in\calR} \; \frac{1}{2}\norm{\bv-\bW\bh^+-\br'}_2^2 + \lambda\norm{\br'}_1,\label{eq:min_r}
\end{align}
where\footnote{At time $t$, the value of $L$ is computed based on $\bW_{t-1}$, i.e., $L_t \defeq \norm{\bW_{t-1}}_2^2$.}  
\begin{equation}
Q_\eta(\bh'\vert \bh) \defeq q(\bh) + \lrangle{\nabla q(\bh)}{\bh'-\bh} + \frac{1}{2\eta} \norm{\bh'-\bh}_2^2,\nn
\end{equation}
$q(\bh) \defeq \frac{1}{2}\norm{\bv-\bW\bh-\br}_2^2$ and  $\eta\in\left(0,1/L\right]$ with $L\defeq \norm{\bW}_2^2$. The steps \eqref{eq:min_h} and \eqref{eq:min_r} can be interpreted based on the framework of {\em block MM} \cite{Raza_13,Hong_16}. Specifically, it is easy to verify that the steps \eqref{eq:min_h} and \eqref{eq:min_r} amount to finding the (unique) minimizers of the majorant functions\footnote{For a function $g$ with domain $\calG$, its majorant at $\kappa\in\calG$, $\tilg$ is the function that satisfies i) $\tilg\ge g$ on $\calG$ and ii) $\tilg(\kappa) = g(\kappa)$.} of $\bh'\mapsto\tilell(\bv,\bW,\bh',\br)$ at $\bh$ and $\br'\mapsto\tilell(\bv,\bW,\bh^+,\br')$ at $\br$ respectively. Therefore, the convergence analysis in \cite{Raza_13} guarantees that such alternating minimization procedure converges to a global optimum of \eqref{eq:min_hr}. 

In addition, we notice that the minimizations in both \eqref{eq:min_h} and \eqref{eq:min_r} have closed-form solutions. For \eqref{eq:min_h}, the solution is given by the PGD update step (with constant step size)
\begin{equation}
\bh^+ := \calP_+(\bh - \eta\nabla q(\bh)).
\end{equation}
For ease of parameter tuning, we rewrite $\eta = \barkappa/L$ ($0<\barkappa\le 1$). Furthermore, we fix $\eta$ (or $\barkappa$) throughout all iterations.

For \eqref{eq:min_r}, if $M=\infty$, the solution is precisely given by  
\begin{equation}
\br^+ := \calS_{\lambda} (\bv-\bW\bh^+),
\end{equation}
where $\calS_{\lambda}$ is the (elementwise) soft-thresholding operator threshold $\lambda$. Otherwise, when $M$ is finite, using \cite[Lemma~5]{ZhangBox_15} (see Lemma~\ref{lem:L1_box}), we have 
\begin{equation}
\br^+ := \widetilde{\calS}_{\lambda,M} (\bv-\bW\bh^+),
\end{equation}
where for any $\bx\in\bbR^F$ and $i\in[F]$,
\begin{equation*}
\left(\widetilde{\calS}_{\lambda,M}(\bx)\right)_i :=\begin{cases}
0, &|x_i|<\lambda\\
x_i - \sgn(x_i)\lambda, & \lambda \le |x_i| \le \lambda+M\\
\sgn(x_i)M, &|x_i| > \lambda+M
\end{cases}.
\end{equation*}


\subsubsection{PGD solver for \eqref{eq:min_W_tr}} \label{sec:algoMM_minW}
Similar to the procedure for solving \eqref{eq:min_h}, we first rewrite \eqref{eq:min_W_tr} as 
\begin{align}
&\min_\bW p_t(\bW) + I_{\calC}(\bW), \mbox{where}\nn\\
& \qquad p_t(\bW) = \frac{1}{2}\tr\left(\bW^T\bW\bA_t\right) - \tr\left(\bW^T\bB_t\right).
\end{align}
First, it is easy to see $p_t$ is convex and differentiable, $\nabla p_t$ is Lipschitz with constant $\tilL_t\defeq\norm{\bA_t}_F$. Thus, we can construct a majorant function $P_t(\bW'|\bW)$ for $p_t(\bW')$ at $\bW\in\calC$
\begin{align}
&P_t(\bW'|\bW) = p_t(\bW) + \lrangle{\nabla p_t(\bW)}{\bW'-\bW}\nn \\
&\hspace{4cm}   + \frac{1}{2\tmeta_t}\norm{\bW'-\bW}_F^2,
\end{align}
where $\nabla p_t(\bW) = \bW\bA_t-\bB_t$ and $\tmeta_t\in(0,1/\tilL_t]$. 
Minimizing $P_t(\bW'|\bW)+I_\calC(\bW')$ over $\bW'\in\calC$, we have
\begin{align}
\bW^+ &:= \argmin_{\bW'\in\calC} \norm{\bW'-(\bW-\tmeta_t\nabla p_t(\bW))}_F^2\\
&:= \calP_{\calC} (\bW-\tmeta_t\nabla p_t(\bW)). \label{eq:proj_W}
\end{align}
By Lemma~\ref{lem:proj_nonnegL2}, the projection step in \eqref{eq:proj_W} is given by  
\begin{equation}
\bW^+_{:j} := \frac{\calP_+(\bW-\tmeta_t\nabla p_t(\bW))_{:j}}{\max\{1,\norm{\calP_+(\bW-\tmeta_t\nabla p_t(\bW))_{:j}}_2\}}, \;\forall\,j\in [K].
\end{equation}
Again, for each iteration given in \eqref{eq:proj_W}, we use the same step size $\tmeta_t = \tkappa_t/L$ where $0<\tkappa_t\le 1$.

\begin{remark}\label{rmk:acc_grad}
In the literature\cite{Tseng_08,Nest_13}, the accelerated proximal gradient descent (APGD) method  has been proposed to accelerate the canonical PGD method. With an additional extrapolation step in each iteration, the convergence rate can be improved from $O(1/k)$ to $O(1/k^2)$, where $k$ denotes the number of iterations. However, since in our implementations, both \eqref{eq:min_hr} and \eqref{eq:min_W_tr} were only solved to a prescribed accuracy (see Section~\ref{sec:stop_crit}), we observed  no significant reduction in running times on the real tasks. Thus for simplicity, we only employ the canonical PGD method. 
\end{remark}

\subsection{Online Algorithm Based on ADMM (OADMM)}\label{sec:ADMM}
Below we only present the update rules derived via ADMM. The detailed derivation steps are shown in Section~\ref{sec:deriv_ADMM}. 
\subsubsection{ADMM solver for \eqref{eq:min_hr}} \label{sec:algoADMM_minhr}
First we reformulate \eqref{eq:min_hr} as
\begin{align}
&\min_{\bh,\br}\;\; \frac{1}{2}\norm{\bv-\bW\bh-\br}_2^2 + \lambda\norm{\br}_1 + I_+(\bu) + I_\calR(\bq)\nn\\
&\st \;\; \bh = \bu, \br = \bq 
\end{align}
Thus the augmented Lagrangian is 
\begin{align}
&\calL(\bh,\br,\bu,\bq,\balpha,\bbeta) = \frac{1}{2}\norm{\bv-\bW\bh-\br}_2^2 + \lambda\norm{\br}_1  \nn\\
&+ I_+(\bu) + I_\calR(\bq) +\balpha^T(\bh-\bu) + \bbeta^T(\br-\bq) \nn\\
&+\frac{\rho_1}{2}\norm{\bh-\bu}^2_2 + \frac{\rho_2}{2}\norm{\br-\bq}_2^2, 
\end{align}
where $\balpha$, $\bbeta$ are dual variables and $\rho_1$, $\rho_2$ are positive penalty parameters. 
Then we sequentially update each of $\bh$, $\br$, $\bu$, $\bq$, $\balpha$ and $\bbeta$ while keeping other variables fixed 
\begin{align}
\bh^+ &:= (\bW^T\bW+\rho_1\bI)^{-1}\left(\bW^T(\bv-\br) + \rho_1\bu-\balpha\right)\label{eq:ADMMupd_h}\\
\br^+ &:= \calS_\lambda(\rho_2\bq+\bv-\bbeta-\bW\bh)/(1+\rho_2)\\
\bu^+ &:= \calP_+\left(\bh^+ +\balpha/\rho_1\right)\\
\bq^+ &:= \calP_\calR\left(\br^+ +\bbeta/\rho_2\right)\\
\balpha^+ &:= \balpha + \rho_1(\bh^+-\bu^+)\\
\bbeta^+ &:= \bbeta + \rho_2(\br^+-\bq^+).
\end{align}

\subsubsection{ADMM solver for \eqref{eq:min_W_tr}} \label{sec:algoADMM_minW}
Again, we rewrite \eqref{eq:min_W_tr} as
\begin{align}
&\min \;\; \frac{1}{2}\tr\left(\bW^T\bW\bA_t\right) - \tr\left(\bW^T\bB_t\right) + I_\calC(\bQ)\nn\\
&\st \;\; \bW=\bQ 
\end{align}
The augmented Lagrangian in this case is
\begin{align}
&\calL(\bW,\bQ,\bD) = \frac{1}{2}\tr\left(\bW^T\bW\bA_t\right) - \tr\left(\bW^T\bB_t\right) + I_\calC(\bQ)\nn\\
&\hspace{2cm}+ \lrangle{\bD}{\bW-\bQ} + \frac{\rho_3}{2}\norm{\bW-\bQ}_F^2,
\end{align}
where $\bD$ is the dual variable and $\rho_3$ is a positive penalty parameter. 
Minimizing $\bW$, $\bQ$ and $\bD$ sequentially yields 
\begin{align}
\bW^+ 
&:= \left(\bB_t-\bD+\rho_3\bQ\right)\left(\bA_t+\rho_3\bI\right)^{-1}\label{eq:ADMMupd_W}\\
\bQ^+ &:=  \calP_\calC\left(\bW^++\bD/\rho_3\right)\label{eq:ADMMupd_Q}\\
\bD^+ &:=  \bD + \rho_3\left(\bW^+-\bQ^+\right).
\end{align}
\begin{remark}\label{rmk:algo}
A few comments are now in order: First, although  algorithms based on both PGD and ADMM exist in the literature for the canonical NMF problem  (see for example, \cite{Lin_07a,Sun_14,Xu_12}), our problem setting is different from those in the previous works. Specifically, our problem explicitly models the outlier vector $\br$ and the constraint set on $\bW$ is more complicated than the non-negative orthant $\bbR_+^{F\times K}$. 
Moreover, for the algorithm based on PGD, we use a fixed step size though all the iterations. In contrast, the usual projected gradient method applied to NMF \cite{Lin_07a} involves  computationally intensive Armijo-type line searches for the step sizes.
Second, although updates \eqref{eq:ADMMupd_h} and \eqref{eq:ADMMupd_W} involve matrix inversions, the operations do not incur high computational complexity since both matrices to be inverted have sizes $K\times K$ where $K \ll F$. 
Third, the proposed two solvers can be easily extended to solve the batch NMF problem with outliers. Thus, we have also effectively proposed {\em two new solvers} for the (batch) robust NMF problem. These extensions are detailed in Section~\ref{sec:BRNMF_algo}. 
In the sequel we term these two batch algorithms as BPGD and BADMM respectively. 
Finally, the stopping criteria of both OPGD and OADMM for solving \eqref{eq:min_hr} and \eqref{eq:min_W_tr} will be described in Section~\ref{sec:stop_crit}. 
\end{remark}

\begin{algorithm}[t]
\caption{Online NMF with outliers (ONMFO)}\label{algo:generic}
\begin{algorithmic} 
\State {\bf Input}: Data samples $\{\bv_i\}_{i\in[N]}$, penalty parameter $\lambda$, initial dictionary matrix $\bW_0$
\State {\bf Initialize} sufficient statistics: $\bA_0 := \vecz$, $\bB_0 := \vecz$
\State {\bf for} $t$ = 1 to $N$ {\bf do}
\State \quad 1) Acquire a data sample $\bv_t$.
\State \quad 2) Learn the coefficient vector $\bh_t$ and the outlier vector $\br_t$ based on $\bW_{t-1}$, 
using the solvers based on PGD or ADMM  (detailed  in Sections~\ref{sec:algoMM_minhr} and~\ref{sec:algoADMM_minhr})
\begin{align}
&\hspace{-.3cm}\left(\bh_t, \br_t\right) := \argmin_{\bh\geq 0,\br\in\calR} \tilell(\bv_t,\bW_{t-1}, \bh, \br).\label{eq:learn_hr}
\end{align} 
\State \quad 3) Update the sufficient statistics 
\begin{align*}
\bA_t &:= {1}/{t}\left\{(t-1)\bA_{t-1} + \bh_t\bh_t^T\right\}, \\
\bB_t &:= 1/t\left\{(t-1)\bB_{t-1} + (\bv_t-\br_t)\bh_t^T\right\}.
\end{align*}
\State \quad 4) Learn the dictionary matrix $\bW_t$ based on $\bA_t$ and $\bB_t$, using the solvers based on PGD or ADMM  (detailed in Section~\ref{sec:algoMM_minW} and \ref{sec:algoADMM_minW})
\begin{equation}
\bW_t := \argmin_{\bW\in\calC}\frac{1}{2}\tr\left(\bW^T\bW\bA_t\right) - \tr\left(\bW^T\bB_t\right) \label{eq:dict_upd}
\end{equation}
\State {\bf end for}
\State {\bf Output}: Final dictionary matrix $\bW_N$
\end{algorithmic}
\end{algorithm}

\section{Convergence Analyses}\label{sec:conv_analysis}
In this section, we analyze the convergence of both the sequence of objective values $\{f(\bW_t)\}_{t\ge 1}$ (Theorem~\ref{thm:as_conv}) and the sequence of dictionary matrices $\{\bW_t\}_{t\ge 1}$ (Theorem~\ref{thm:local_min}) produced by Algorithm~\ref{algo:generic}. It is worth noticing that the analyses are independent of the specific solvers used in steps 2) and 4) in Algorithm~\ref{algo:generic}. 
Indeed, in our analyses, we only leverage the fact that both \eqref{eq:min_hr} and \eqref{eq:min_W_tr} can be exactly solved. 

\subsection{Preliminaries}\label{sec:prelim}
First, given any $\bh'\ge 0$ and $\br'\in\calR$, we observe that  $(\bv,\bW)\mapsto\tilell(\bv,\bW, \bh', \br')$ and $\tilf_t$ serve as upper-bound functions for $\ell$ (on $\bbR_+^F\times\calC$) and $f_t$ (on $\calC$) respectively. 
Denote $(\bh^*,\br^*)$ as an optimal solution of \eqref{eq:ith_loss}. (Note that $(\bh^*,\br^*)$ is a function of $\bv$ and $\bW$ and always exists since $\tilell(\bv',\bW', \bh, \br)$ is closed and coercive on $\bbR_+^K\times \calR$ for any $\bv'\in\bbR_+^F$ and $\bW'\in\calC$.) Clearly we have $\tilell(\bv,\bW, \bh^*, \br^*) = \ell(\bv,\bW)$, for any $(\bv,\bW)\in\bbR_+^F\times\calC$. 

Next, we notice that $\ell$ can be equivalently defined as 
\begin{equation}
\ell(\bv,\bW) = \min_{(\bh,\br)\in\calH\times\calR} 
\tilell(\bv,\bW, \bh, \br), \label{eq:new_ithloss}
\end{equation}
where $\calH$ is a compact and convex set in $\bbR^K_+$. This is because $\bh^*$ in \eqref{eq:ith_loss} is bounded, due to the boundedness of $\bv_i$ and $\bW$. Thus it suffices to consider minimizing $\bh$ over a compact set in $\bbR^K_+$. We can further choose $\calH$ to be convex. If $M=\infty$, we can similarly show $\br^*$ is bounded thus still consider minimizing $\br$ over a compact and convex set $\calR$ in $\bbR^F$. 

We make three assumptions for our  subsequent analyses.
\begin{assump}\quad
\begin{enumerate}
\item The data generation distribution $\bbP$ has a compact support set $\calV$. \label{assum:comp_supp}
\item For all $(\bv,\bW)\in\calV\times\calC$, $\tilell(\bv,\bW, \bh, \br)$ is {\em jointly} $m_1$-strongly convex in $(\bh,\br)$ for some constant $m_1>0$.\label{assum:sc_hr}
\item For all $t\ge 1$, $\tilf_t(\bW)$ is $m_2$-strongly convex on $\calC$ for some constant $m_2>0$. \label{assum:sc_W}
\end{enumerate}
\end{assump}
\begin{remark}\label{rmk:assump}
All of the assumptions above are made to simplify the convergence analyses. 
Among them, Assumption~\ref{assum:comp_supp} naturally holds for real data, which are always bounded. 
Assumptions~\ref{assum:sc_hr} and \ref{assum:sc_W} play roles in proving Lemma~\ref{lem:regularity} and \ref{lem:rate_W} respectively. Apart from this, they have no effects in proving our main theorems (i.e., Theorem \ref{thm:as_conv} and \ref{thm:local_min}). These two assumptions hold by simply adding strongly convex regularizers to $\tilell(\bv,\bW, \bh, \br)$ or $\tilf_t(\bW)$. For example, we can add Tikhonov regularizer $(\nu_1\norm{\bh}_2^2 + \nu_2\norm{\br}_2^2)/2$ (for some positive $\nu_1,\nu_2$) to $\tilell(\bv,\bW, \bh, \br)$, then in such case $m_1 = \min(\nu_1,\nu_2)$. Adding such regularizers can be regarded as a way to promote smoothness and avoid over-fitting on $\bW$, $\bh$ and $\br$ (see Section~\ref{sec:extension}). Moreover, including the regularizers in $\tilell(\bv,\bW, \bh, \br)$ or $\tilf_t(\bW)$ will not alter any steps in our analyses. 
In terms of our algorithms, for small regularization weights (e.g., $\nu_1,\nu_2\ll 1$), such regularizers will only slightly alter steps \eqref{eq:learn_hr} and \eqref{eq:dict_upd} in Algorithm~\ref{algo:generic}. For example, in \eqref{eq:learn_hr}, $\tilell(\bv_t,\bW_{t-1}, \bh, \br)$ will become
\begin{align}
&\tilell'(\bv_t,\bW_{t-1}, \bh, \br) \defeq \frac{1}{2}\norm{\bv-\bW\bh-\br}_2^2 +\frac{\nu_1}{2}\norm{\bh}_2^2 \nn\\
&\hspace{4.5cm}+\lambda\norm{\br}_1 + \frac{\nu_2}{2}\norm{\br}_2^2.
\end{align}
Thus for simplicity we omit such regularizers in the algorithm. However, both of our solvers (OPGD and OADMM) can be straightforwardly adapted to the regularized case. 
\end{remark}

Finally, we define the notion of {\em stationary point} of a constrained optimization problem. 
\begin{definition}\label{def:stat_pt}
Given a real Hilbert space $\calX$, a differentiable function $g:\calX\to\bbR$  and a set $\calK\subseteq\calX$, $x_0\in\calK$ is a stationary point of $\min_{x\in\calK}g(x)$ if $\lrangle{\nabla g(x_0)}{x-x_0}\ge 0$, for all $x\in\calK$.
\end{definition}

\subsection{Main Results and Key Lemmas}
In this section we present our main  results (Theorem~\ref{thm:as_conv} and \ref{thm:local_min}) 
and key lemmas to prove these results (Lemma~\ref{lem:regularity} and \ref{lem:rate_W}). We only show proof sketches here. Detailed proofs are deferred to 
Section~\ref{sec:regularity} to \ref{sec:Lips_b} in the supplemental material. 


\begin{theorem}[Almost sure convergence of the sequence of objective values]\label{thm:as_conv}
In Algorithm~\ref{algo:generic}, the nonnegative stochastic process $\{\tilf_t(\bW_t)\}_{t\ge 1}$ converges a.s.. Furthermore, $\{f(\bW_t)\}_{t\ge 1}$ converges to the same almost sure limit as $\{\tilf_t(\bW_t)\}_{t\ge 1}$. 
\end{theorem}
\begin{psketch}
The proof proceeds in two major steps. First, making use of the quasi-martingale convergence theorem \cite[Theorem 9.4 \& Proposition 9.5]{Met_82} (see Lemma~\ref{lem:quasi_conv}), we prove $\{\tilf_t(\bW_t)\}_{t\ge 1}$ converges a.s.\ by showing the series of the expected positive variations of this process converges. A key step in proving the convergence of this series  is to bound each summand of the series by Donsker's theorem (see Lemma~\ref{lem:Donsker}). To invoke this theorem, we show the class of measurable functions $\{\ell(\cdot,\bW):\bW\in\calC\}$ is $\bbP$-Donsker \cite{Vaart_00} using the Lipschitz continuity of $\ell(\bv,\cdot)$ on $\calC$ (see Lemma~\ref{lem:regularity}) and \cite[Example 19.7]{Vaart_00} (see Lemma~\ref{lem:suff_Donsker}). 
Second, we show 
\begin{equation}
f_t(\bW_t)-\tilf_t(\bW_t)\convas 0.\label{eq:as_conv_diff}
\end{equation}
by showing $\sum_{t=1}^\infty \frac{\tilf_t(\bW_t)-f_t(\bW_t)}{t+1}$ converges a.s.. Define $b_t \defeq \tilf_t(\bW_t)-f_t(\bW_t)$. Using the Lipschitz continuity of both $\tilf_t$ and $f_t$ on $\calC$ and Lemma~\ref{lem:rate_W}, we can show $\abs{b_{t+1} - b_t} = O(1/t)$ a.s.. Now we invoke \cite[Lemma 8]{Mairal_10} (see Lemma~\ref{lem:conv_nonneg}) to conclude $f_t(\bW_t)-\tilf_t(\bW_t)\convas 0$. 
By Lemma~\ref{lem:regularity} and \cite[Example 19.7]{Vaart_00} (see Lemma~\ref{lem:suff_Donsker}), we can show the class of measurable functions $\{\ell(\cdot,\bW):\bW\in\calC\}$ is $\bbP$-Glivenko-Cantelli \cite{Vaart_00}. By the Glivenko-Cantelli theorem (see Lemma~\ref{lem:GlivenkoCantelli}) we have 
\begin{equation}
\sup_{\bW\in\calC} \abs{f_t(\bW) - f(\bW)} \convas 0. \label{eq:Gliv_Cantelli}
\end{equation}
In particular, 
\begin{equation}
f_t(\bW_t) - f(\bW_t)\convas 0. \label{eq:as_conv_diff2}
\end{equation}
Finally, \eqref{eq:as_conv_diff} and \eqref{eq:as_conv_diff2} together imply that  $\{f(\bW_t)\}_{t\ge 1}$ converges to the same almost sure limit as $\{\tilf_t(\bW_t)\}_{t\ge 1}$.
\end{psketch}

\begin{theorem}[Almost sure convergence of the sequence of dictionaries]\label{thm:local_min}
The stochastic  process $\{\bW_t\}_{t\ge 1}$ converges to the set of stationary points of \eqref{eq:exp_loss} a.s..\footnote{Given a metric space $(\calX,d)$, a sequence $(x_n)$ in $\calX$ is said to converge to a set $\calA\subseteq\calX$ if $\lim_{n\to\infty}\inf_{a\in\calA}d(x_n,a) = 0$.}
\end{theorem}
\begin{psketch}
Fix a realization $\{\bv_t\}_{t\ge 1}$ such that $\tilf_t(\bW_t) - f_t(\bW_t)\to 0$,\footnote{The set of all such realizations have probability one by \eqref{eq:as_conv_diff}.} and generate $\{\bW_t\}_{t\ge 1}$ according to Algorithm~\ref{algo:generic}. Then it suffices to show every subsequential limit of $\{\bW_t\}_{t\ge 1}$ is a stationary point of \eqref{eq:exp_loss}. 
The compactness of $\calC$ enables us to find a convergent subsequence $\{\bW_{t_m}\}_{m\ge 1}$. Furthermore, it is possible to find a further subsequence of $\{t_m\}_{m\ge 1}$, $\{t_k\}_{k\ge 1}$ such that all $\{\bA_{t_k}\}_{k\ge 1}$, $\{\bB_{t_k}\}_{k\ge 1}$ and $\{\tilf_{t_k}(\vecz)\}_{k\ge 1}$ converge.  We focus on $\{\bW_{t_k}\}_{k\ge 1}$ hereafter and denote its limit by $\barbW$. Our goal is to show  for any $\bW\in\calC$, the directional derivative $\lrangle{\nabla f(\barbW)}{\bW - \barbW}\ge 0$. First, we show the sequence of differentiable functions $\{\tilf_{t_k}\}_{k\ge 1}$ converges uniformly to a differentiable function $\barf$. By \eqref{eq:Gliv_Cantelli} we also have that $\{f_{t_k}\}_{k\ge 1}$ converges uniformly to $f$. Denote $g_t \defeq \tilf_t - f_t$, we have $\{g_{t_k}\}_{k\ge 1}$ converges uniformly to $\barg = \barf - f$. By Lemma~\ref{lem:regularity}, $f$ is differentiable so $\barg$ is differentiable on $\calC$. Next, we show for any $\bW\in\calC$, $\lrangle{\nabla \barf(\barbW)}{\bW - \barbW}\ge 0$ by showing $\barbW$ is a global minimizer of $\barf$. Then we show $\nabla \barg(\barbW) = \vecz$ using the first-order Taylor expansion. These results suffice to imply $\lrangle{\nabla f(\barbW)}{\bW - \barbW}\ge 0$ for any $\bW\in\calC$. 
\end{psketch}

\begin{remark}\label{rmk:nonconvexity}
Due to the nonconvexity of the canonical NMF problem, finding global optima of 
 \eqref{eq:canon_NMF} is in general out-of-reach. Indeed, \cite{Vavasis_09} shows \eqref{eq:canon_NMF} is NP-hard. 
 Therefore in the literature \cite{Lin_07b,Haji_16}, convergence to the set of stationary points (see Definition~\ref{def:stat_pt}) of \eqref{eq:canon_NMF}  has been studied instead. In the online setting, it is even harder to  show that the sequence of dictionaries $\{\bW_t\}_{t\ge 1}$ converges to the global optima of \eqref{eq:exp_loss} (in some probabilistic sense). Thus Theorem~\ref{thm:local_min} only states the convergence result with respect to the stationary point of \eqref{eq:exp_loss}, which has been the state-of-the-art in the literature of online matrix factorization \cite{Mairal_10,Shen_14b}. On the other hand, we notice that under certain assumptions on the data matrix $\bV$, e.g., the separability condition proposed in \cite{Donoho_04}, exact NMF algorithms have been proposed in~\cite{Cohen_93,Arora_12,Gillis_12}. The optimization techniques therein are discrete (and combinatorial) in nature, and vastly differ from the continuous optimization techniques typically employed in solving the canonical NMF problem. In addition, some heuristics for obtaining exact NMF using continuous optimization techniques have been proposed in \cite{Vanda_16}. Nevertheless, developing exact NMF algorithms in the online setting could be an interesting future research direction.
\end{remark}

The following two lemmas are used in the proof of Theorem~\ref{thm:as_conv} and \ref{thm:local_min}. In particular, Lemma~\ref{lem:regularity} states that the loss functions $\ell$ and $f$ respectively defined in \eqref{eq:ith_loss} and \eqref{eq:exp_loss} satisfy some regularity conditions.
Lemma~\ref{lem:rate_W} then bounds the variations in the stochastic process $\{\bW_t\}_{t\ge 1}$. 

\begin{lemma}[Regularity properties  of $\ell$ and $f$]\label{lem:regularity}\quad
The loss functions $\ell$ and $f$ satisfy the following properties\\
1) $\ell(\bv,\bW)$ is differentiable on $\calV\times\calC$ and $\nabla_\bW\ell(\bv,\bW)$ is continuous on $\calV\times \calC$. Moreover, for all $\bv\in\calV$, $\nabla_\bW\ell(\bv,\bW)$ is Lipschitz on $\calC$ with Lipschitz constant independent of $\bv$. \\
2) The expected loss function $f(\bW)$ is differentiable on $\calC$ and $\nabla f(\bW)$ is Lipschitz on $\calC$.
\end{lemma}
\begin{psketch}
The regularity properties of $\ell$ on $\calV\times\calC$ originate  from the regularity properties  of $\tilell$ on $\calV\times\calC$. Since Assumption~\ref{assum:sc_hr} ensures the minimizer for \eqref{eq:new_ithloss}, $(\bh^*(\bv,\bW),\br^*(\bv,\bW))$ is unique for any $(\bv,\bW)\in\calV\times\calC$, we can invoke Danskin's theorem (see Lemma~\ref{lem:min_diff}) to guarantee the differentiability of $\ell$. Moreover, we invoke the maximum theorem (see Lemma~\ref{lem:min_cont}) to ensure the continuity of  $(\bh^*(\bv,\bW),\br^*(\bv,\bW))$ on $\calV\times\calC$. These results together imply that $\ell$ satisfies the desired regularity properties. The regularity properties of $f$ on $\calC$ hinges upon the regularity properties  of $\ell$. Indeed, by  Leibniz integral rule (see Lemma~\ref{lem:Leib_int}), we can show $f$ is differentiable and 
\begin{equation}
\nabla f(\bW) = \bbE_\bv[\nabla_\bW \ell(\bv,\bW)]. \label{eq:grad_f}
\end{equation} 
Hence the Lipschitz continuity of $\nabla f$ follows naturally from the Lipschitz continuity of $\bW\mapsto\nabla_\bW\ell(\bv,\bW)$. 
\end{psketch}

\begin{lemma} \label{lem:rate_W}
In Algorithm~\ref{algo:generic}, $\norm{\bW_{t+1}-\bW_t}_F = O\left(1/t\right)$ a.s..
\end{lemma}
\begin{psketch}
The upper bound for $\norm{\bW_{t+1}-\bW_t}_F$ results from the strong convexity and Lipschitz continuity of $\tilf_t$. Specifically, they together imply an order difference between the upper and lower bounds of $\tilf_t(\bW_{t+1}) - \tilf_t(\bW_t)$, i.e., 
\begin{align}
\frac{m_2}{2}\norm{\bW_{t+1}-\bW_t}_F^2& \le \tilf_t(\bW_{t+1}) - \tilf_t(\bW_t) \\
& \le c_t \norm{\bW_{t+1}-\bW_t}_F, 
\end{align}
where $c_t$ is only dependent on $t$. Some calculations reveal that $c_t = O(1/t)$ a.s.. 
\end{psketch}


\subsection{Discussions}\label{sec:conv_dis}
We highlight the distinctions between our convergence analyses and those in the previous works. 
Our first main result (Theorem~\ref{thm:as_conv}) is somewhat analogous to the results in  \cite{Mairal_10,Feng_13,Shen_14b,Guan_12}. 
However, the stochastic optimization of \eqref{eq:exp_loss} in \cite{Guan_12} is based on   robust stochastic approximations \cite{Nemi_09}, a very different approach from ours (see Section~\ref{sec:algo}). This leads to significant differences in the analyses. For example, at time $t$, their dictionary update step does not necessarily minimize $\tilf_t(\bW)$, a fact that we heavily leverage. The rest of the works fall under a similar framework as ours. However, in \cite{Mairal_10}, a general sparse coding problem is considered. Thus, they assume a sufficient condition ensuring a unique solution in the Lasso-like sparse coding step holds. This assumption adds certain complications in proving $\nabla f$ is Lipschitz. Because our problem setting is different, we avoid these issues. Thus we are able to provide a more succinct proof, despite having some additional features, e.g., the presence of the outlier vector $\br$. The most closely related works are~\cite{Feng_13} and \cite{Shen_14b}, both of which consider the online robust PCA problems but with different  loss functions. 
However, the nonnegativity constraints on $\bW$ and $\bh$ and box constraints on $\br$ distinguish our proof from theirs. 
For our second main result (Theorem~\ref{thm:local_min}), we note that in   previous works 
\cite{Feng_13,Shen_14b}, seemingly stronger results have been shown. 
Specifically, these works manage to show $\{\bW_t\}_{t\ge 1}$ converges, either to {\em one} stationary point \cite{Shen_14b} or to the global minimizer \cite{Feng_13} of $f$.
However, due to different problem settings, their proof techniques cannot be easily adapted to our problem setting. Inspired by \cite{Mairal_13}, we instead manage to characterize the subsequential limits of $\{\bW_t\}_{t\ge 1}$ generated by almost sure realizations of $\{\bv_t\}_{t\ge 1}$. This result is almost as good as the results in the abovementioned previous works since it implies that there exists a stationary point of $f$ such that its neighborhood contains infinitely many points of the sequence almost surely. 

\section{Extensions} \label{sec:extension}
In Section~\ref{sec:prob}, we considered a basic formulation of the online NMF problem with outliers. Here, 
we extend this formulation to a wider class of problem settings, with different constraints and regularizers on $\bW$, $\bh$ and $\br$. We also discuss how to adapt our algorithms and analyses to these settings. 

\subsection{Different Constraints on $\bW$ and $\br$}
If the constraint set $\calC$ changes, in terms of both solvers in Section~\ref{sec:algo},   
only \eqref{eq:proj_W} and \eqref{eq:ADMMupd_Q}, the updates involving projections onto $\calC$, need to be changed. In terms of the analyses, using a similar argument as in Section~\ref{sec:prelim}, we can show the (unique) optimal solution $\bW^*(\bA_t,\bB_t)$ in \eqref{eq:min_W_tr} is bounded\footnote{From \eqref{eq:min_W_tr}, we observe that $\bW_t$ is a function of only $\bA_t$ and $\bB_t$ if Assumption~\ref{assum:sc_W} holds.} regardless of boundedness of $\calC$. Thus, the optimization problem \eqref{eq:min_W_tr} can always be restricted to a compact set. The rest of the analyses proceeds as usual. Therefore, in the following, we only discuss the (convex) variants of $\calC$ and the associated (Euclidean) projection methods. 

1) The nonnegative orthant  $\calC_1 \defeq \bbR_+^{F\times K}$. This constraint is the most basic yet the most widely used constraint in NMF, although it does not well control the scale of $\bW$. The projection onto the nonnegative orthant involves simple entrywise thresholding. 

2) The probability simplex $\calC_2 \defeq \{\bW\in\bbR_+^{F\times K}\,|\,\norm{\bW_{:i}}_1= 1, \forall\,i\in[K]\}$ \cite{Guan_12}. Different from $\calC$, $\calC_2$ also prevents some columns of $\bW$ from having arbitrarily small norms. Efficient projection algorithms onto the probability simplex have been well studied. See \cite{Duchi_08,WangSimp_13} for details. 

3) The nonnegative elastic-net ball \cite{Witten_09,Zou_05,Sind_12}:  $\calC_3 \defeq \{\bW\in\bbR_+^{F\times K}|\gamma_1\norm{\bW_{:i}}_1 + \gamma_2/2\norm{\bW_{:i}}_2^2 \le 1, \forall\,i\in[K]\}$, where both $\gamma_1$ and $\gamma_2$ are {\em nonnegative}. $\calC_3$ is a general constraint set since it subsumes both $\calC$ (the nonnegative $\ell_2$ ball) and the nonnegative $\ell_1$ ball. Compared to $\calC$, this constraint encourages sparsity on the basis matrix $\bW$, thus leading to better interpretability on the basis vectors. Efficient algorithms for projection onto the elastic-net ball have been proposed in \cite{Duchi_09,Mairal_10}.

For the outlier vector $\br$, the nonnegativity constraint can be added to $\calR$ to model (bounded) nonnegative outliers, so the new constraint $\calR'\defeq\{\br\in\bbR_+^{F}\,|\,\norm{\br}_\infty\le M\}$. In such case $\norm{\br}_1 = \bone^T\br$ so \eqref{eq:min_r} amounts to a quadratic minimization program with box constraints. 
Both the algorithms and the analyses can be easily adapted to such a simpler case. 

\subsection{Regularizers for $\bW$, $\bh$ and $\br$} 
We use $\lambda_i\ge 0$ to denote the penalty parameters. 
For $\bW$, the elastic-net regularization can be employed, i.e., $\lambda_1\norm{\bW}_{1,1} + \lambda_2/2\norm{\bW}_F^2$. 
This includes   $\ell_1$ and $\ell_2$ regularizers on $\bW$ as special cases. As pointed out in \cite{WuShen_14,Kim_12} and \cite{Witten_09}, $\ell_1$ and $\ell_2$ regularizers promote sparsity and smoothness on $\bW$ respectively. Because  $\bW\ge 0$, \eqref{eq:min_W_tr} with the elastic-net regularization is still quadratic. In terms of the analyses, Assumption~\ref{assum:sc_W} can be removed and the rest remains the same. 

For $\bh$, several regularizers can be used:

1) The 
Lasso regularizer $\lambda_3 \norm{\bh}_1$. 
It induces sparsity on $\bh$, hence the  optimization problem \eqref{eq:min_h} with such regularizer is termed {\em nonnegative sparse coding}\cite{Sind_12,Wang_13}. 

2) The Tikhonov regularizer: $\lambda_4/2 \norm{\bh}_2^2$. This regularizer induces smoothness and avoids over-fitting on $\bh$. Both the $\ell_1$ and $\ell_2$ regularizers preserve the quadratic nature of \eqref{eq:min_h}.

3) The group Lasso regularizer $\lambda_5 \sum_{\alpha} \sqrt{\xi_\alpha}\norm{\bh_\alpha}_2$, where $\bh_\alpha$'s are (non-overlapping) subvectors of $\bh$ with lengths $\xi_\alpha$. This regularizer induces sparsity among groups of entries in $\bh$. Efficient algorithms to solve \eqref{eq:min_h} under this regularization have been proposed in \cite{Liu_11,Kim_12}. 

4) The 
 sparse group Lasso regularizer $\lambda_6 (\nu\sum_{\alpha} \sqrt{\xi_\alpha}\norm{\bh_\alpha}_2$ $ + (1-\nu)\norm{\bh}_1)$, where $\nu\in[0,1]$. Compared with the group lasso regularizer, this one also encourages sparsity within each group. Efficient algorithms for solving \eqref{eq:min_h} with this regularization have been  discussed in \cite{Friedman_10,Simon_13}.

Since \eqref{eq:min_h} with each regularizer above is still a convex program, the analyses remain unchanged. 
Similar regularizers on $\bh$ can be applied to $\br$, so for simplicity we omit the discussions here. Compared to $\bh$, the only differences are that $\br$ may be negative and bounded. However, standard algorithms in such case are available in the literature \cite{Yuan_06,Simon_13,Friedman_10}. 

\begin{remark} \label{rmk:cons}
Some remarks are in order. First, for certain pairs of constraints and regularizers above (e.g., the elastic-net ball constraint of $\bW$ and the elastic-net regularization on $\bW$), the optimization problems are equivalent (admit the same optimal solutions) for specific value pairs of $(\gamma_i,\lambda_i)$. However, the algorithms and analyses for these equivalent problems can be much different. Second, all the alternative constraint sets and regularizers discussed above are convex, since our analyses heavily leverage the (strong) convexity of $\tilell$ and $\tilf$ (see Section~\ref{sec:prelim}). It would be a meaningful future research direction to consider   nonconvex constraints or regularizers.
\end{remark}

\section{Experiments} \label{sec:numericals}

We  present results of numerical experiments to demonstrate the computational efficiency and efficacy of our online algorithms (OPGD and OADMM). We first state the experimental setup. Next we introduce some heuristics used in our experiments and the choices of parameters. We show the efficiency of our algorithms by examining the convergence speeds of the objective functions. The efficacy of our algorithms is shown via the (parts-based) basis representations and three meaningful real-world  applications---image denoising, shadow removal and foreground-background separation.  
All the experiments were run in 64-bit $\mbox{Matlab}^\circledR$ (R2015b) on a machine with $\mbox{Intel}^\circledR$ $\mbox{Core}$ i7-4790 3.6 GHz CPU and 8 GB RAM. 

\subsection{Experimental Setup}
The datasets used in the experiments include one synthetic dataset, two face datasets (including the {\tt CBCL} dataset \cite{Lee_00} and the {\tt YaleB} dataset \cite{YaleB_01}) and two video sequences in the {\tt i2r} dataset\cite{i2r_04}. 
For each task, we compared the performances of our online algorithms against those of four other matrix factorization algorithms, including BPGD, BADMM, online robust PCA (ORPCA) \cite{Feng_13} and online NMF (ONMF) \cite{Guan_12}. Such comparisons allow us to show multiple advantages of the proposed algorithms. Specifically, compared with their batch counterparts (BPGD and BADMM), our online algorithms have much faster convergence speeds in terms of the objective values.\footnote{\label{fn:online_batch} The problem formulations in the existing batch robust NMF algorithms \cite{Gao_15, Huang_14, Yang_13, Du_12, Ding_12, Ding_11, Cichoc_11, Fev_14,Shen_14,Zhang_11} are different from ours here. Specifically, most of the algorithms do not use the $\ell_{1,1}$ regularization to remove the effect of outliers. Moreover, the constraints on $\bW$ and $\bR$ are simpler in their formulations. All these differences contribute to different results in the experiments. 
Thus for the purpose of fair comparisons, we chose to compare our online algorithms with their batch counterparts.}
 Moreover, compared with ORPCA, 
we are able to learn parts-based basis representations. Finally, compared with ONMF, we can recover the underlying low-dimensional data structure with minimal effects of outliers. 

\subsection{Strategies in Practical Implementations} \label{sec:heuristics}
\subsubsection{Initializations} \label{sec:init}
For both online algorithms (OPGD and OADMM), the entries of initial dictionary $\bW_0$ were randomly generated independently and identically from the standard uniform distribution $\scU[0,1]$. This distribution can certainly be of other forms, e.g.,  the half-normal distribution\footnote{$\scHN\left(\sigma^2\right)$ denotes the half-normal distribution with scale parameter $\sigma^2$, i.e., $\scHN(y ; \sigma^2)=\frac{\sqrt{2}}{\sigma\sqrt{\pi}}\exp(-\frac{y^2}{2\sigma^2})$ for $y\ge 0$. If $Y\sim\calN(0,\sigma^2)$, $|Y|$ is half-normal with scale parameter $\sigma^2$.} $\scHN(0,1)$. However, we observed in all the experiments that different initialization schemes of $\bW_0$ led to similar results. At each time instant, similar initialization methods on $\bh$ and $\br$ were also used in solving \eqref{eq:learn_hr}. 
While for solving \eqref{eq:dict_upd}, we use $\bW_{t-1}$ as the initialization to exploit the similarities between dictionaries in adjacent iterations. This approach led to improved computational efficiency in practice. 
For the batch algorithms (BPGD and BADMM), the initial dictionary and coefficient matrix were similarly initialized as $\bW_0$.

\subsubsection{Stopping Criteria}\label{sec:stop_crit} 
For the optimization problem \eqref{eq:learn_hr},  we employed the same stopping criterion for both algorithms OPGD and OADMM. Specifically, at time $t$, we stopped \eqref{eq:learn_hr} at iteration $k$ if 
\begin{equation}
\frac{\abs{\tilell(\bv_t,\bW_{t-1},\bh^k,\br^k) - \tilell(\bv_t,\bW_{t-1},\bh^{k-1},\br^{k-1})}}{\tilell(\bv_t,\bW_{t-1},\bh^{k-1},\br^{k-1})}< \epsilon_{\rmth}\nn
\end{equation}
or $k>\varkappa_{\max}$, where $\epsilon_{\rmth}>0$ denotes the threshold for the relative decrease of objective values and $\varkappa_{\max}\in\bbN$  denotes the maximum number of iterations. We set $\epsilon_{\rmth}=1\times 10^{-3}$ and $\varkappa_{\max}=50$ for all the experiments. Similar stopping criterion applies to both  OPGD and OADMM in solving \eqref{eq:dict_upd}, except in this case, we set $\epsilon_{\rmth}=1\times 10^{-4}$ and $\varkappa_{\max}=200$ for better precision in updating the dictionary. 

\subsubsection{Heuristics}
We now describe some heuristics employed in the implementations of our online algorithms. We remark that all the heuristics used in this work are common in the literature of online matrix factorization \cite{Mairal_10,Bucak_08}.
First, since the sample size of the {\em benchmarking} datasets in real world may not be large enough, 
we replicate it $p$ times to aggregate its size. 
Each (aggregated) dataset is represented as an $F\times N$ matrix, where $F$ equals the ambient dimension of the data\footnote{Each image (or video frame) is vectorized and stacked as a column of the data matrix, so $F$ equals the number of pixels in the image (or video frame).}  and $N$ equals $p$ times the size of the original dataset.
Next, for the aggregated image dataset, we randomly shuffle the images in it to simulate an \iid data stream. We also normalize each data sample so that it has unit maximum pixel value. 
The last heuristic is called the mini-batch extension, namely at a time instant, we process $\tau$ images  altogether ($\tau\ge 1$). This amounts to sampling $\tau$ \iid samples from $\bbP$ at each time. 

\subsection{Parameter Settings} 
We   discuss how to set the values of the important parameters in our online algorithms. 
These parameters include the regularization parameter $\lambda$ in \eqref{eq:ith_loss}, the constraint parameter $M$ in the definition of $\calR$, the mini-batch size $\tau$, the latent data dimensionality $K$,  the penalty parameters $\{\rho_i\}_{i\in[3]}$ 
in OADMM, and the step sizes $\barkappa$ and $\{\tkappa_t\}_{t\ge 1}$ in OPGD. First, we set $\lambda = 1/\sqrt{F}$, following the convention in the literature \cite{Mairal_10,Feng_13,Shen_14b}. Since each image has unit maximum pixel value, it is straightforward to set $M=1$. 
For the mini-batch size $\tau$, we propose a rule-of-thumb, that is to choose $\tau \le 4\times 10^{-4}N$, e.g., $\tau = 5\times 10^{-5}N$. The choice of this parameter involves a trade-off between the stability of the basis matrix  and the convergence speed of the online algorithms. In general this parameter is data-dependent so there is no principled way of choosing it in the literature~\cite{Mairal_10,Guan_12}. 
Therefore, our approach presents a heuristic way to partially resolve this issue. Regarding the latent data dimensionality $K$, there are some works \cite{Tan_13a,Bishop_98}  that   describe how to choose this parameter from a Bayesian perspective. However, they involve complex Bayesian modeling and intensive computations. Thus for efficiency we set $K=49$.  Finally we discuss the choice of the parameters specific to each online algorithm. We first reduce the number of parameters in both algorithms by setting 
$\rho_i =\rho$, for any $i\in[3]$ in OADMM\footnote{See the literature on ADMM (e.g., \cite{Boyd_11}) for  more sophisticated ways of choosing this parameter.}  and $\barkappa = \tkappa_t= \kappa$, for any $t\ge 1$ in OPGD. Then we set $\rho = 1$ and $\kappa = 0.7$. 
The above parameter setting will be referred as {\em the canonical parameter setting} in the sequel. 
Later we will show the convergence speeds of both our online algorithms are insensitive to the parameters that are set to fixed values (including $\tau$, $K$, $\rho$ and $\kappa$) 
within wide ranges on both the synthetic and {\tt CBCL} datasets. 
\subsection{Convergence Speeds}\label{sec:conv_spd}
The convergence speeds of the objective values  of OPGD and OADMM  are compared to the other four benchmarking algorithms (BPGD, BADMM, ORPCA and ONMF) on both the synthetic and {\tt CBCL} datasets. For better illustration, the comparison is divided into two parts. The first   involves 
comparison between our online algorithms and their batch counterparts. The second   involves comparison between our online algorithms to other online algorithms. 
For the first part, we use the surrogate loss function $\tilf_t$ as a measure of convergence  because i) $\tilf_t(\bW_t)$ has the same a.s.\ limit as $f(\bW_t)$ (as $t\to\infty$) and ii)   with storage of the past statistics $\{\bv_i,\bh_i,\br_i\}_{i\in[t]}$, $\tilf_t$ is easier to compute than the expected loss $f$ or the empirical loss $f_t$. For the second part, since different online algorithms have different objective functions, we propose a heuristic and unified measure of convergence $\hatf_t:\calC\to\bbR_+$, defined as 
\begin{equation}
\hatf_t(\bW) \defeq \frac{1}{t} \sum_{i=1}^t \frac{1}{2}\norm{\bv^o_{i} - \bW\bh_i}_2^2,
\end{equation}
where $\bv^o_i$ denotes the clean data sample (without outliers and observation noise) at time $i$. 
Loosely speaking, $\hatf_t(\bW_t)$ can be interpreted as the averaged regret of reconstructing the clean data up to time $t$.
In the following experiments, unless compared with other online algorithms, we always use $\tilf_t$ as the objective function for our online algorithms. 
\subsubsection{Generation of Synthetic Dataset}
The procedure to generate the synthetic dataset is described as follows. First, we generated $\bW^o\in \bbR_+^{F\times K^o}$ such that $(w^o)_{ij}\iidsim\scHN(1/\sqrt{K^o})$, for any $(i,j)\in[F]\times[K^o]$, where $K^o$ denotes the ground-truth latent dimension. 
Similarly, we generated $\bH^o\in\bbR_+^{K^o\times N}$ such that $(h^o)_{ij}\iidsim\scHN(1/\sqrt{K^o})$, for any $(i,j)\in[K^o]\times[N]$. The (clean) data matrix $\bV^o = \calP_{\tilcalV}\left(\bW^o\bH^o\right)$, where  $\tilcalV = [0,1]^{F\times N}$. Note that the normalization (projection) $\calP_{\tilcalV}$ preserves the boundedness      of the data. 
Then we generated the outlier vectors $\{\br_i\}_{i\in[N]}$. 
First, we uniformly randomly  selected a subset of $[N]$ 
with size $\floor{\nu N}$ and denoted it as $\calI$, where $\nu\in(0,1)$ denotes the fraction of  columns in $\bV^o$ to be contaminated by the outliers. 
For each $i\in\calI$, we generated the outlier vector $\br_i\in[-1,1]^F$ by first uniformly selecting a support set with cardinality $\floor{\tnu F}$, where $\tnu\in(0,1)$ denotes the outlier density in each data sample. Then the nonzero entries in $\br_i$ were \iid generated from the uniform distribution $\calU[-1,1]$. For each $i\not\in\calI$, we set $\br_i=\vecz$. Define the outlier matrix $\bR^o$ such that $(\bR^o)_{:i} = \br_i$, for any $i\in[N]$. Then the (contaminated) data matrix $\bV = \calP_{\tilcalV} (\bV^o + \bR^o+\bN)$, where $\bN\in\bbR^{F\times N}$ denotes the observation noise matrix with \iid standard normal entries. 


\subsubsection{Comparison to Other Online and Batch NMF Algorithms}
To make a fair comparison, all the algorithms were initialized with the same $\bW_0$ for each parameter setting. Moreover, all   online algorithms had the same mini-batch size $\tau$. The {\em synthetic} data matrix $\bV$ was generated using the parameters values $F = 400$, $K^o = 49$, $N = 1\times10^5$, $\nu = 0.7$ and $\tnu = 0.1$.

The convergence speeds of the aforementioned algorithms under the canonical parameter setting\footnote{For BPGD and BADMM, the step sizes and penalty parameters had the same values as those of the online algorithms. For ORPCA and ONMF, we kept the parameter settings in the original works.} are shown in Figure~\ref{fig:canon}. From Figure~\ref{fig:canon}(a), we observe our online algorithms converge much faster than their batch counterparts. From Figure~\ref{fig:canon}(b), we observe the ORPCA algorithm converges slightly faster than our online algorithms (the time difference is less than 1s). This is because ORPCA's formulation does not incorporate nonnegativity constraints (and   magnitude constraints for the outliers) so the algorithm has fewer projection steps, leading to a lower computational complexity. However, we will show in Section~\ref{sec:basis_rep} that it fails to learn the parts-based   representations due to the lack of nonnegativity constraints. Moreover, we also observe  that the ONMF algorithm fails to converge because it is unable to handle the outliers. Thus the constant perturbation from the outliers in the data samples on the learned basis matrix keeps the algorithm from converging. Furthermore, from both subfigures, we observe the objective function of OADMM has a smoother decrease that that of OPGD. This is OADMM solves the optimization problems \eqref{eq:min_hr} and \eqref{eq:min_W_tr} in the dual space so it avoids the sharp transitions from  plateaus to valleys in the primal space. Apart from this difference, {\em the convergence speeds of OADMM and OPGD are almost the same}. 
\subsubsection{Insensitivity to Parameters}
We  now examine the influences of various parameters, including $\tau$, $K$, $\rho$ and $\kappa$, on the convergence speeds of our online algorithms on the synthetic dataset. To do so, we vary one parameter at a time while keeping the other parameters fixed as in the canonical setting. We first vary $\tau$ from 5 to 40 in a log-scale, since our proposed heuristic rule indicates $\tau\le 40$. Figure~\ref{fig:tau} shows that our rule-of-thumb works well since within its indicated range, different values of $\tau$ have very little effects on the convergence speeds of both our online algorithms. Next, we consider other values of the latent dimension $K$. In particular, we consider $K=25$ and $K=100$. Figure~\ref{fig:K} shows that the convergence speeds of both our online and batch algorithms are insensitive to this parameter. Then we vary $\rho$ (in both BADMM and OADMM) from 1 to 1000 on a log-scale. We can make two observations from Figure~\ref{fig:rho}. First, as $\rho$ increases,  both ADMM-based algorithms converge slower.  However, the change of the convergence speed of OADMM is much less drastic compared to BADMM. In fact, in less than 10s, all the OADMM algorithms (with different values of $\rho$) converge. This is because for OADMM, a large $\rho$ only has a significant effect on its convergence speed in the initial transient phase. After the algorithm reaches its steady state (i.e., the basis matrix becomes stable), the effect  of a large $\rho$ becomes minimal in terms of solving both \eqref{eq:min_hr} and \eqref{eq:min_W_tr}. Similar effects of the different values of $\kappa$ on the convergence speeds of the PGD-based algorithms are observed in Figure~\ref{fig:kappa}. 
Together, Figure~\ref{fig:rho} and \ref{fig:kappa} reveal {\em another advantage} of our online algorithms. That is, our online algorithms have a significantly larger tolerance of   suboptimal parameters than their batch counterparts. 

\subsubsection{Effects of Proportion of Outliers}
We evaluate the performance of our online (and batch) algorithms (with the canonical parameter setting) on the synthetic dataset with a larger proportion of outliers. Specifically, in generating the synthetic data matrix $\bV$, we keep $F$, $K^o$ and $N$ unchanged, but simultaneously increase $\nu$ and $\tnu$. Figures~\ref{fig:den}(a) and \ref{fig:den}(b) show the convergence results of our algorithms on the synthetic dataset with $(\nu,\tnu)$ equal to $(0.8,0.2)$ and $(0.9,0.3)$ respectively. From Figure~\ref{fig:den}, we observe that our online algorithms still converge fast and steadily (without fluctuations) to a stationary point even though both $\nu$ and $\tnu$ have increased. 

\subsubsection{Experiments on the {\tt CBCL} Dataset}\label{sec:exp_CBCL}
To generate the {\tt CBCL} face dataset with outliers, we first replicated it by a factor $p = 50$ (so that the total number of data samples is of the order $10^5$) and randomly permuted the aggregated dataset as described in  Section~\ref{sec:heuristics}. We then contaminated it with outliers in the same manner as for the synthetic dataset. However, we avoided adding observation noise since it had been already introduced by the image acquisition process. 

We then conducted experiments on 
the contaminated face dataset with the same parameter settings as those on the synthetic data. 
In the interest of space, we defer the convergence results of all the online and batch algorithms to Section~\ref{sec:exp_res}. 
The results show that both our online algorithms demonstrate fast and steady convergences on this dataset. Moreover, the convergence speeds are insensitive to the key parameters ($\tau$, $K$, $\rho$ and $\kappa$). 
Therefore, in the subsequent experiments, we will only focus on  the canonical parameter setting unless mentioned otherwise. 

\begin{figure}[t]
\subfloat[]{\includegraphics[width=.5\columnwidth,height=.4\columnwidth]{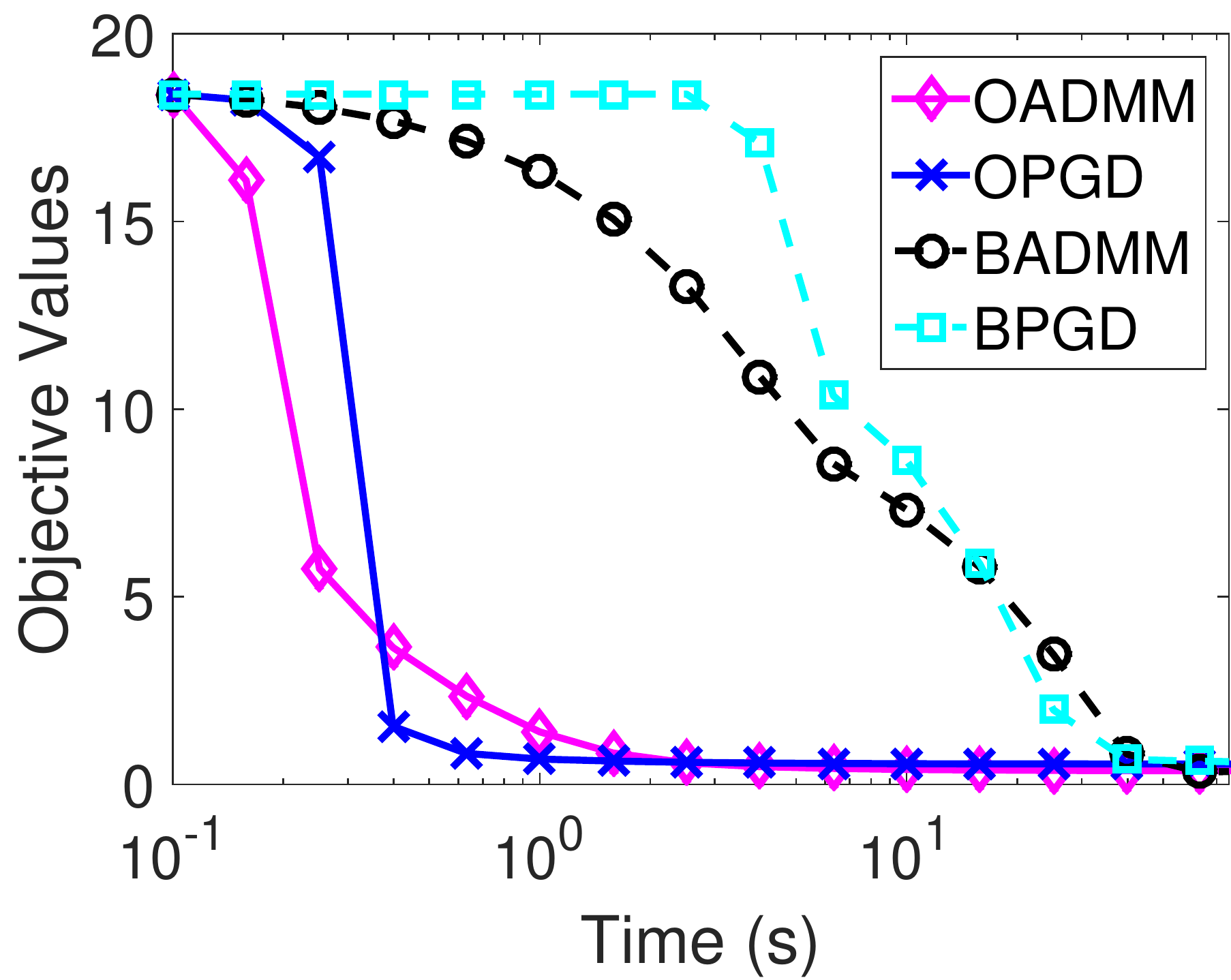}}\hfill
\subfloat[]{\includegraphics[width=.475\columnwidth,height=.4\columnwidth]{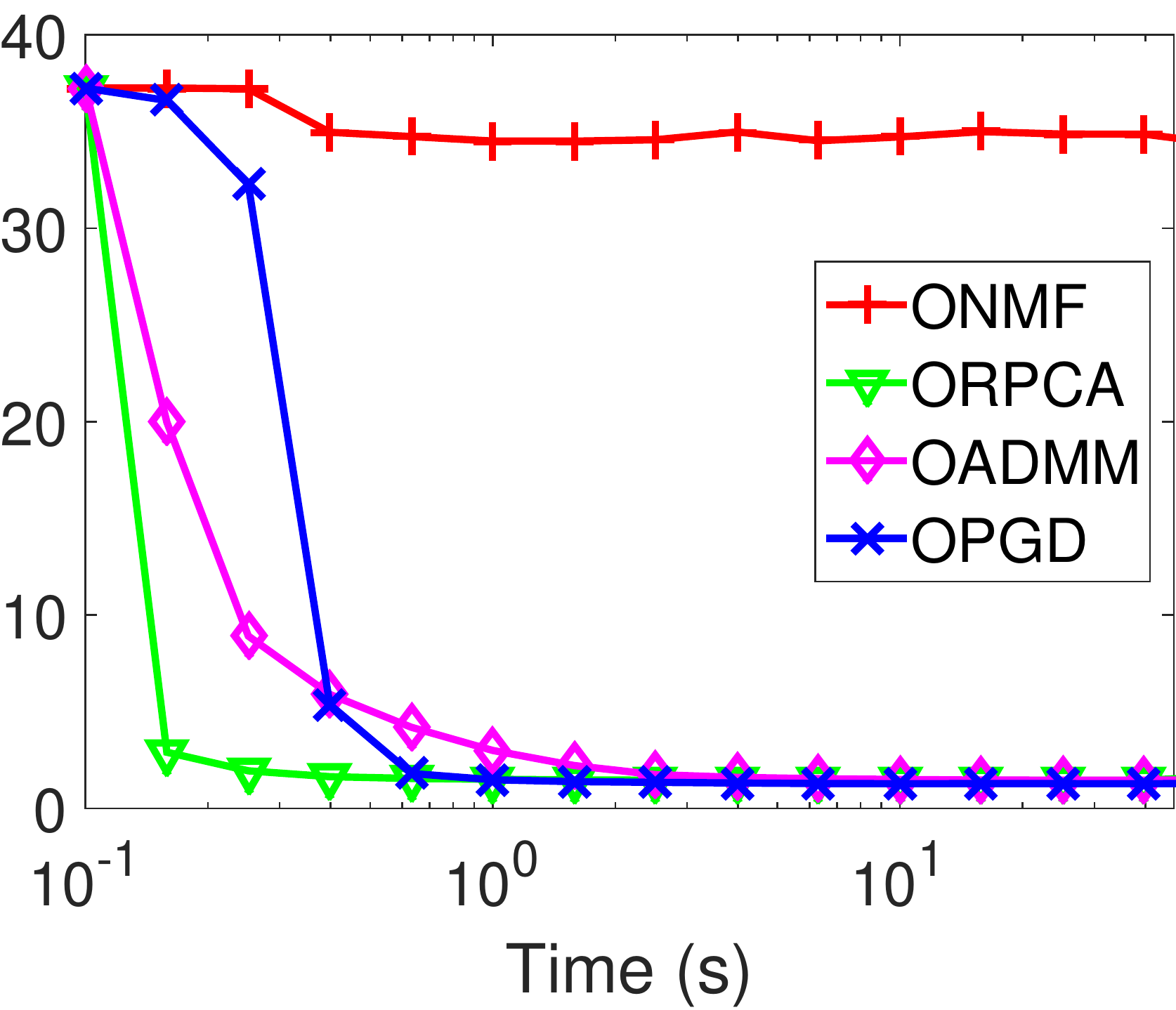}}
\caption{The objective values (as a function of time) of (a) our online algorithms and their batch counterparts (b) our online algorithms and other online algorithms. The parameters are set according to the canonical setting.}\label{fig:canon}
\end{figure}

\begin{figure}[t]
\subfloat[OPGD]{\includegraphics[width=.5\columnwidth,height=.38\columnwidth]{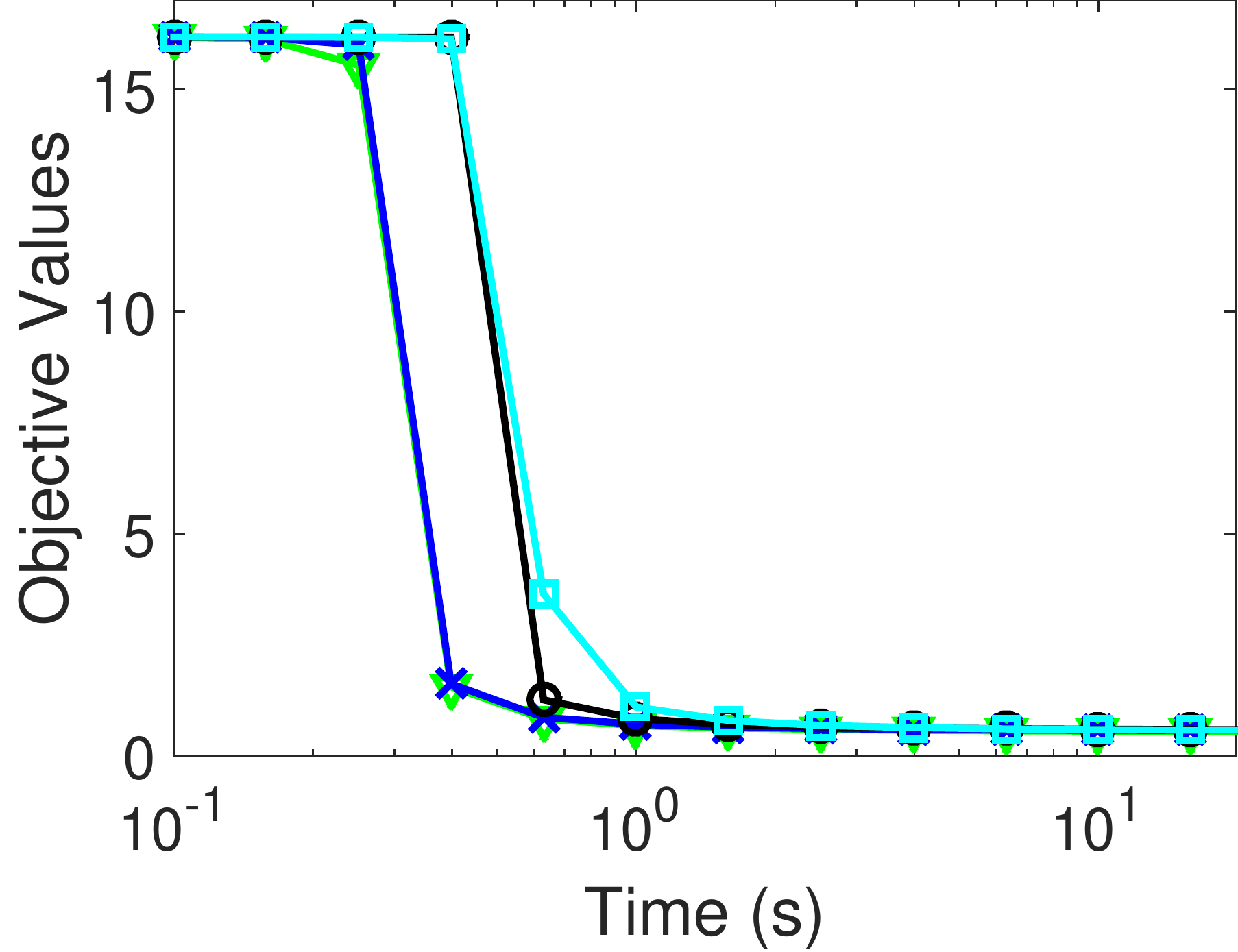}}\hfill
\subfloat[OADMM]{\includegraphics[width=.475\columnwidth,height=.38\columnwidth]{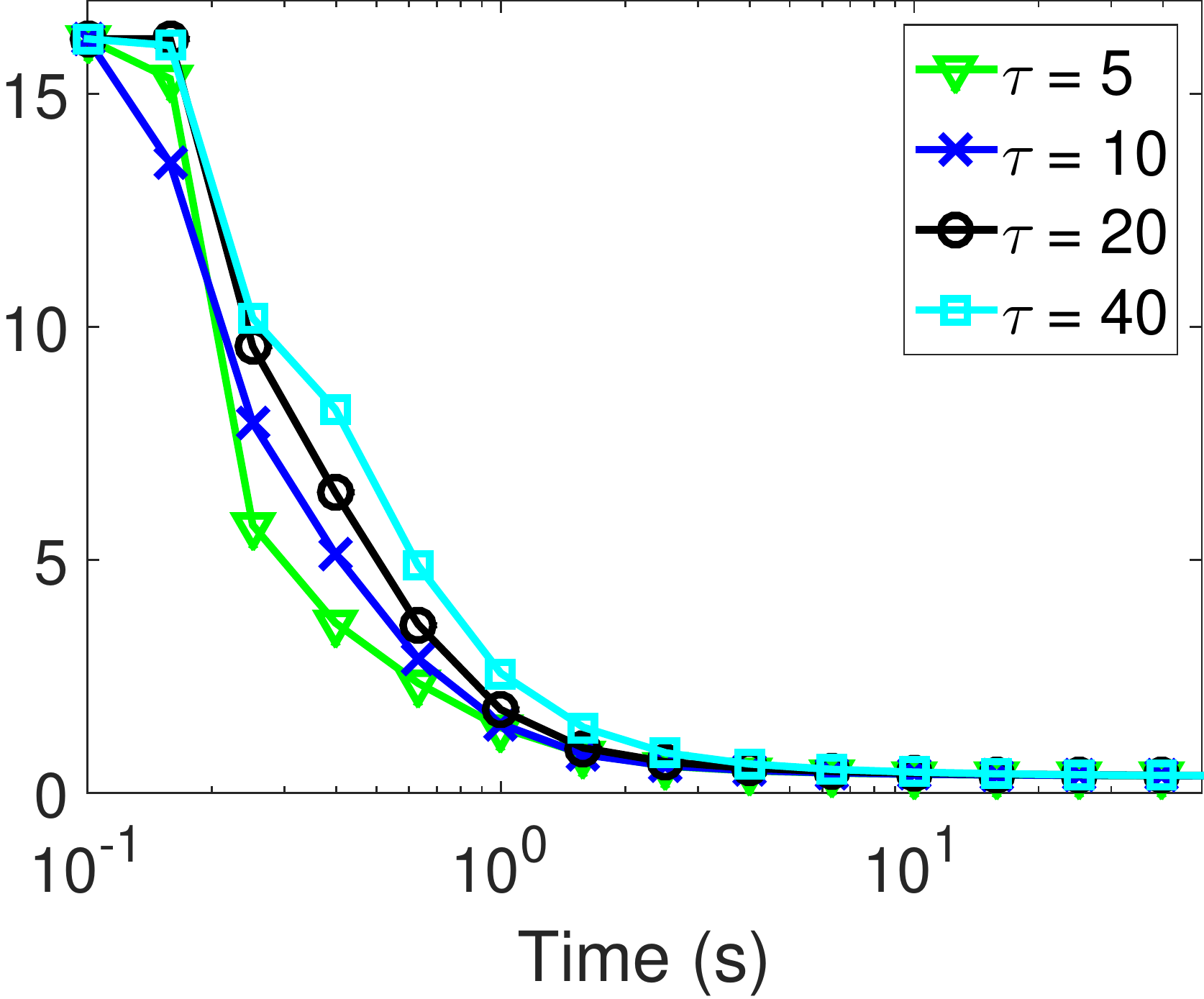}}
\caption{The objective values  of (a) OPGD and (b) OADMM for different values of $\tau$. All the other parameters are set according to the canonical setting.}\label{fig:tau}
\end{figure}

\begin{figure}[t]
\subfloat[ $K=25$ ]{\includegraphics[width=.5\columnwidth,height=.4\columnwidth]{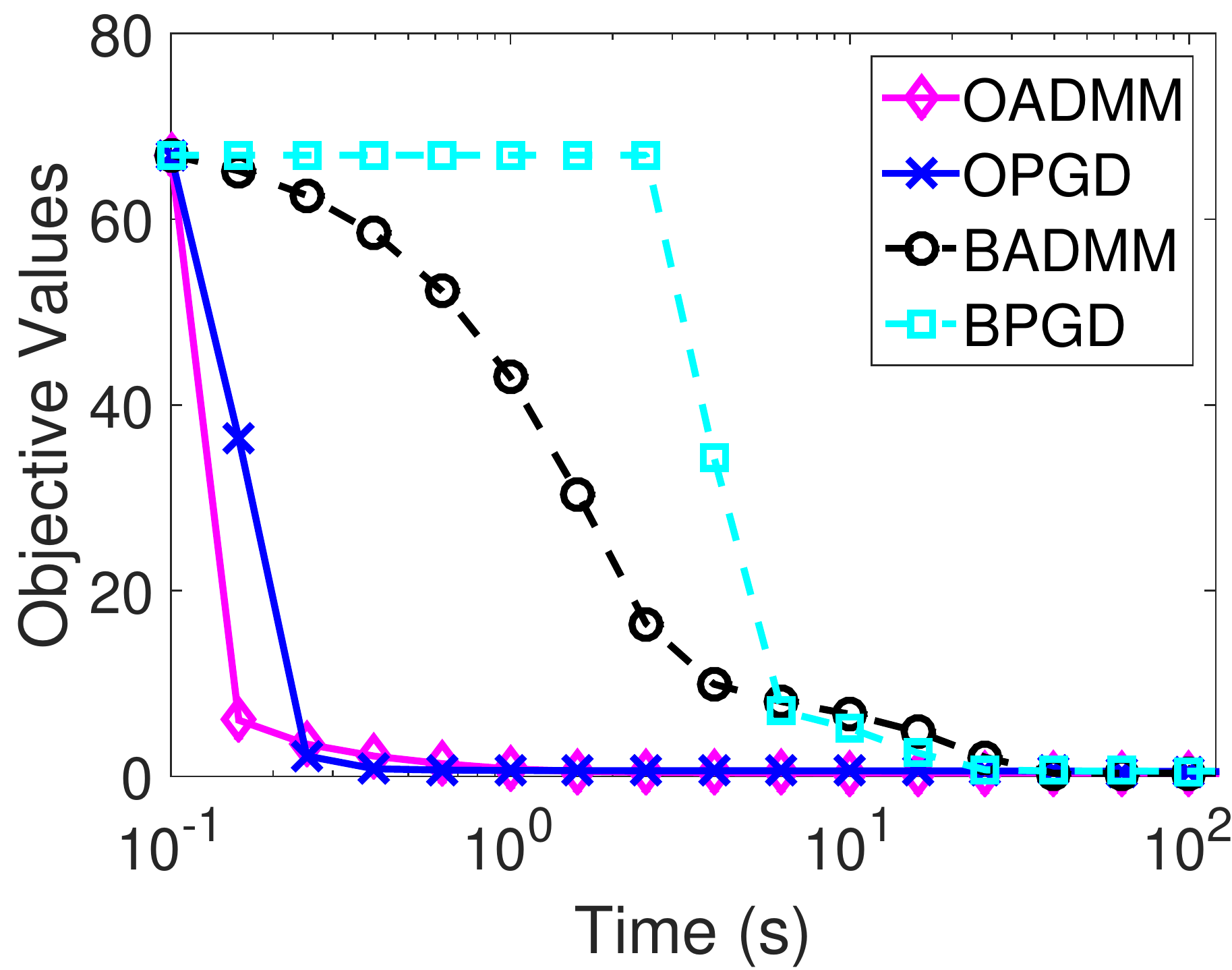}}\hfill
\subfloat[$K=25$ ]{\includegraphics[width=.475\columnwidth,height=.4\columnwidth]{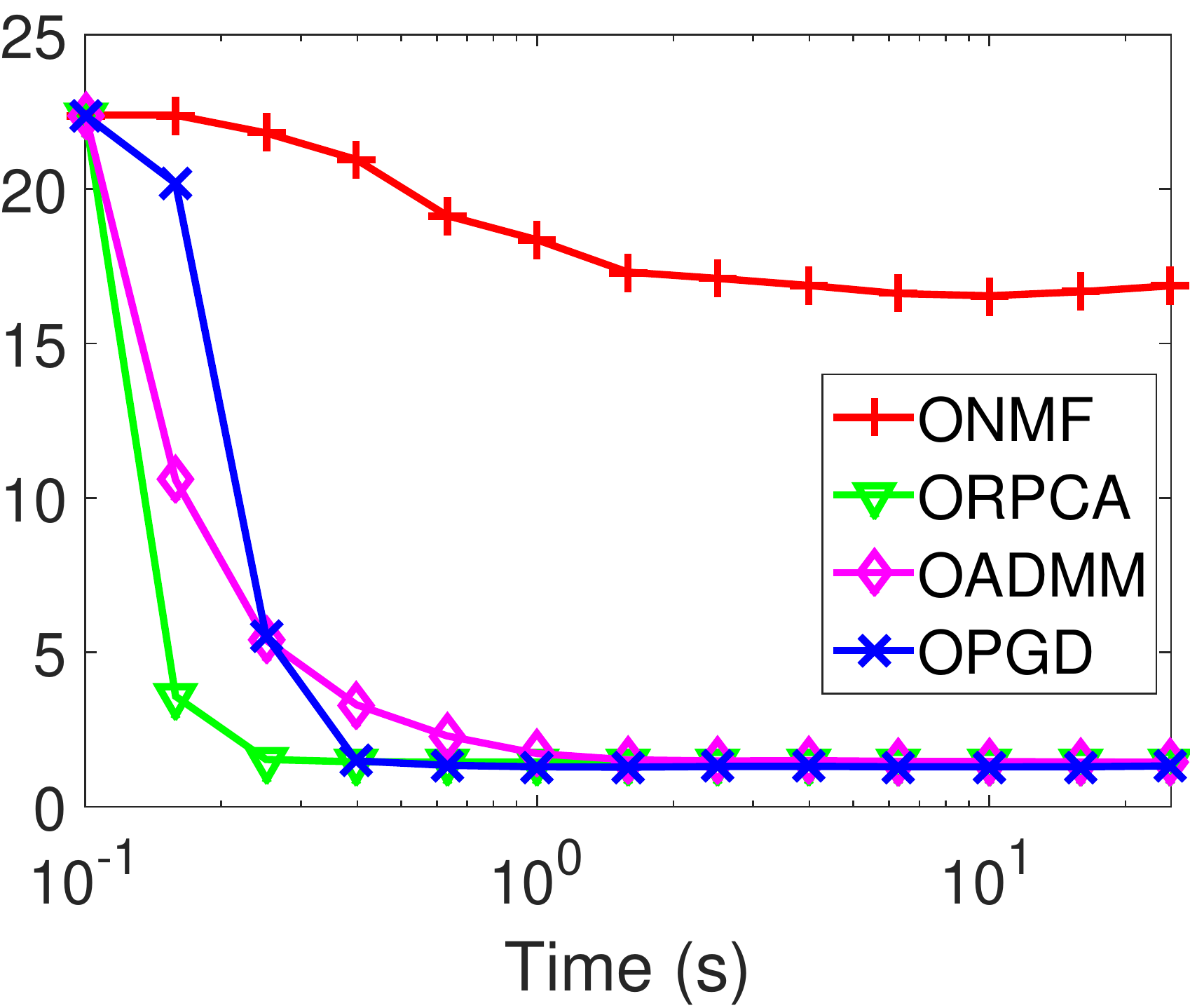}}\\
\subfloat[$K=100$ ]{\includegraphics[width=.5\columnwidth,height=.4\columnwidth]{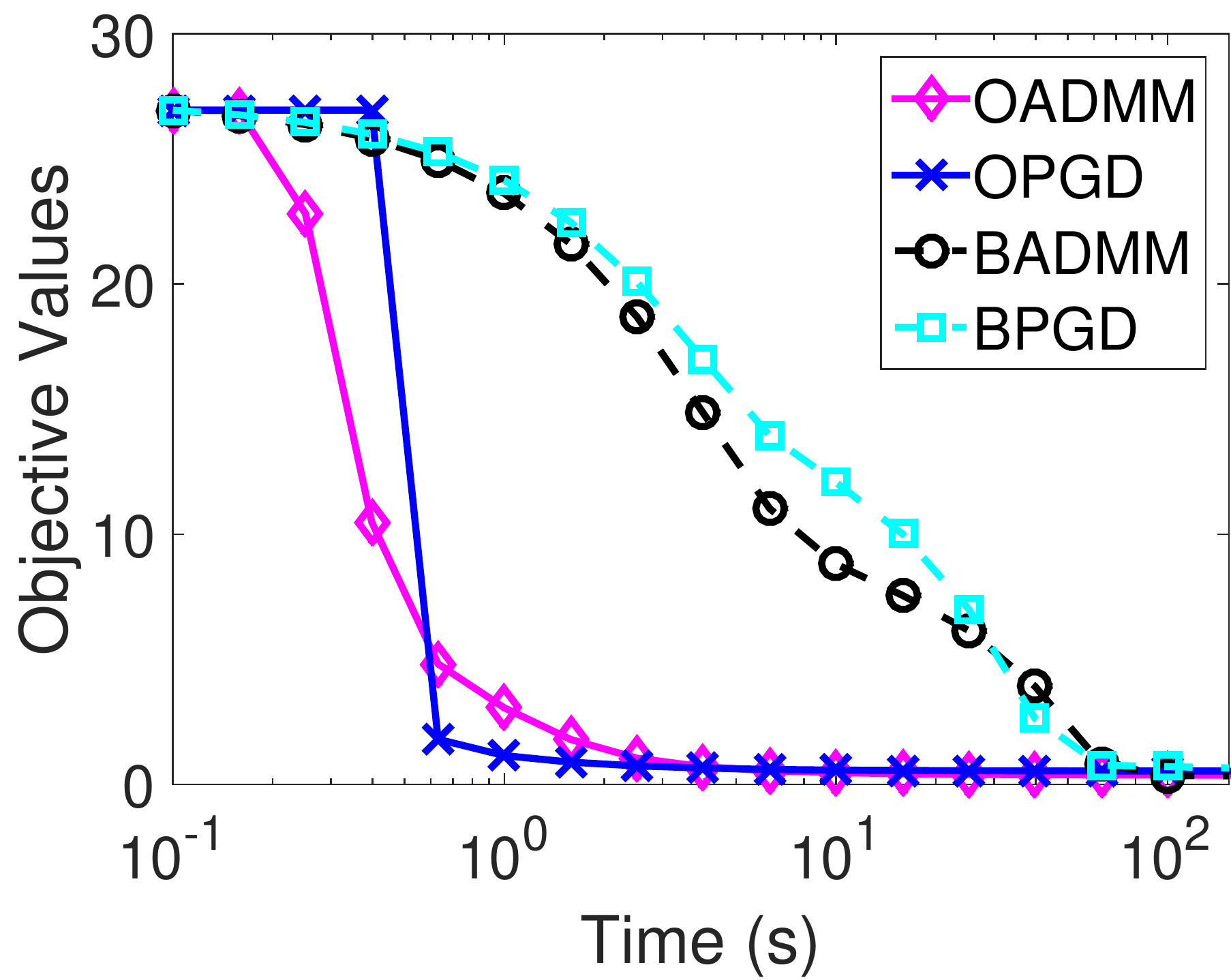}}\hfill
\subfloat[$K=100$ ]{\includegraphics[width=.475\columnwidth,height=.4\columnwidth]{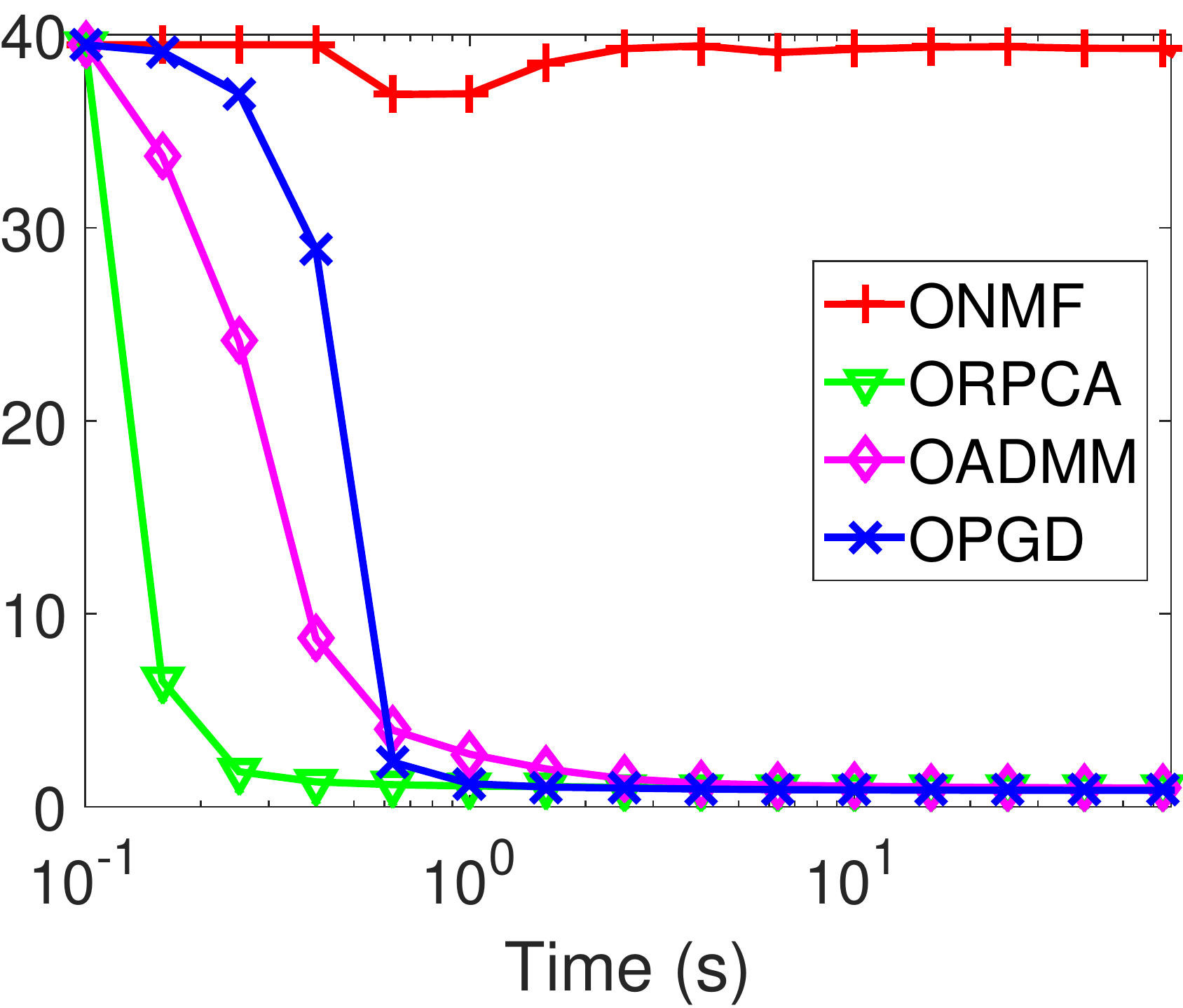}}
\caption{The objective values (as a function of time) of all the algorithms for different values of $K$. In (a) and (b), $K=25$. In (c) and (d), $K=100$. All the other parameters are set according to the canonical setting.}\label{fig:K}
\end{figure}

\begin{figure}[t]
\subfloat[BADMM]{\includegraphics[width=.5\columnwidth,height=.4\columnwidth]{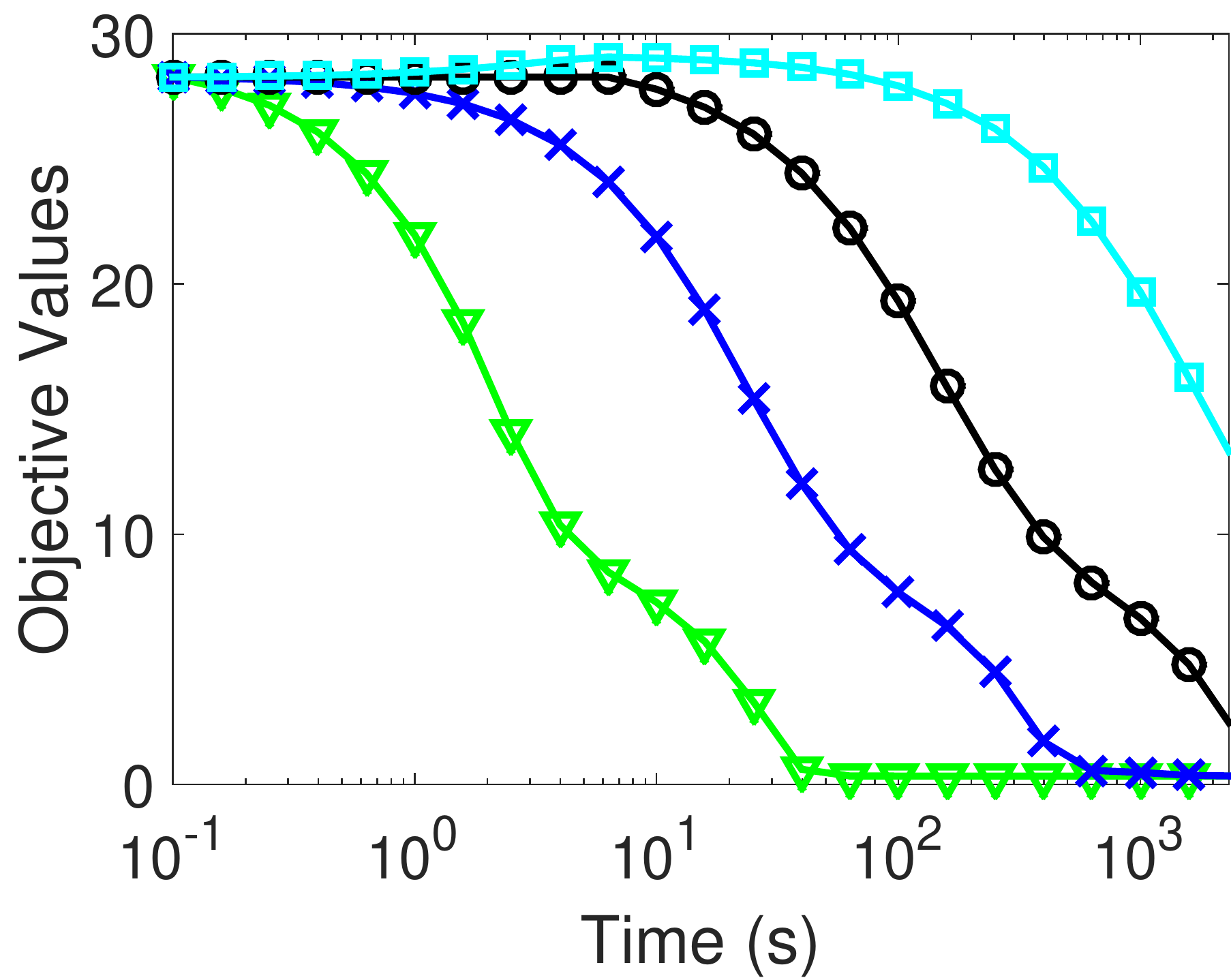}}\hfill
\subfloat[OADMM]{\includegraphics[width=.475\columnwidth,height=.39\columnwidth]{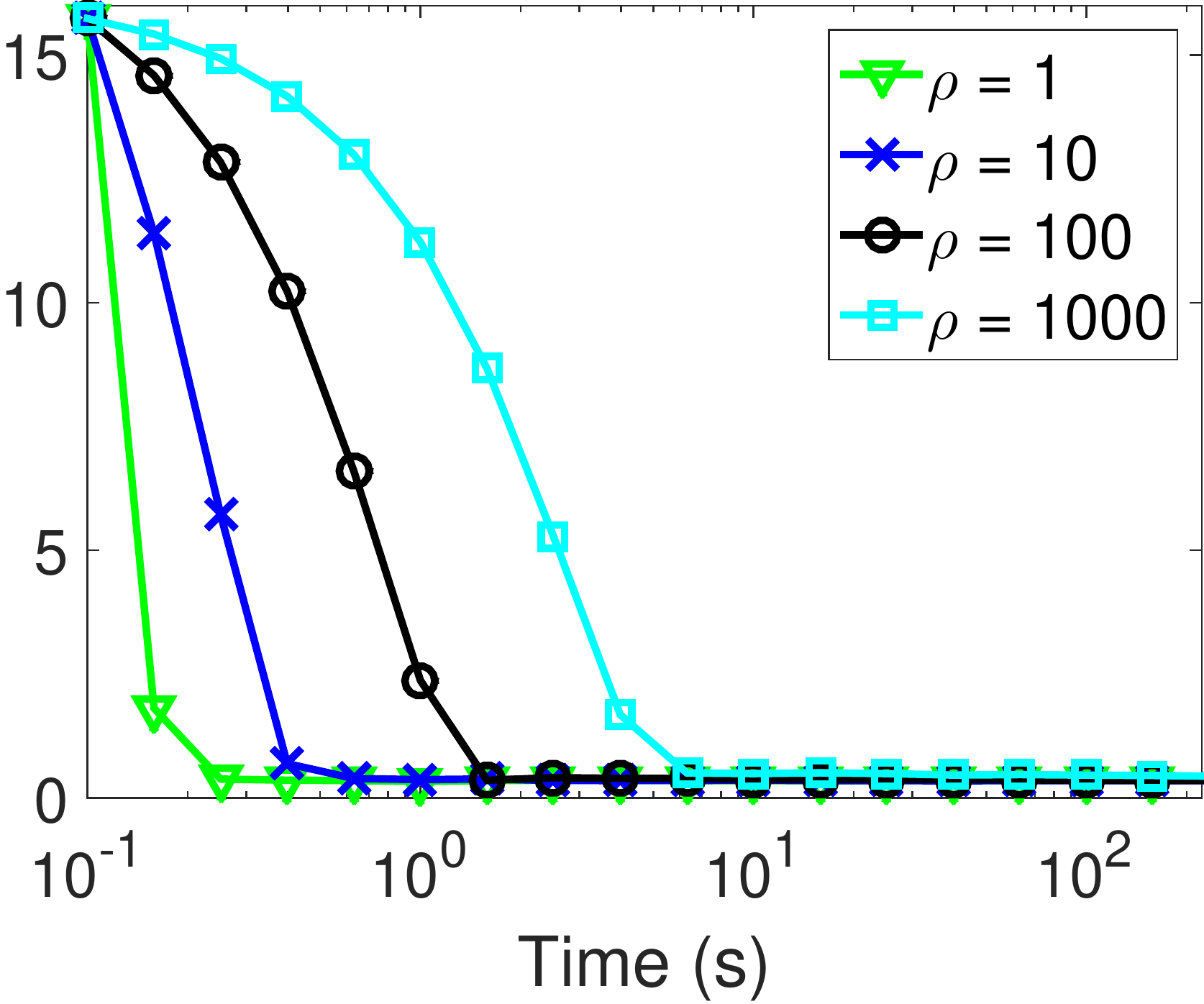}}
\caption{The objective values  of (a) BADMM and (b) OADMM for different values of $\rho$. All the other parameters are set according to the canonical setting.}\label{fig:rho}
\end{figure}
 
\begin{figure}[t]
\subfloat[BPGD]{\includegraphics[width=.5\columnwidth,height=.4\columnwidth]{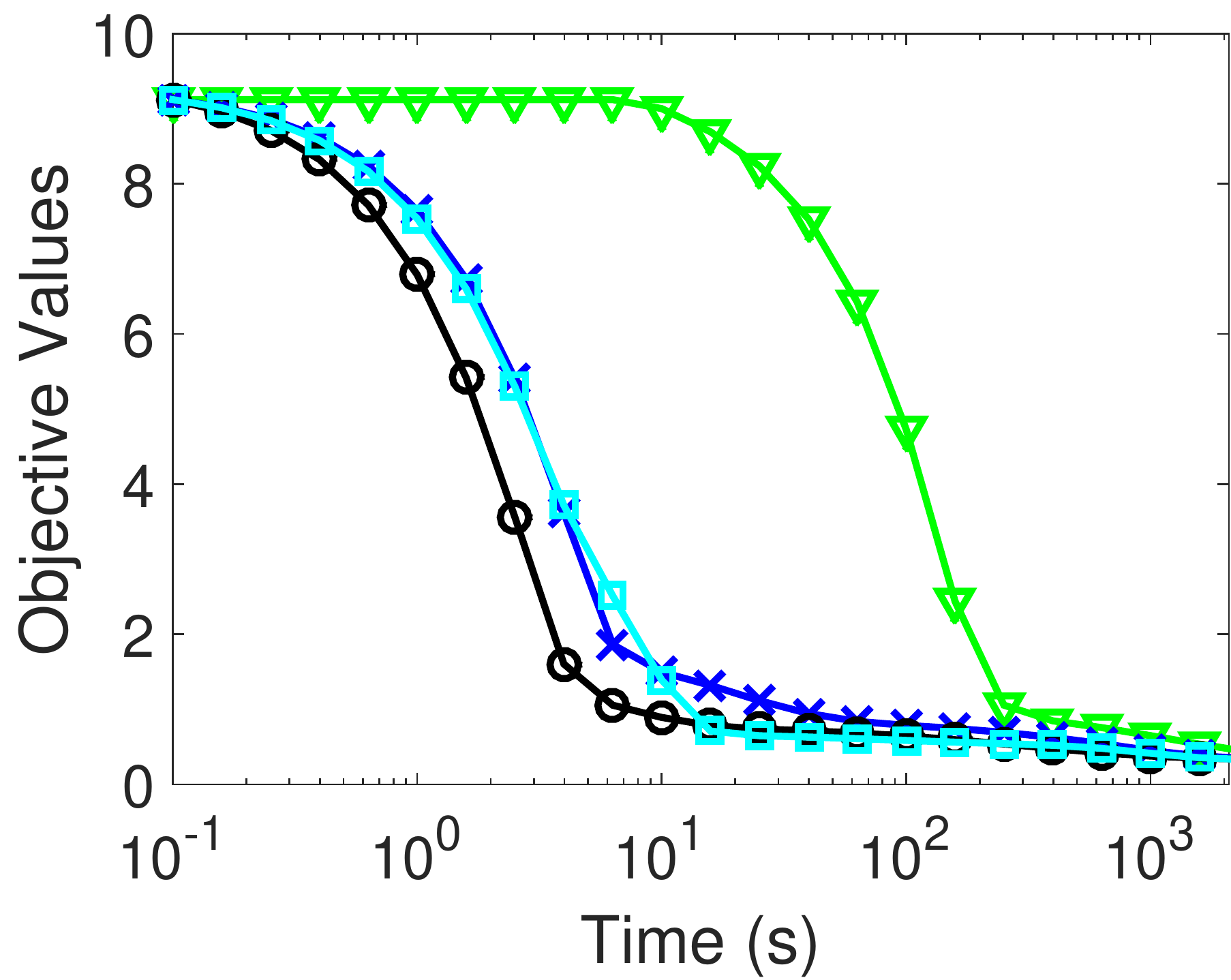}}\hfill
\subfloat[OPGD]{\includegraphics[width=.475\columnwidth,height=.4\columnwidth]{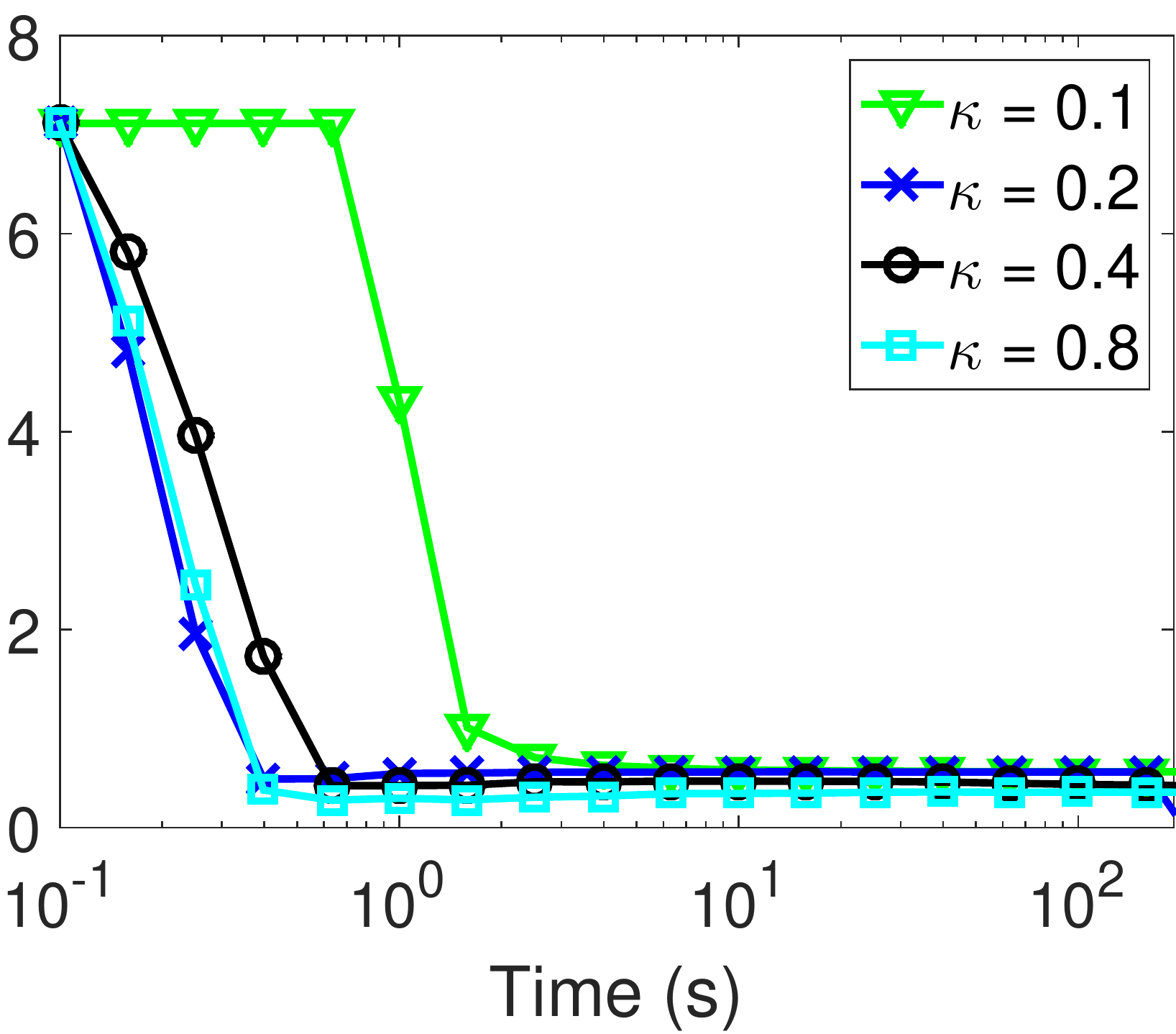}}
\caption{The objective values  of (a) BPGD and (b) OPGD for different values of $\kappa$. All the other parameters are set according to the canonical setting.}\label{fig:kappa}
\end{figure}

\begin{figure}[t]
\subfloat[$\nu = 0.8$, $\tnu = 0.2$]{\includegraphics[width=.5\columnwidth,height=.4\columnwidth]{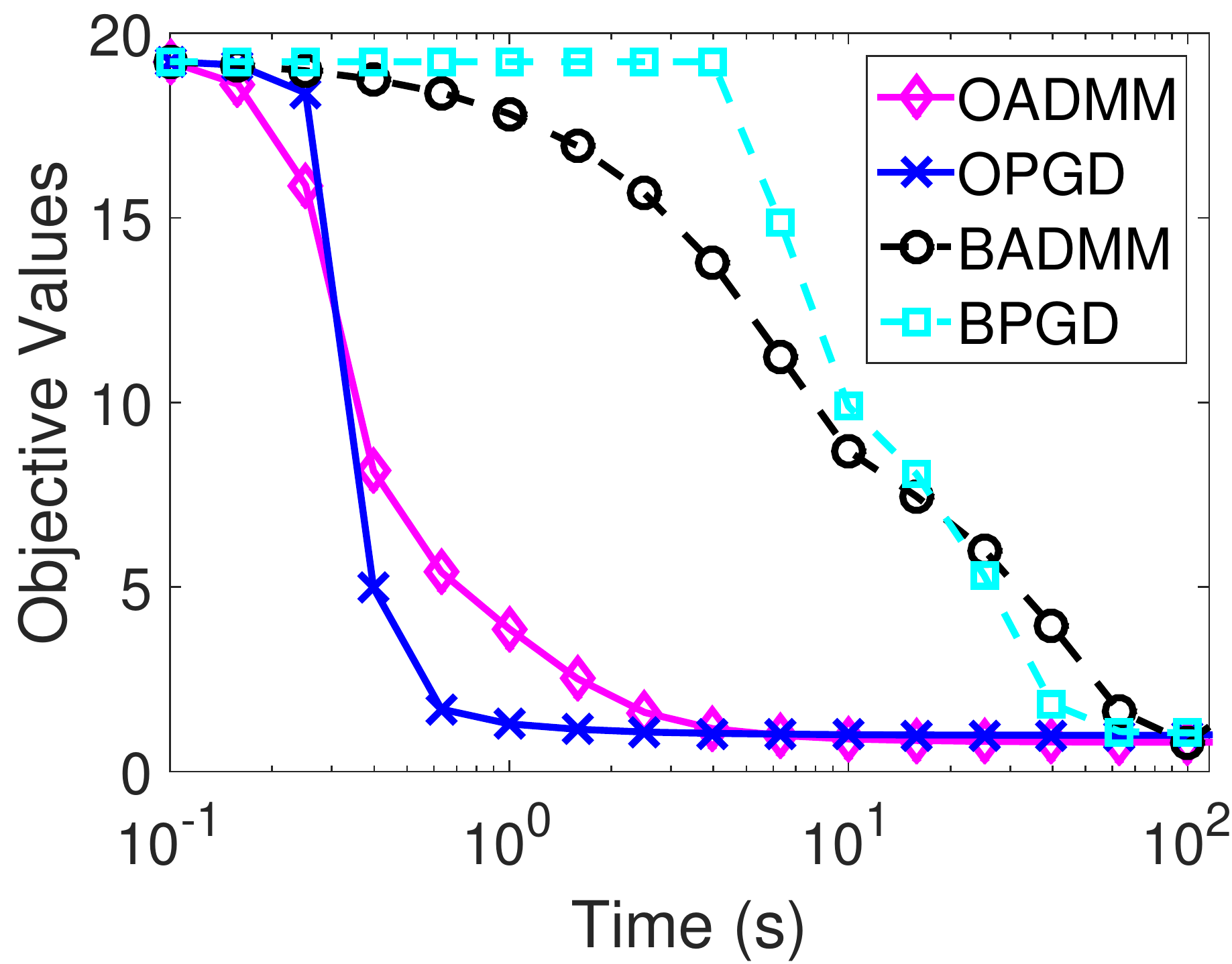}}\hfill
\subfloat[$\nu = 0.9$, $\tnu = 0.3$]{\includegraphics[width=.475\columnwidth,height=.39\columnwidth]{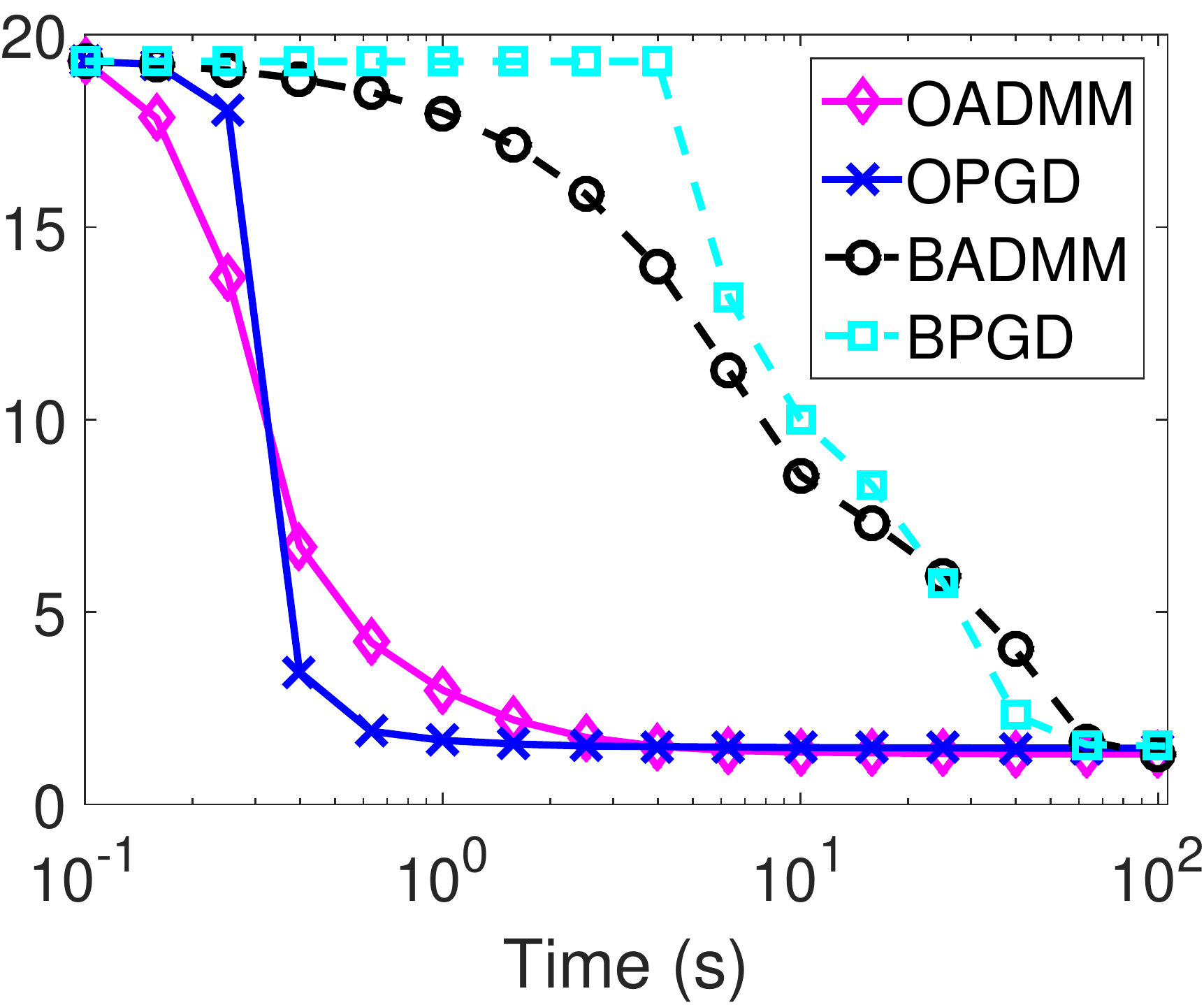}}
\caption{The objective values of our online and batch algorithms on the synthetic dataset with a larger proportion of outliers.}\label{fig:den}
\end{figure}

\subsection{Basis Representations}\label{sec:basis_rep}
We now examine the basis matrices learned by all the algorithms  on the {\tt CBCL} dataset with outliers as introduced in Section~\ref{sec:exp_CBCL}. The parameters controlling the outlier density, $\nu$ and $\tnu$ are set to 0.7 and 0.1 respectively. Figure~\ref{fig:basis} shows the basis representations learned by all the algorithms. From this figure, we observe that the basis images learned by ONMF have large residues of salt and pepper noise. Also, the basis images learned by ORPCA appear to be non-local and do not yield meaningful interpretations. In contrast, the basis images learned by our online (and batch) algorithms are free of noise and much more local than those learned by ORPCA. We can easily observe facial parts from the learned basis. Also, the basis images learned by our four algorithms are of similar quality.
To further enhance the sparsity of the basis images learned, one can add sparsity constraints to each column of the basis matrix, as done for example in \cite{Hoyer_04}. However, we omit such constraints in our discussions. Another parameter affecting the sparsity of the learned basis images is the latent dimension $K$. In general, the larger $K$ is, the sparser the basis images will be. Here we set $K=49$, in accordance with the original setting in \cite{Lee_99}.

\begin{figure}
\subfloat[ONMF]{\includegraphics[width=.475\columnwidth]{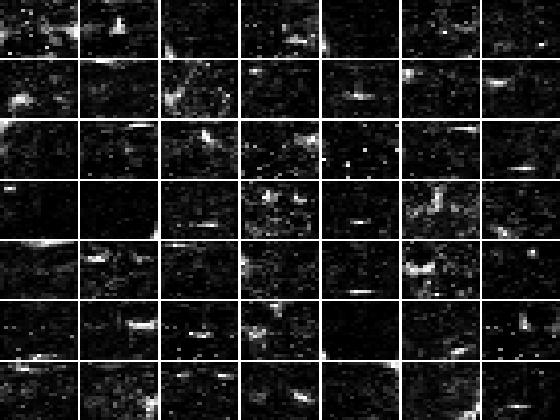}}\hfill
\subfloat[ORPCA]{\includegraphics[width=.475\columnwidth]{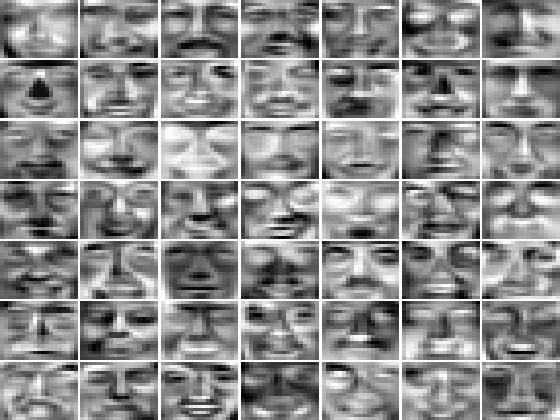}}\\
\subfloat[OADMM]{\includegraphics[width=.475\columnwidth]{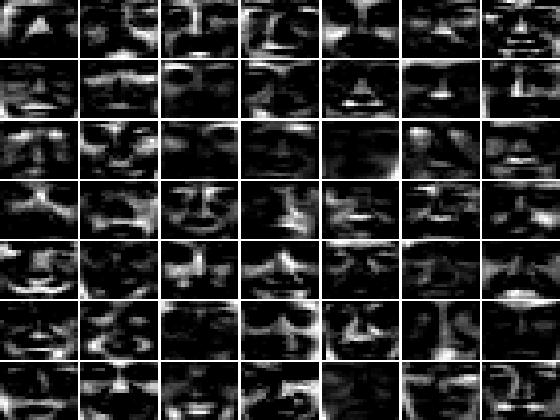}}\hfill
\subfloat[OPGD]{\includegraphics[width=.475\columnwidth]{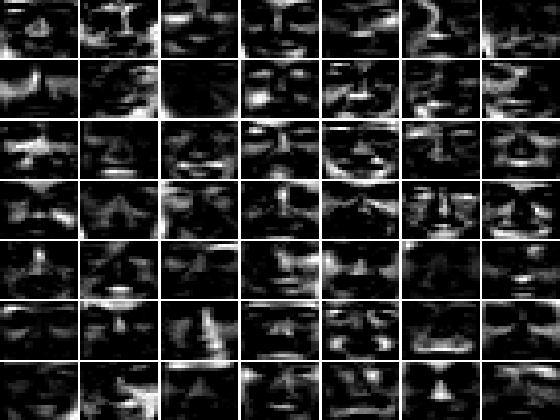}}\\
\subfloat[BADMM]{\includegraphics[width=.475\columnwidth]{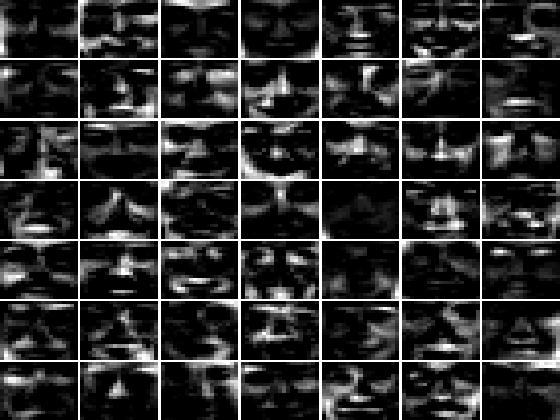}}\hfill
\subfloat[BPGD]{\includegraphics[width=.475\columnwidth]{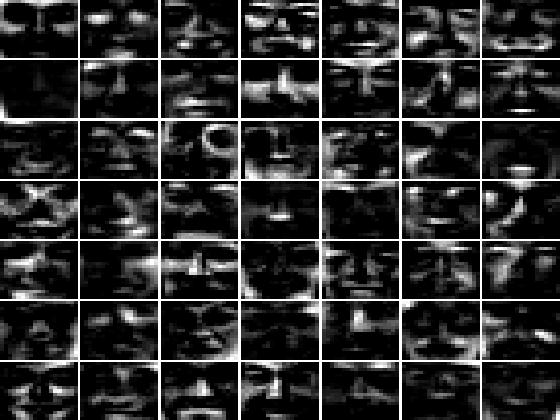}}
\caption{Basis representations learned by all the algorithms ($\nu = 0.7$, $\tnu = 0.1$). All the parameters are in the canonical setting ($K=49$).} \label{fig:basis}
\end{figure}

\subsection{Application I: Image Denoising}\label{sec:img_denois}
A natural application that arises from the experiments on the {\tt CBCL} dataset would be image denoising, in particular, removing the salt and pepper noise on images.
The metric commonly used to measure the quality of the reconstructed images is peak signal-to-noise ratio (PSNR)\cite{PSNR_98}. A larger value of PSNR indicates better quality of image denoising. With a slight abuse of definition, we apply this metric to all the recovered images. For the batch algorithms, we define 
\begin{equation}
\mbox{PSNR} \defeq -10\log_{10}\left\{\normt{\bV^o-\hatbW\hatbH}_F^2/(FN)\right\},
\end{equation}
where $\bV^o$, $\hatbW$ and $\hatbH$ denote the matrix of the clean images, estimated basis matrix and estimated coefficient matrix respectively. For the online algorithms, we instead define PSNR in terms of $\bW_N$ (the final dictionary output by Algorithm~\ref{algo:generic}) and the past statistics $\{\bh_i\}_{i\in[N]}$
\begin{align}
\mbox{PSNR} &\defeq -10\log_{10}\left\{\sum_{i=1}^N \norm{\bv_i^o-\bW_N\bh_i}_2^2/(FN)\right\}\\
&\eqcst -10\log_{10} \hatf_N(\bW_N).
\end{align}
In other words, a low averaged regret  will result in a high PSNR. 
Table~\ref{tab:PSNR_K49} shows the image denoising results of all the algorithms on the {\tt CBCL} face dataset. Here the  settings 1, 2 and 3 represent different densities of the salt and pepper noise. Specifically, these settings correspond to $(\nu,\tnu)$ equal to $(0.7,0.1)$, $(0.8,0.2)$ and $(0.9,0.3)$ respectively. 
All the results were obtained using ten random initializations  of $\bW_0$.
From Table~\ref{tab:PSNR_K49}(a), we observe that for all the three settings, in terms of PSNRs, our online algorithms only slightly underperform their batch counterparts, but slightly outperform ORPCA and greatly outperform ONMF. 
From Table~\ref{tab:PSNR_K49}(b), we observe our online algorithms only have slightly longer running times than ORPCA, but are significantly faster than the rest ones. Thus in terms of the trade-off between the computational efficiency and the quality of image denoising, our online algorithms achieve comparable performances with ORPCA but are much better than the rest algorithms. Also, in terms of PSNRs and running times, no significant differences can be observed between our online algorithms.
Similar results were observed when $K = 25$ and $K=100$. See Section~\ref{sec:exp_res} 
for details. 


\begin{table}[t!]\centering\footnotesize
\setlength\tabcolsep{2.5pt}
\caption{Average PSNRs (in dB) and running times (in seconds) with standard deviations of all the algorithms on the {\tt CBCL} face dataset with different noise density ($K=49$).}\label{tab:PSNR_K49}
\subfloat[PSNRs]{
\begin{tabular}{|c|c|c|c|c|c|c|c|c|c|}
\hline
\quad&Setting 1 & Setting 2 & Setting 3\\\hline
OADMM &  11.37 $\pm$ 0.02 &  11.35 $\pm$ 0.15 &  11.33 $\pm$ 0.11\\ \hline
OPGD &  11.48 $\pm$ 0.10 &  11.47 $\pm$ 0.05 &  11.39 $\pm$ 0.03\\ \hline
BADMM &  11.53 $\pm$ 0.19 &  11.51 $\pm$ 0.10 &  11.48 $\pm$ 0.03\\ \hline
BPGD &  11.56 $\pm$ 0.07 &  11.52 $\pm$ 0.14 &  11.48 $\pm$ 0.05\\ \hline
ONMF &   5.99 $\pm$ 0.12 &   5.97 $\pm$ 0.18 &   5.95 $\pm$ 0.17\\ \hline
ORPCA &  11.25 $\pm$ 0.04 &  11.24 $\pm$ 0.19 &  11.23 $\pm$ 0.05\\ \hline
\end{tabular}
}\\
\subfloat[Running times]{
\begin{tabular}{|c|c|c|c|c|c|c|c|c|c|}
\hline
\quad&Setting 1 & Setting 2 & Setting 3\\\hline
OADMM & 416.68 $\pm$ 3.96 & 422.35 $\pm$ 3.28 & 425.83 $\pm$ 4.25\\ \hline
OPGD & 435.32 $\pm$ 4.80 & 449.25 $\pm$ 3.18 & 458.58 $\pm$ 4.67\\ \hline
BADMM & 1000.45 $\pm$ 10.18 & 1190.55 $\pm$ 5.88 & 1245.39 $\pm$ 10.59\\ \hline
BPGD & 1134.89 $\pm$ 11.37 & 1185.84 $\pm$ 9.83 & 1275.48 $\pm$ 9.48\\ \hline
ONMF & 2385.93 $\pm$ 11.15 & 2589.38 $\pm$ 15.57 & 2695.47 $\pm$ 14.15\\ \hline
ORPCA & 368.35 $\pm$ 3.23 & 389.59 $\pm$ 3.49 & 399.85 $\pm$ 4.12\\ \hline
\end{tabular}
}
\end{table}

\subsection{Application II: Shadow Removal}\label{sec:shadow_rmv}
We evaluated the performances of all the online and batch algorithms on removing shadows in the face images in the {\tt YaleB} dataset. It is well-known from the image processing literature that the shadows in an image can be regarded as inherent outliers. Therefore, the shadow removal task serves as another meaningful application of our online (and batch) algorithms.  We first briefly describe the experiment procedure. For each subject in the {\tt YaleB} dataset, we aggregated and randomly permuted his/her images. We set the aggregation factor $p$ to $50$ in consistency with the previous experiments. We then regarded the resulting data as the input to all the algorithms. 
Finally, we reconstructed the images in a similar way as in Section~\ref{sec:img_denois}.
The experiments were run using ten random initializations of $\bW_0$. 
Figure~\ref{fig:shadow} shows some (randomly sampled) reconstructed images on subjects No.2 and No.8 (with one initialization of $\bW_0$). From this figure, we observe that overall, our online algorithms perform almost as well as their batch counterparts, except for small  artifacts (e.g., salt noise) in the recovered images by our online algorithms. We also observe that the other two online algorithms are inferior to our online algorithms. Specifically, ORPCA has more prominent artifacts (e.g., salt noise with larger density) in the recovered images. ONMF in most cases either fails to remove shadow or causes large distortions to the original images. 
The results by other initializations of $\bW_0$ are similar to those shown in Figure~\ref{fig:shadow}.
Table~\ref{tab:shadow_time} shows the running times of all algorithms on the images of subject No.2. The average running times of these algorithms on other subjects are similar to those on subject No.2. This table, together with Figure~\ref{fig:shadow}, suggests  that our online algorithms achieve the best trade-off between the computational efficiency and the quality of shadow removal. Also, both of our online algorithms have similar performances on the shadow removal task. 

\begin{table}[t]\centering\footnotesize
\caption{Average running times (in seconds) of all the algorithms on subject No.2.}\label{tab:shadow_time}
\setlength\tabcolsep{1.5pt}
\begin{tabular}{|c|c|c|c|}
\hline
Algorithms & Time (s) & Algorithms & Time (s)\\\hline
ONMF & $436.11\pm 10.85$ & ORPCA & $16.38\pm 2.43$\\\hline
OADMM & $51.12\pm 2.69$ & OPGD & $62.34\pm 2.75$\\\hline
BADMM & $175.52\pm 7.84$ & BPGD & $212.71\pm 6.33$ \\\hline
\end{tabular}
\end{table}

\begin{figure}
\subfloat{\includegraphics[width=\columnwidth]{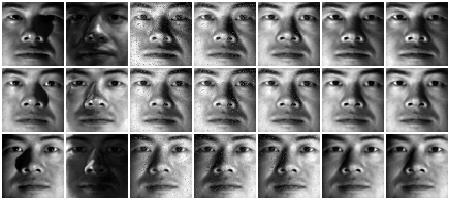}}\vspace{-0.3cm}\\
\subfloat{\includegraphics[width=\columnwidth]{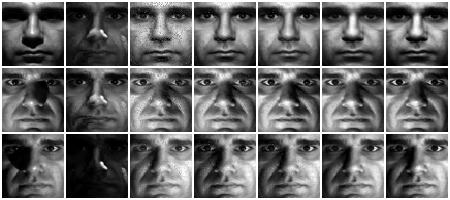}}\vspace{-0.1cm}\\
\hspace*{0.5cm}(a)\hspace*{.85cm}(b)\hspace*{.85cm}(c)\hspace*{.85cm}(d)\hspace*{.85cm}(e)\hspace*{.9cm}(f)\hspace*{.85cm}(g)
\caption{Results of shadow removal by all the online and batch algorithms on subjects No.2 (upper)  and No.8 (lower) in the {\tt YaleB} face dataset. The labels (a) to (g) denote the original images and results produced by ONMF, ORPCA, OADMM, OPGD, BADMM, BPGD respectively.}\label{fig:shadow}
\end{figure}

\subsection{Application III: Foreground-Background Separation}\label{sec:backgnd}
Finally, we also evaluated the performance of all the online and batch algorithms on the task of foreground-background separation, which is important in video surveillance \cite{i2r_04,Candes_11}. Since the foreground objects in each video frame generally only occupy a small fraction of pixels, they have been considered as outliers in the literature \cite{Candes_11,Netra_14}. Therefore, our online (and batch) algorithms can be applied to estimate the background scene, as well as learn the foreground objects. In this section, we consider two video sequences, {\tt Hall} and {\tt Escalator} from the {\tt i2r} dataset\cite{i2r_04}. Each video sequence consists of $200$ video frames and the resolutions of the frames in {\tt Hall} and {\tt Escalator} are $144\times 176$ and $130\times 160$ respectively.\footnote{For simplicity we converted the color video frames to gray-scale using the built-in Matlab function \texttt{\textbf{rgb2gray}}.} We set $p=10$ for storage space considerations of the batch algorithms. We also repeated the experiments using ten random initializations of $\bW_0$. 

The average running times over the ten initializations of $\bW_0$ (with standard deviations) on the two video sequences are shown in Table~\ref{tab:backgnd_time}. We notice that consistent messages are delivered by Table~\ref{tab:backgnd_time} as those in Table~\ref{tab:PSNR_K49}(b) and Table~\ref{tab:shadow_time}. Namely, the running times of our online algorithms are slightly longer than those of ORPCA but greatly shorter than the rest algorithms. Figure~\ref{fig:background} shows some (randomly sampled)  background  scenes and foreground objects separated by each algorithm. For each frame, the background was reconstructed using the methods introduced in Section~\ref{sec:img_denois}, and the foreground was directly recovered from the corresponding column in the (estimated) outlier matrix $\hatbR$. Since $\hatbR$ is absent in the formulation of ONMF, it is estimated using the difference between the original video frames and the recovered background scenes. From Figure~\ref{fig:background}, we observe that on both video sequences, our online algorithms are able to separate the foreground objects from the background fairly successfully,  with unnoticeable residues on both foreground and background. Compared with the other algorithms, the separation results of our online algorithms are comparable to their batch counterparts and slightly better than ORPCA on the {\tt Hall} sequence (with less residues in the recovered foreground). Due to the lack of robustness, the background scenes recovered by ONMF appear to be dark and the foreground objects cannot be separated from the background.
Thus again, in terms of the trade-off between the computational efficiency and the quality of foreground-background separation, our online algorithms achieve the best performances. Also, both of our online algorithms have similar performances on the foreground-background separation task.

\begin{table}[t]
\centering\footnotesize
\caption{Average running times (in seconds) of all the algorithms on video sequences (a) {\tt Hall} and (b) {\tt Escalator}.}\label{tab:backgnd_time}
\setlength\tabcolsep{1.5pt}
\subfloat[]{\begin{tabular}{|c|c|c|c|}\hline
Algorithms & Time (s) & Algorithms & Time (s)\\\hline
ONMF & $1525.57\pm 14.43$ & ORPCA & $166.85\pm 6.09$\\\hline
OADMM & $172.29\pm 5.64$ & OPGD & $178.46\pm 7.57$\\\hline
BADMM & $1280.06\pm 13.57$ & BPGD & $1167.95\pm 13.38$ \\\hline
\end{tabular}
}\\
\subfloat[]{\begin{tabular}{|c|c|c|c|}\hline
Algorithms & Time (s) & Algorithms & Time (s)\\\hline
ONMF & $1324.86\pm 11.45$ & ORPCA & $160.85\pm 5.03$\\\hline
OADMM & $166.83\pm 5.64$ & OPGD & $170.46\pm 5.91$\\\hline
BADMM & $898.47\pm 8.57$ & BPGD & $867.41\pm 9.35$ \\\hline
\end{tabular}
}
\end{table}

\begin{figure*}
\subfloat{\includegraphics[width=\textwidth,height=.36\textwidth]{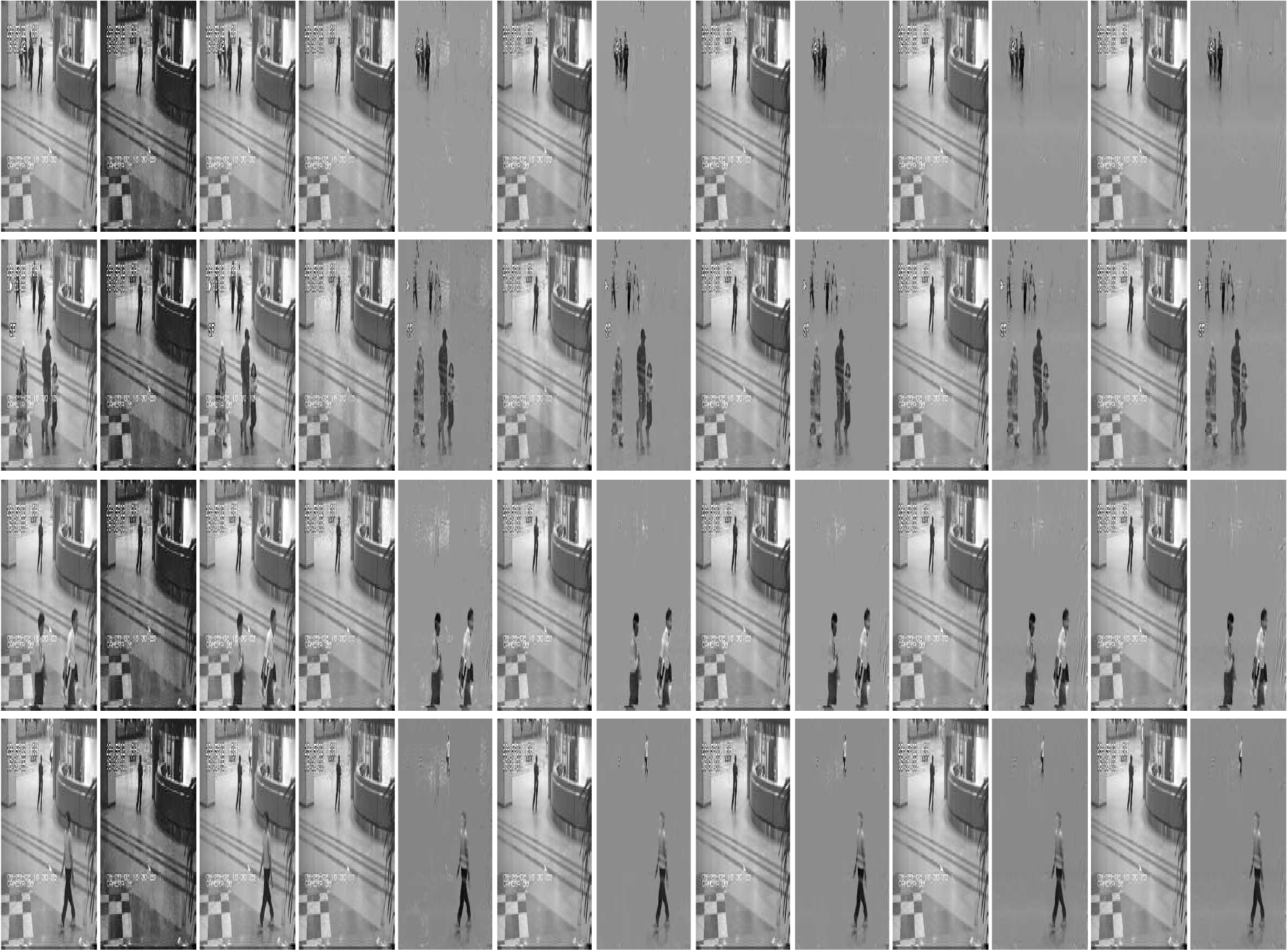}}\vspace{-.1cm}\\
\subfloat{\includegraphics[width=\textwidth,height=.36\textwidth]{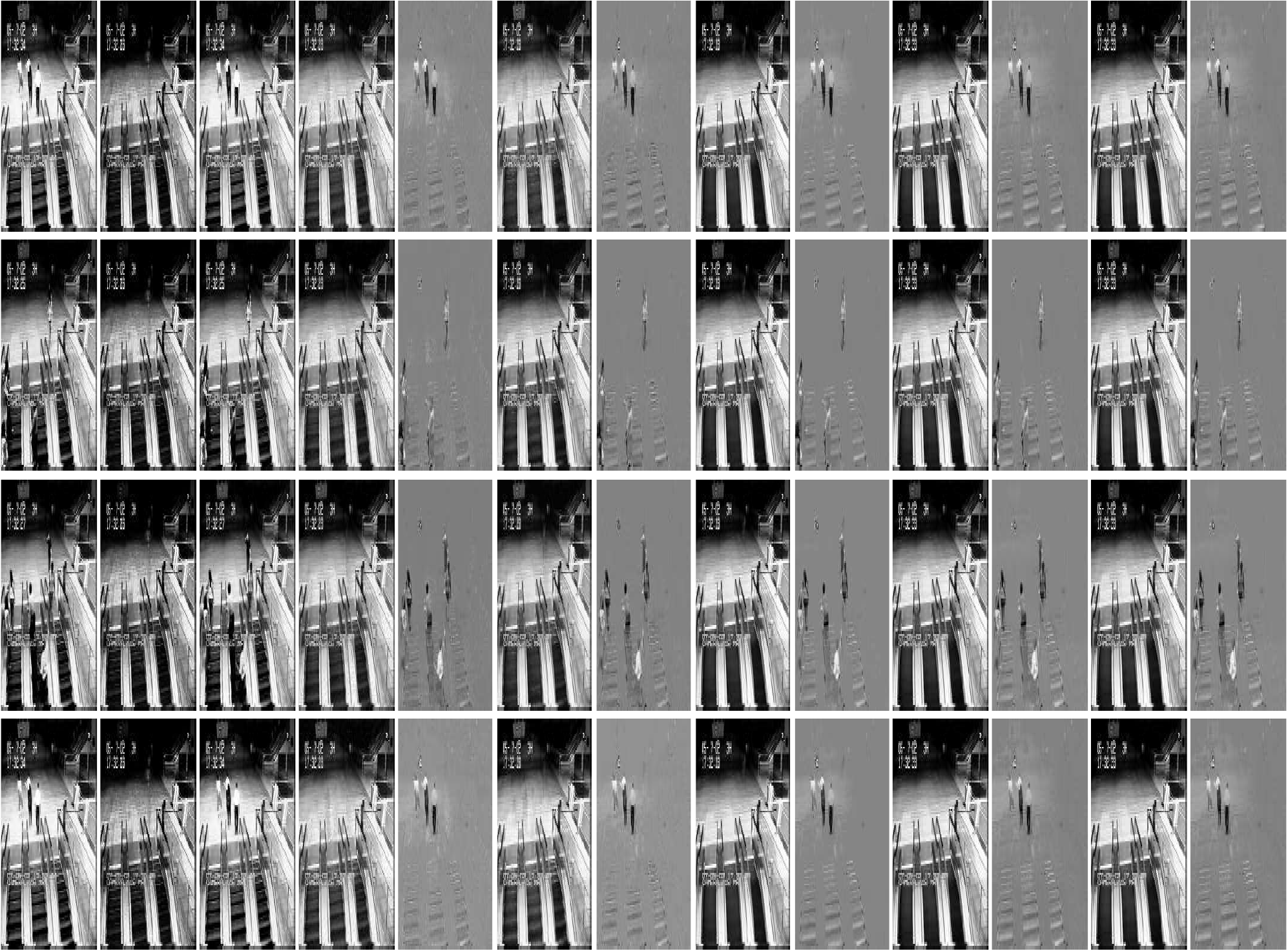}}\vspace{-.1cm}\\
$|\hspace{-.15cm}\xleftarrow{\makebox[.16cm]{}}\mbox{(a)}\xrightarrow{\makebox[.16cm]{}}\hspace{-.15cm}$
$|\hspace{-.15cm}\xleftarrow{\makebox[.87cm]{}}\mbox{(b)}\xrightarrow{\makebox[.87cm]{}}\hspace{-.15cm}|$
$\hspace{-.15cm}\xleftarrow{\makebox[.87cm]{}}\mbox{(c)}\xrightarrow{\makebox[.87cm]{}}\hspace{-.15cm}|$
$\hspace{-.15cm}\xleftarrow{\makebox[.87cm]{}}\mbox{(d)}\xrightarrow{\makebox[.87cm]{}}\hspace{-.15cm}|$
$\hspace{-.15cm}\xleftarrow{\makebox[.87cm]{}}\mbox{(e)}\xrightarrow{\makebox[.87cm]{}}\hspace{-.15cm}|$
$\hspace{-.15cm}\xleftarrow{\makebox[.87cm]{}}\mbox{(f)}\xrightarrow{\makebox[.87cm]{}}\hspace{-.15cm}|$
$\hspace{-.15cm}\xleftarrow{\makebox[.87cm]{}}\mbox{(g)}\xrightarrow{\makebox[.87cm]{}}\hspace{-.15cm}|$
\caption{Results of background-foreground separation by all the online and batch algorithms on the {\tt Hall} (upper) and {\tt Escalator} (lower) video sequences. The labels (a) to (g) denote the original video frames and results produced by ONMF, ORPCA, OADMM, OPGD, BADMM, BPGD respectively. For each algorithm, the left column denotes the background and the right column denotes the foreground (moving objects).}\label{fig:background}
\end{figure*}

\section{Conclusion and Future Work} \label{sec:con}
In this paper, we have developed online algorithms for NMF where the data samples are contaminated by outliers. We have shown that the proposed class of algorithms is robust to sparsely corrupted data samples and performs efficiently  on large datasets. 
Finally, we have proved almost sure convergence guarantees of the objective function as well as the sequence of basis matrices generated by the algorithms. 

 We hope to pursue the following three directions for further research. First, it would be interesting to extend the convergence analyses given i.i.d.\ data sequences to   weakly dependent data sequences (e.g.,   martingale  or auto-regressive processes). This is because real data (e.g., video and audio recordings) have correlations among adjacent data samples. 
 Second, it would be meaningful to consider the case where the fit-to-data term between $\bv$ and $\bW\bh+\br$, $d(\bv\Vert \bW\bh+\br)$ is not the squared $\ell_2$ loss,  e.g., the $\beta$-divergence \cite{Fev_09,Fev_11}. This is because in many scenarios, the observation noise is not Gaussian \cite{Fev_09} so other types of loss functions need to be used. Finally, as mentioned in Remark \ref{rmk:cons}, it would be meaningful to consider nonconvex constraint sets and/or regularizers on $\bW$, $\bh$ or $\br$. This is because the nonconvex constraints and regularizers often appear in real applications \cite{Bao_15}.
  
\subsubsection*{Acknowledgment}
The authors are grateful for the helpful clarifications and suggestions from Julien Mairal.

\bibliographystyle{IEEEtran}
\bibliography{RNMF_ref,ORNMF_ref,stoc_ref,math_opt,stat_ref,dataset}

\begin{thebibliography}{100}
\providecommand{\url}[1]{#1}
\csname url@samestyle\endcsname
\providecommand{\newblock}{\relax}
\providecommand{\bibinfo}[2]{#2}
\providecommand{\BIBentrySTDinterwordspacing}{\spaceskip=0pt\relax}
\providecommand{\BIBentryALTinterwordstretchfactor}{4}
\providecommand{\BIBentryALTinterwordspacing}{\spaceskip=\fontdimen2\font plus
\BIBentryALTinterwordstretchfactor\fontdimen3\font minus
  \fontdimen4\font\relax}
\providecommand{\BIBforeignlanguage}[2]{{%
\expandafter\ifx\csname l@#1\endcsname\relax
\typeout{** WARNING: IEEEtran.bst: No hyphenation pattern has been}%
\typeout{** loaded for the language `#1'. Using the pattern for}%
\typeout{** the default language instead.}%
\else
\language=\csname l@#1\endcsname
\fi
#2}}
\providecommand{\BIBdecl}{\relax}
\BIBdecl

\bibitem{Zhao_16}
R.~Zhao and V.~Y.~F. Tan, ``Online nonnegative matrix factorization with
  outliers,'' in \emph{Proc.\ ICASSP}, Shanghai, China, Mar. 2016, pp.
  2662--2666.

\bibitem{Tsuge_01}
S.~Tsuge, M.~Shishibori, S.~Kuroiwa, and K.~Kita, ``Dimensionality reduction
  using non-negative matrix factorization for information retrieval,'' in
  \emph{Proc.\ SMC}, vol.~2, Tucson, Arizona, USA, Oct. 2001, pp. 960--965.

\bibitem{Lee_99}
D.~D. Lee and H.~S. Seung, ``Learning the parts of objects by nonnegative
  matrix factorization,'' \emph{Nature}, vol. 401, pp. 788--791, October 1999.

\bibitem{Lee_00}
------, ``Algorithms for non-negative matrix factorization,'' in \emph{Proc.\
  NIPS}, Denver, USA, Dec. 2000, pp. 556--562.

\bibitem{Kim_08a}
J.~Kim and H.~Park, ``Toward faster nonnegative matrix factorization: A new
  algorithm and comparisons,'' in \emph{Proc.\ ICDM}, Pisa, Italy, Dec. 2008,
  pp. 353--362.

\bibitem{Lin_07a}
C.-J. Lin, ``Projected gradient methods for nonnegative matrix factorization,''
  \emph{Neural Comput.}, vol.~19, no.~10, pp. 2756--2779, Oct. 2007.

\bibitem{Kim_08b}
H.~Kim and H.~Park, ``Nonnegative matrix factorization based on alternating
  nonnegativity constrained least squares and active set method,'' \emph{SIAM
  J. Matrix Anal. A.}, vol.~30, no.~2, pp. 713–--730, 2008.

\bibitem{Xu_12}
Y.~Xu, W.~Yin, Z.~Wen, and Y.~Zhang, ``An alternating direction algorithm for
  matrix completion with nonnegative factors,'' \emph{Frontiers Math. China},
  vol.~7, pp. 365–--384, 2012.

\bibitem{Ding_05}
C.~Ding, X.~He, and H.~D. Simon, ``On the equivalence of nonnegative matrix
  factorization and spectral clustering,'' in \emph{Proc.\ SDM}, Newport Beach,
  California, USA, Apr. 2005, pp. 606--610.

\bibitem{Yuan_15}
Y.~Yuan, Y.~Feng, and X.~Lu, ``Projection-based nmf for hyperspectral
  unmixing,'' \emph{IEEE J. Sel. Top. Appl. Earth Obs. Remote Sens.}, vol.~8,
  no.~6, pp. 2632--2643, Jun. 2015.

\bibitem{Durrieu_11}
J.~L. Durrieu, B.~David, and G.~Richard, ``A musically motivated mid-level
  representation for pitch estimation and musical audio source separation,''
  \emph{IEEE J. Sel. Top. Signal Process.}, vol.~5, no.~6, pp. 1180--1191, Oct.
  2011.

\bibitem{Cao_07}
B.~Cao, D.~Shen, J.-T. Sun, X.~Wang, Q.~Yang, and Z.~Chen, ``Detect and track
  latent factors with online nonnegative matrix factorization,'' in
  \emph{Proc.\ IJCAI}, Hyderabad, India, Jan. 2007, pp. 2689--2694.

\bibitem{Zhang_11}
L.~Zhang, Z.~Chen, M.~Zheng, and X.~He, ``Robust nonnegative matrix
  factorization,'' \emph{Frontiers Elect. Electron. Eng. China}, vol.~6, no.~2,
  pp. 192--200, 2011.

\bibitem{RAD_15}
Netflix, ``Rad - outlier detection on big data,''
  \url{http://techblog.netflix.com/2015/02/rad-outlier-detection-on-big-data.html},
  2015.

\bibitem{Bucak_08}
S.~S. Bucak and B.~Gunsel, ``Incremental subspace learning via non-negative
  matrix factorization,'' \emph{Pattern Recogn.}, vol.~42, no.~5, pp.
  788–--797, May 2009.

\bibitem{Guan_12}
N.~Guan, D.~Tao, Z.~Luo, and B.~Yuan, ``Online nonnegative matrix factorization
  with robust stochastic approximation,'' \emph{IEEE Trans. Neural Netw. Learn.
  Syst.}, vol.~23, no.~7, pp. 1087--1099, Jul. 2012.

\bibitem{Lef_11}
A.~Lef\`evre, F.~Bach, and C.~F\'evotte, ``Online algorithms for nonnegative
  matrix factorization with the itakura-saito divergence,'' in \emph{Proc.\
  WASPAA}, New Paltz, New York, USA, Oct 2011, pp. 313--316.

\bibitem{WangFei_11}
F.~Wang, C.~Tan, A.~C. Konig, and P.~Li, ``Efficient document clustering via
  online nonnegative matrix factorizations,'' in \emph{Proc.\ SDM}, Mesa,
  Arizona, USA, Apr. 2011, pp. 908--919.

\bibitem{WuShen_14}
Y.~Wu, B.~Shen, and H.~Ling, ``Visual tracking via online nonnegative matrix
  factorization,'' \emph{IEEE Trans. Circuits Syst. Video Technol.}, vol.~24,
  no.~3, pp. 374--383, Mar. 2014.

\bibitem{Liu_10}
C.~Liu, H.~chih Yang, J.~Fan, L.-W. He, and Y.-M. Wang, ``Distributed
  nonnegative matrix factorization for web-scale dyadic data analysis on
  mapreduce,'' in \emph{Proc.\ WWW}, Raleigh, North Carolina, USA, Apr. 2010,
  pp. 681--690.

\bibitem{Gemulla_11}
R.~Gemulla, P.~J. Haas, Y.~Sismanis, C.~Teflioudi, and F.~Makari, ``Large-scale
  matrix factorization with distributed stochastic gradient descent,'' in
  \emph{Proc.\ KDD}, San Diego, California, USA, Aug. 2011, pp. 69--77.

\bibitem{Du_14}
S.~S. Du, Y.~Liu, B.~Chen, and L.~Li, ``Maxios: Large scale nonnegative matrix
  factorization for collaborative filtering,'' in \emph{Proc.\ NIPS-DMLMC},
  Montr\'eal, Canada, Dec. 2014.

\bibitem{Chen_15}
J.~Chen, Z.~J. Towfic, and A.~H. Sayed, ``Dictionary learning over distributed
  models,'' \emph{IEEE Trans. Signal Process.}, vol.~63, no.~4, pp. 1001--1016,
  2015.

\bibitem{Halko_11}
N.~Halko, P.~G. Martinsson, and J.~A. Tropp, ``Finding structure with
  randomness: Probabilistic algorithms for constructing approximate matrix
  decompositions,'' \emph{SIAM Rev.}, vol.~53, no.~2, pp. 217--–288, 2011.

\bibitem{Tepper_16}
M.~Tepper and G.~Sapiro, ``Compressed nonnegative matrix factorization is fast
  and accurate,'' \emph{IEEE Trans. Signal Process.}, vol.~64, no.~9, pp.
  2269--2283, May 2016.

\bibitem{Gao_15}
H.~Gao, F.~Nie, W.~Cai, and H.~Huang, ``Robust capped norm nonnegative matrix
  factorization: Capped norm {NMF},'' in \emph{Proc.\ CIKM}, Melbourne,
  Australia, Oct. 2015, pp. 871--880.

\bibitem{Huang_14}
J.~Huang, F.~Nie, H.~Huang, and C.~Ding, ``Robust manifold nonnegative matrix
  factorization,'' \emph{ACM Trans. Knowl. Discov. Data}, vol.~8, no.~3, pp.
  1--21, Jun. 2014.

\bibitem{Yang_13}
S.~Yang, C.~Hou, C.~Zhang, and Y.~Wu, ``Robust non-negative matrix
  factorization via joint sparse and graph regularization for transfer
  learning,'' \emph{Neural Comput. and Appl.}, vol.~23, no.~2, pp. 541--559,
  2013.

\bibitem{Du_12}
L.~Du, X.~Li, and Y.-D. Shen, ``Robust nonnegative matrix factorization via
  half-quadratic minimization,'' in \emph{Proc.\ ICDM}, Brussels, Belgium, Dec.
  2012, pp. 201--210.

\bibitem{Ding_12}
C.~Ding and D.~Kong, ``Nonnegative matrix factorization using a robust error
  function,'' in \emph{Proc.\ ICASSP}, Kyoto, Japan, Mar. 2012, pp. 2033--2036.

\bibitem{Ding_11}
D.~Kong, C.~Ding, and H.~Huang, ``Robust nonnegative matrix factorization using
  $\ell_{21}$-norm,'' in \emph{Proc.\ CIKM}, Glasgow, Scotland, UK, Oct. 2011,
  pp. 673--682.

\bibitem{Cichoc_11}
A.~Cichocki, S.~Cruces, and S.-i. Amari, ``Generalized alpha-beta divergences
  and their application to robust nonnegative matrix factorization,''
  \emph{Entropy}, vol.~13, no.~1, pp. 134--170, 2011.

\bibitem{Kasi_12}
S.~P. Kasiviswanathan, H.~Wang, A.~Banerjee, and P.~Melville, ``Online
  $\ell_1$-dictionary learning with application to novel document detection,''
  in \emph{Proc.\ NIPS}, Lake Tahoe, USA, Dec. 2012, pp. 2258--2266.

\bibitem{Fev_14}
C.~F\'evotte and N.~Dobigeon, ``Nonlinear hyperspectral unmixing with robust
  nonnegative matrix factorization,'' \emph{IEEE Trans. Image Process.},
  vol.~24, no.~12, pp. 4810--4819, 2015.

\bibitem{Shen_14}
B.~Shen, B.~Liu, Q.~Wang, and R.~Ji, ``Robust nonnegative matrix factorization
  via $l_1$ norm regularization by multiplicative updating rules,'' in
  \emph{Proc.\ ICIP}, Paris, France, Oct. 2014, pp. 5282--5286.

\bibitem{Feng_13}
J.~Feng, H.~Xu, and S.~Yan, ``Online robust {PCA} via stochastic
  optimization,'' in \emph{Proc.\ NIPS}, Lake Tahoe, USA, Dec. 2013, pp.
  404--412.

\bibitem{Shen_14b}
J.~Shen, H.~Xu, and P.~Li, ``Online optimization for max-norm regularization,''
  in \emph{Proc.\ NIPS}, Montr\'eal, Canada, Dec. 2014, pp. 1718--1726.

\bibitem{Wang_13}
N.~Wang, J.~Wang, and D.-Y. Yeung, ``Online robust non-negative dictionary
  learning for visual tracking,'' in \emph{Proc.\ ICCV}, Sydney, Australia,
  Dec. 2013, pp. 657--664.

\bibitem{Rock_70}
R.~T. Rockafellar, \emph{Convex analysis}.\hskip 1em plus 0.5em minus
  0.4em\relax Princeton University Press, 1970.

\bibitem{Vaart_00}
A.~W. van~der Vaart, \emph{Asymptotic Statistics}.\hskip 1em plus 0.5em minus
  0.4em\relax Cambridge Press, 2000.

\bibitem{Met_82}
M.~M\'{e}tivier, \emph{Semimartingales: A Course on Stochastic
  Processes}.\hskip 1em plus 0.5em minus 0.4em\relax de Gruyter, 1982.

\bibitem{Paa_94}
P.~Paatero and U.~Tapper, ``Positive matrix factorization: A non-negative
  factor model with optimal utilization of error estimates of data values,''
  \emph{Environmetrics}, vol.~5, no.~2, pp. 111--126, 1994.

\bibitem{Hoyer_04}
P.~O. Hoyer, ``Non-negative matrix factorization with sparseness constraints,''
  \emph{J. Mach. Learn. Res.}, vol.~5, pp. 1457--1469, Dec. 2004.

\bibitem{Mairal_10}
J.~Mairal, F.~Bach, J.~Ponce, and G.~Sapiro, ``Online learning for matrix
  factorization and sparse coding,'' \emph{J. Mach. Learn. Res.}, vol.~11, pp.
  19--60, Mar 2010.

\bibitem{WangDong_13}
D.~Wang and H.~Lu, ``On-line learning parts-based representation via
  incremental orthogonal projective non-negative matrix factorization,''
  \emph{Signal Process.}, vol.~93, no.~6, pp. 1608 -- 1623, 2013.

\bibitem{Xing_13}
J.~Xing, J.~Gao, B.~Li, W.~Hu, and S.~Yan, ``Robust object tracking with online
  multi-lifespan dictionary learning,'' in \emph{Proc.\ ICCV}, Sydney,
  Australia, Dec. 2013, pp. 665--672.

\bibitem{Zhang_15}
S.~Zhang, S.~Kasiviswanathan, P.~C. Yuen, and M.~Harandi, ``Online dictionary
  learning on symmetric positive definite manifolds with vision applications,''
  in \emph{Proc.\ AAAI}, Austin, Texas, Jan. 2015, pp. 3165--3173.

\bibitem{Zhang_15b}
X.~Zhang, N.~Guan, D.~Tao, X.~Qiu, and Z.~Luo, ``Online multi-modal robust
  non-negative dictionary learning for visual tracking,'' \emph{PLoS ONE},
  vol.~10, no.~5, pp. 1--17, 2015.

\bibitem{Zhan_16}
J.~Zhan, B.~Lois, and N.~Vaswani, ``Online (and offline) robust {PCA}: Novel
  algorithms and performance guarantees,'' arXiv:1601.07985, 2016.

\bibitem{Lois_15}
B.~Lois and N.~Vaswani, ``Online matrix completion and online robust {PCA},''
  arXiv:1503.03525, 2015.

\bibitem{Guo_15}
X.~Guo, ``Online robust low rank matrix recovery,'' in \emph{Proc.\ IJCAI},
  Buenos Aires, Argentina, Jul. 2015, pp. 3540--3546.

\bibitem{Spre_12}
P.~Sprechmann, A.~Bronstein, and G.~Sapiro, ``Real-time online singing voice
  separation from monaural recordings using robust low-rank modeling,'' in
  \emph{Proc.\ ISMIR}, Porto, Portugal, Oct. 2012, pp. 67--72.

\bibitem{Spre_12b}
P.~Sprechmann, A.~M. Bronstein, and G.~Sapiro, ``Learning efficient sparse and
  low rank models,'' arXiv:1212.3631, 2012.

\bibitem{Tsypkin_71}
Y.~Tsypkin, \emph{Adaptation and Learning in automatic systems}.\hskip 1em plus
  0.5em minus 0.4em\relax Academic Press, New York, 1971.

\bibitem{Bottou_98}
L.~Bottou, ``Online algorithms and stochastic approximations,'' in \emph{Online
  Learning and Neural Networks}.\hskip 1em plus 0.5em minus 0.4em\relax
  Cambridge University Press, 1998.

\bibitem{Bottou_08}
L.~Bottou and O.~Bousquet, ``The tradeoffs of large scale learning,'' in
  \emph{Proc.\ NIPS}, Vancouver, Canada, Dec. 2008, pp. 161--168.

\bibitem{Shap_07}
A.~Shapiro and A.~Philpott, ``A tutorial on stochastic programming,''
  http://stoprog.org/sites/default/files/SPTutorial/TutorialSP.pdf, 2007.

\bibitem{Lin_07b}
C.-J. Lin, ``On the convergence of multiplicative update algorithms for
  nonnegative matrix factorization,'' \emph{IEEE Trans. Neural Netw.}, vol.~18,
  no.~6, pp. 1589--1596, 2007.

\bibitem{Mensch_16}
A.~Mensch, J.~Mairal, B.~Thirion, and G.~Varoquaux, ``Dictionary learning for
  massive matrix factorization,'' in \emph{Proc.\ ICML}, New York, USA, Jul.
  2016.

\bibitem{Mairal_13}
J.~Mairal, ``Stochastic majorization-minimization algorithms for large-scale
  optimization,'' in \emph{Proc.\ NIPS}, Lake Tahoe, USA, Dec. 2013, pp.
  2283--2291.

\bibitem{Raza_16}
M.~Razaviyayn, M.~Sanjabi, and Z.-Q. Luo, ``A stochastic successive
  minimization method for nonsmooth nonconvex optimization with applications to
  transceiver design in wireless communication networks,'' \emph{Math.
  Program.}, vol. 157, no.~2, pp. 515--545, 2016.

\bibitem{Sun_14}
D.~Sun and C.~Fevotte, ``Alternating direction method of multipliers for
  non-negative matrix factorization with the beta-divergence,'' in \emph{Proc.\
  ICASSP}, Florence, Italy, May 2014, pp. 6201--6205.

\bibitem{Raza_13}
M.~Razaviyayn, M.~Hong, and Z.-Q. Luo, ``A unified convergence analysis of
  block successive minimization methods for nonsmooth optimization,''
  \emph{SIAM J. Optim.}, vol.~23, no.~2, pp. 1126–--1153, 2013.

\bibitem{Hong_16}
M.~Hong, M.~Razaviyayn, Z.~Q. Luo, and J.~S. Pang, ``A unified algorithmic
  framework for block-structured optimization involving big data: With
  applications in machine learning and signal processing,'' \emph{IEEE Signal
  Process. Mag.}, vol.~33, no.~1, pp. 57--77, Jan. 2016.

\bibitem{ZhangBox_15}
H.~Zhang and L.~Cheng, ``Projected shrinkage algorithm for box-constrained
  $\ell_1$-minimization,'' \emph{Optim. Lett.}, vol.~1, no.~1, pp. 1–--16,
  2015.

\bibitem{Tseng_08}
P.~Tseng, ``On accelerated proximal gradient methods for convex-concave
  optimization,'' Technical report, University of Washington, Seattle, 2008.

\bibitem{Nest_13}
Y.~Nesterov, ``Gradient methods for minimizing composite functions,''
  \emph{Math. Program.}, vol. 140, no.~1, pp. 125--161, 2013.

\bibitem{Vavasis_09}
S.~A. Vavasis, ``On the complexity of nonnegative matrix factorization,''
  \emph{SIAM J. Optim.}, vol.~20, no.~3, pp. 1364--–1377, 2009.

\bibitem{Haji_16}
D.~Hajinezhad, T.~H. Chang, X.~Wang, Q.~Shi, and M.~Hong, ``Nonnegative matrix
  factorization using admm: Algorithm and convergence analysis,'' in
  \emph{Proc.\ ICASSP}, Shanghai, China, Mar. 2016, pp. 4742--4746.

\bibitem{Donoho_04}
D.~Donoho and V.~Stodden, ``When does non-negative matrix factorization give
  correct decomposition into parts?'' in \emph{Proc.\ NIPS}, Vancouver, Canada,
  Dec. 2004, pp. 1141--1148.

\bibitem{Cohen_93}
J.~E. Cohen and U.~G. Rothblum, ``Nonnegative ranks, decompositions, and
  factorizations of nonnegative matrices,'' \emph{Linear Algebra Appl.}, vol.
  190, no.~1, pp. 149 -- 168, 1993.

\bibitem{Arora_12}
S.~Arora, R.~Ge, R.~Kannan, and A.~Moitra, ``Computing a nonnegative matrix
  factorization -- provably,'' in \emph{Proc.\ STOC}, New York, New York, USA,
  May 2012, pp. 145--162.

\bibitem{Gillis_12}
N.~Gillis, ``Sparse and unique nonnegative matrix factorization through data
  preprocessing,'' \emph{J. Mach. Learn. Res.}, vol.~13, no.~1, pp. 3349--3386,
  Nov. 2012.

\bibitem{Vanda_16}
A.~Vandaele, N.~Gillis, F.~Glineur, and D.~Tuyttens, ``Heuristics for exact
  nonnegative matrix factorization,'' \emph{J. Glob. Optim.}, vol.~65, no.~2,
  pp. 369--400, 2016.

\bibitem{Nemi_09}
A.~Nemirovski, A.~Juditsky, G.~Lan, and A.~Shapiro, ``Robust stochastic
  approximation approach to stochastic programming,'' \emph{SIAM J. Optimiz.},
  vol.~19, no.~4, pp. 1574--1609, Jan. 2009.

\bibitem{Duchi_08}
J.~Duchi, S.~Shalev-Shwartz, Y.~Singer, and T.~Chandra, ``Efficient projections
  onto the l1-ball for learning in high dimensions,'' in \emph{Proc.\ ICML},
  Helsinki, Finland, Jul. 2008, pp. 272--279.

\bibitem{WangSimp_13}
W.~Wang and M.~A. Carreira-Perpi\"{n}\'{a}n, ``Projection onto the probability
  simplex: An efficient algorithm with a simple proof, and an application,''
  arXiv:1309.1541, 2013.

\bibitem{Witten_09}
D.~M. Witten, R.~Tibshirani, and T.~Hastie, ``A penalized matrix decomposition,
  with applications to sparse principal components and canonical correlation
  analysis,'' \emph{Biostatistics}, vol.~10, no.~3, pp. 515--534, 2009.

\bibitem{Zou_05}
H.~Zou and T.~Hastie, ``Regularization and variable selection via the elastic
  net,'' \emph{J. Roy. Statist. Soc. Ser. B}, vol.~67, no.~2, pp. 301--320,
  2005.

\bibitem{Sind_12}
V.~Sindhwani and A.~Ghoting, ``Large-scale distributed non-negative sparse
  coding and sparse dictionary learning,'' in \emph{Proc.\ KDD}, Beijing,
  China, Aug. 2012, pp. 489--497.

\bibitem{Duchi_09}
\BIBentryALTinterwordspacing
J.~Duchi, ``Elastic net projections,'' 2009. [Online]. Available:
  \url{http://stanford.edu/~jduchi/projects/proj_elastic_net.pdf}
\BIBentrySTDinterwordspacing

\bibitem{Kim_12}
J.~Kim, R.~Monteiro, and H.~Park, ``Group sparsity in nonnegative matrix
  factorization,'' in \emph{Proc.\ SDM}, Anaheim, California, USA, Apr. 2012,
  pp. 851--862.

\bibitem{Liu_11}
X.~Liu, H.~Lu, and H.~Gu, ``Group sparse non-negative matrix factorization for
  multi-manifold learning,'' in \emph{Proc.\ BMVC}, Dundee, UK, Aug. 2011, pp.
  1--11.

\bibitem{Friedman_10}
J.~Friedman, T.~Hastie, and R.~Tibshirani, ``A note on the group lasso and a
  sparse group lasso,'' arXiv:1001.0736, 2010.

\bibitem{Simon_13}
N.~Simon, J.~Friedman, T.~Hastie, and R.~Tibshirani, ``A sparse-group lasso,''
  \emph{J. Comput. Graph. Statist.}, vol.~22, pp. 231--245, 2013.

\bibitem{Yuan_06}
M.~Yuan and Y.~Lin, ``Model selection and estimation in regression with grouped
  variables,'' \emph{J. Roy. Statist. Soc. Ser. B}, vol.~68, no.~1, pp. 49--67,
  2006.

\bibitem{YaleB_01}
A.~S. Georghiades, P.~N. Belhumeur, and D.~J. Kriegman, ``From few to many:
  illumination cone models for face recognition under variable lighting and
  pose,'' \emph{IEEE Trans. Pattern Anal. Mach. Intell.}, vol.~23, no.~6, pp.
  643--660, Jun. 2001.

\bibitem{i2r_04}
L.~Li, W.~Huang, I.~Y.-H. Gu, and Q.~Tian, ``Statistical modeling of complex
  backgrounds for foreground object detection,'' \emph{IEEE Trans. Image
  Process.}, vol.~13, no.~11, pp. 1459--1472, Nov. 2004.

\bibitem{Tan_13a}
V.~Y.~F. Tan and C.~F\'evotte, ``Automatic relevance determination in
  nonnegative matrix factorization with the $\beta$-divergence,'' \emph{IEEE
  Trans. Pattern Anal. Mach. Intell.}, vol.~35, no.~7, pp. 1592--1605, Jul.
  2013.

\bibitem{Bishop_98}
C.~M. Bishop, ``Bayesian {PCA},'' in \emph{Proc.\ NIPS}, Denver, USA, Jan.
  1999, pp. 382--–388.

\bibitem{Boyd_11}
S.~Boyd, N.~Parikh, E.~Chu, B.~Peleato, and J.~Eckstein, ``Distributed
  optimization and statistical learning via the alternating direction method of
  multipliers,'' \emph{Found. Trends Mach. Learn.}, vol.~3, no.~1, pp. 1--122,
  Jan. 2011.

\bibitem{PSNR_98}
T.~Veldhuizen, ``Measures of image quality,''
  \url{http://homepages.inf.ed.ac.uk/rbf/CVonline/LOCAL_COPIES/VELDHUIZEN/node18.html},
  1998.

\bibitem{Candes_11}
E.~Cand\`es, X.~Li, Y.~Ma, and J.~Wright, ``Robust principal component
  analysis?'' \emph{J. ACM}, vol.~58, no.~3, pp. 1--37, Jun. 2011.

\bibitem{Netra_14}
P.~Netrapalli, N.~U~N, S.~Sanghavi, A.~Anandkumar, and P.~Jain, ``Non-convex
  robust pca,'' in \emph{Proc.\ NIPS}, Montr\'eal, Canada, Dec. 2014, pp.
  1107--1115.

\bibitem{Fev_09}
C.~F\'evotte, N.~Bertin, and J.-L. Durrieu, ``Nonnegative matrix factorization
  with the {I}takura-{S}aito divergence. {W}ith application to music
  analysis,'' \emph{Neural Comput.}, vol.~21, no.~3, pp. 793--830, Mar. 2009.

\bibitem{Fev_11}
C.~F\'evotte and J.~Idier, ``Algorithms for nonnegative matrix factorization
  with the beta-divergence,'' \emph{Neural Comput.}, vol.~23, no.~9, pp.
  2421--2456, 2011.

\bibitem{Bao_15}
C.~Bao, H.~Ji, Y.~Quan, and Z.~Shen, ``Dictionary learning for sparse coding:
  Algorithms and analysis,'' \emph{IEEE Trans. Pattern Anal. Mach. Intell.},
  vol.~38, no.~7, pp. 1356--1369, July 2015.

\bibitem{Bon_98}
J.~F. Bonnans and A.~Shapiro, ``Optimization problems with perturbations: A
  guided tour,'' \emph{{SIAM Rev.}}, no.~2, pp. 228–--264, 1998.

\bibitem{Syd_05}
K.~Sydsaeter, P.~Hammond, A.~Seierstad, and A.~Strom, \emph{{Further
  Mathematics for Economic Analysis}}.\hskip 1em plus 0.5em minus 0.4em\relax
  Pearson, 2005.

\bibitem{Bill_86}
P.~Billingsley, \emph{Probability and Measure}, 2nd~ed.\hskip 1em plus 0.5em
  minus 0.4em\relax John Wiley \& Sons, 1986.

\bibitem{Fisk_65}
D.~L. Fisk, ``Quasi-martingales,'' \emph{Trans. Amer. Math. Soc.}, vol. 120,
  no.~3, pp. 369--89, 1965.

\end{thebibliography}

\newpage
\onecolumn
\renewcommand{\thedefinition}{S-\arabic{definition}}
\renewcommand{\thelemma}{S-\arabic{lemma}}
\renewcommand{\thecorollary}{S-\arabic{corollary}}
\renewcommand{\theequation}{S-\arabic{equation}}
\renewcommand{\thesection}{S-\arabic{section}}
\setcounter{section}{0}
\setcounter{equation}{0}
\setcounter{lemma}{0}
\setcounter{corollary}{0}
\setcounter{definition}{0}

{\huge\quad\quad\hspace{4cm} Supplemental Material}\\[.1cm]

In this supplemental material, the indices of all the sections, definitions, lemmas and equations are prepended with an `S' to distinguish those in  the main text. The organization of this article is as follows. In section~\ref{sec:deriv_ADMM}, we drive the ADMM algorithms presented in Section~\ref{sec:ADMM}. In Section~\ref{sec:BRNMF_algo}, we extend our two solvers (based on PGD and ADMM) in Section~\ref{sec:algo} 
to the {\em batch} NMF problems with outliers. Then, we provide detailed proofs of the theorems and lemmas in Section~\ref{sec:conv_analysis} in Section~\ref{sec:regularity} to \ref{sec:Lips_b}. The technical lemmas used in the algorithm derivation (Section~\ref{sec:algo}) and the convergence analysis (Section~\ref{sec:conv_analysis}) are shown in Section~\ref{sec:tech_lemma}. Finally, we show additional experiment results in Section~\ref{sec:exp_res} to supplement those in Section~\ref{sec:numericals}. 
The {\em finite} constants $c$, $c_1$ and $c_2$ are used repeatedly in different sections, and their meanings depend on the context. 

\section{Derivation of the ADMM Algorithms in Section~\ref{sec:ADMM}} \label{sec:deriv_ADMM}
\subsection{Algorithms for \eqref{eq:min_hr}}
Minimizing $\bh$ and $\br$ amounts to solving the following two unconstrained problems
\begin{align}
&\min_\bh \quad \frac{1}{2}\norm{\bW\bh+(\br-\bv)}_2^2 + \balpha^T(\bh-\bu) + \frac{\rho_1}{2}\norm{\bh-\bu}_2^2\nn\\
\Longleftrightarrow\quad & \min_\bh \quad \frac{1}{2}\bh^T(\bW^T\bW+\rho_1\bI)\bh + \left(\bW^T(\br-\bv)-\rho_1\bu+\balpha\right)^T\bh\label{eq:ADMM_h}
\end{align}
and
\begin{align}
&\min_\br \quad \frac{1}{2}\norm{(\bv-\bW\bh)-\br}_2^2 + \frac{\trho_2}{2}\norm{\br-\bq}^2_2 + \bbeta^T(\br-\bq) + \lambda\norm{\br}_1\nn\\
\Longleftrightarrow\quad & \min_\br \quad \frac{1+\trho_2}{2} \br^T\br + (\bbeta-\bv+\bW\bh-\trho_2\bq)^T\br + \lambda\norm{\br}_1\nn\\
\Longleftrightarrow\quad & \min_\br \quad \frac{1+\trho_2}{2} \norm{\br - \frac{\trho_2\bq+\bv-\bbeta-\bW\bh}{1+\trho_2}}_2^2+\lambda\norm{\br}_1.\label{eq:ADMM_r}
\end{align}
We notice that \eqref{eq:ADMM_h} is a standard strongly convex quadratic minimization problem and \eqref{eq:ADMM_r} is a standard proximal minimization problem with $\ell_1$ norm, thus the closed-form optimal solutions for \eqref{eq:ADMM_h} and \eqref{eq:ADMM_r} are 
\begin{align*}
\bh^* &= (\bW^T\bW+\rho_1\bI)^{-1}(\bW^T(\bv-\br)+\rho_1\bu-\balpha)\\
\br^* &= \calS_{\lambda/(1+\trho_2)}\left(\frac{\trho_2\bq+\bv-\bbeta-\bW\bh}{1+\trho_2}\right) = \frac{\calS_\lambda(\trho_2\bq+\bv-\bbeta-\bW\bh)}{1+\trho_2}.
\end{align*}

Minimizing $\bu$ and $\bq$ amounts to solving the following two constrained problems
\begin{align}
\min_{\bu\ge 0} \quad &\frac{\rho_1}{2}\norm{\bh-\bu}_2^2 - \balpha^T\bu\\
\min_{\norm{\bq}_\infty\le M} \quad &\frac{\trho_2}{2}\norm{\br-\bq}_2^2 -\bbeta^T\bq.
\end{align}
Since both constraints  $\bu\ge 0$ and $\norm{\bq}_\infty\le M$ are separable across coordinates, we can simply solve the unconstrained  quadratic minimization problems and then project the optimal solutions to the feasible sets. 

\subsection{Algorithms for \eqref{eq:min_W_tr}}
Minimizing $\bW$ amounts to solving the unconstrained quadratic minimization problem
\begin{align*}
&\min_\bW \quad  \frac{1}{2}\tr\left(\bW^T\bW\bA_t\right) - \tr\left(\bW^T\bB_t\right) + \lrangle{\bD}{\bW-\bQ} + \frac{\trho_3}{2}\norm{\bW-\bQ}_F^2\\
\Longleftrightarrow\quad &\min_\bW \quad \frac{1}{2}\tr\left(\bW\left(\bA_t+\trho_3\bI\right)\bW^T\right) - \tr\left(\left(\bB_t-\bD+\trho_3\bQ\right)^T\bW\right)
\end{align*}
so 
\begin{equation}
\bW^* = \left(\bB_t-\bD+\trho_3\bQ\right)\left(\bA_t+\trho_3\bI\right)^{-1}.
\end{equation}
Minimizing $\bQ$ amounts to solving the constrained quadratic minimization problem
\begin{align*}
&\min_{\bQ\in\calC} \quad \frac{\trho_3}{2}\norm{\bW-\bQ}_F^2 - \lrangle{\bD}{\bQ}\\
\Longleftrightarrow\quad &\min_{\bQ\in\calC} \quad \frac{\trho_3}{2}\norm{\bQ - (\bW+\bD/\trho_3)}_F^2.
\end{align*}
Then we have
\begin{equation}
\bQ^* = \calP_\calC (\bW+\bD/\trho_3).
\end{equation}

\section{Extension to the Batch NMF Problem with Outliers}\label{sec:BRNMF_algo}

\subsection{Problem Formulation}
As usual, we denote the data matrix (with outliers) as $\bV$, basis matrix as $\bW$, coefficient matrix as $\bH$ and outlier matrix as $\bR$. The number of total data samples is denoted as $N$. Then, the batch counterpart for the online NMF problem with outliers can be formulated as follows
\begin{align*}
&\min \quad \frac{1}{2}\norm{\bV-\bW\bH-\bR}_F^2 + \lambda\norm{\bR}_{1,1}\\
&\st \quad \bH\in\bbR_+^{K\times N}, \bR\in\tilcalR, \bW\in\calC, \numberthis \label{eq:brnmf}
\end{align*}
where $\tilcalR = \{\bR\in\bbR^{F\times N}\,|\, \abs{r_{i,j}} \le M, \forall\,(i,j)\in[F]\times[N]\}$.

\subsection{Notations}
In the sequel, we overload soft-thresholding operators $\tilcalS_{\lambda,M}$ and $\calS_\lambda$. When these two operators are applied to matrices, each operator denotes entrywise soft-thresholding. The updated variables are denoted with superscripts `$+$'.

\subsection{Batch algorithm based on PGD (BPGD)}
Based on the principle of block coordinate descent, we update $\bH$, $\bR$ and $\bW$ sequentially as follows.
\begin{align}
\bH^+ &:= \calP_+\left(\bH - \eta_1(\bW)\bW^T(\bW\bH+\bR-\bV)\right)\\
\bR^+ &:= \tilcalS_{\lambda,M} \left(\bV-\bW\bH^+\right)\\
\bW^+_{:j} &:= \frac{\calP_+(\bW-\eta_2(\bH^+)\bG)_{:j}}{\max\{1,\norm{\calP_+(\bW-\eta_2(\bH^+)\bG)_{:j}}_2\}}, \;\forall\,j\in [K],
\end{align}
where $\bG = (\bW\bH^++\bR^+-\bV){\bH^+}^T$, $\eta_1(\bW) \in \left(0,\norm{\bW}_2^{-2}\right]$ and $\eta_2(\bH^+)\in\left(0,\norm{\bH^+{\bH^+}^T}_F^{-1}\right]$.

\subsection{Batch algorithm based on ADMM (BADMM)}
\eqref{eq:brnmf} can be reformulated as 
\begin{align*}
&\min \quad \frac{1}{2}\norm{\bV-\bW\bH-\bR}_F^2 + \lambda\norm{\bR}_{1,1}\\
&\st \quad \bH = \bU,  \bR=\bQ, \bW = \bPsi,  \bU\in\bbR_+^{K\times N}, \bQ\in\tilcalR, \bPsi\in\calC. 
\end{align*}
Thus the augmented Lagrangian is 
\begin{align*}
&\tilcalL(\bH,\bR,\bW,\bU,\bQ,\bPsi,\bA,\bB,\bPsi) = \frac{1}{2}\norm{\bV-\bW\bH-\bR}_F^2 + \lambda\norm{\bR}_{1,1} + \lrangle{\bA}{\bH-\bU} + \lrangle{\bB}{\bR-\bQ} +\lrangle{\bD}{\bW-\bPsi} \\
&\hspace{8.5cm}+\frac{\trho_1}{2}\norm{\bH-\bU}_F^2 + \frac{\trho_2}{2}\norm{\bR-\bQ}_F^2 + \frac{\trho_3}{2} \norm{\bW-\bPsi}_F^2, \numberthis
\end{align*}
where $\bA$, $\bB$ and $\bD$ are dual variables and $\trho_1$, $\trho_2$ and $\trho_3$ are positive penalty parameters. 
Therefore we can derive the following update rules 
\begin{align}
\bH^+ &:= (\bW^T\bW+\trho_1\bI)^{-1}\left(\bW^T(\bV-\bR) + \trho_1\bU-\bA\right)\\
\bR^+ &:= \calS_\lambda(\trho_2\bQ+\bV-\bB-\bW\bH^+)/(1+\trho_2)\\
\bW^+ &:= \left((\bV-\bR^+){\bH^+}^T-\bD+\trho_3\bPsi\right)\left(\bH^+{\bH^+}^T+\trho_3\bI\right)^{-1}\\
\bU^+ &:= \calP_+\left(\bH^+ +\bA/\trho_1\right)\\
\bQ^+ &:= \calP_{\tilcalR} \left(\bR^+ +\bB/\trho_2\right)\\
\bPsi^+ &:= \calP_\calC\left(\bW^++\bD/\trho_3\right)\\
\bA^+ &:= \bA + \trho_1(\bH^+-\bU^+)\\
\bB^+ &:= \bB + \trho_2(\bR^+-\bQ^+)\\
\bD^+ &:= \bD + \trho_3\left(\bW^+ -\bPsi^+\right),
\end{align}

\section{Proof of Lemma~\ref{lem:regularity}} \label{sec:regularity}
Before proving Lemma~\ref{lem:regularity}, we present two lemmas which will be used in the proof. Both lemmas can be proved using straightforward calculations. See Section~\ref{sec:Lips_nablaf} and \ref{sec:Lips_b} for detailed proofs. 

\begin{lemma}\label{lem:Lips_nablaf}
If for each $\bv\in\calV$, both $\bh^*(\bv,\bW)$  and $\br^*(\bv,\bW)$ in \eqref{eq:nablaW_ell} are Lipschitz on $\calC$, with Lipschitz constants $c_1$ and $c_2$ (independent of $\bv$) respectively, then $\bW\mapsto\nabla_\bW\ell(\bv,\bW)$ is Lipschitz on $\calC$ with Lipschitz constant $c_3$ (independent of $\bv$). Consequently, $\nabla f(\bW)$ in \eqref{eq:nablaW_f} is Lipschitz on $\calC$ with Lipschitz constant $c_3$. 
\end{lemma}

\begin{lemma} \label{lem:Lips_b}
Let $\bz$, $\bz' \in\calZ\subseteq \bbR^m$ and $\bA$, $\bA'\in\calA\subseteq \bbR^{m\times n}$, where both $\calZ$ and $\calA$ are compact sets. Let $\calB$ be a compact set in $\bbR^n$, and define $g:\calB\to\bbR$ as $g(\bb) = 1/2\norm{\bz-\bA\bb}_2^2 - 1/2 \norm{\bz'-\bA'\bb}^2_2$. Then $g$ is Lipschitz on $\calB$ with Lipschitz constant $c_1\norm{\bz-\bz'}_2 + c_2\norm{\bA-\bA'}_2$, where $c_1$ and $c_2$ are two positive constants. In particular, when both $\bz'$ and $\bA'$ are zero, we have that $\tilg(\bb) = 1/2\norm{\bz-\bA\bb}_2^2$ is Lipschitz on $\calB$ with Lipschitz constant $c$ independent of $\bz$ and $\bA$. 
\end{lemma}

It is easy to verify that the following conditions hold
\begin{enumerate}[1)]
\item $\tilell(\bv,\bW, \bh, \br)$ is differentiable on $\calV\times\calC$, for each $(\bh,\br)\in\calH\times\calR$,
\item $\tilell(\bv,\bW, \bh, \br)$ and $\nabla_{(\bv,\bW)}\tilell(\bv,\bW, \bh, \br)$ are continuous on $\calV\times\calC\times\calH\times\calR$,
\item \eqref{eq:new_ithloss} has unique minimizer $(\bh^*(\bv,\bW),\br^*(\bv,\bW))$ for each $(\bv,\bW)\in\calV\times\calC$, due to Assumption~\ref{assum:sc_hr}. 
\end{enumerate}
Thus, we can invoke Danskin's theorem (see Lemma~\ref{lem:min_diff}) to conclude that $\ell(\bv,\bW)$ is differentiable on $\calV\times\calC$ and 
\begin{equation}
\nabla_\bW\ell(\bv,\bW) = \left(\bW\bh^*(\bv,\bW)+\br^*(\bv,\bW)-\bv\right)\bh^*(\bv,\bW)^T. \label{eq:nablaW_ell}
\end{equation}
Furthermore, we can show $(\bh^*(\bv,\bW),\br^*(\bv,\bW))$ is continuous on $\calV\times\calC$ by the maximum theorem (see Lemma~\ref{lem:min_cont}), since the conditions in this theorem are trivially satisfied in our case. Thus, $\nabla_\bW\ell(\bv,\bW)$ is continuous on $\calV\times \calC$. 

Leveraging the regularity of $\ell(\bv,\bW)$, we proceed to show the regularity of $f(\bW)$. Since for all $\bv\in\calV$, both $\ell(\bv,\bW)$ and $\nabla_\bW\ell(\bv,\bW)$  are continuous on $\calC$, by Leibniz integral rule (see Lemma~\ref{lem:Leib_int}), we conclude that $f(\bW)$ is differentiable on $\calC$ and 
\begin{equation}
\nabla f(\bW) = \bbE_\bv[\nabla_\bW \ell(\bv,\bW)]. \label{eq:nablaW_f}
\end{equation}


By Lemma~\ref{lem:Lips_nablaf}, to show both $\bW\mapsto\nabla_\bW\ell(\bv,\bW)$ and $\nabla f(\bW)$ are Lipschitz on $\calC$, it suffices to show both $\bh^*(\bv,\bW)$ and $\br^*(\bv,\bW)$ are Lipschitz on $\calC$, for all $\bv\in\calV$. Fix arbitrary $\bv_1,\bv_2\in\calV$ and $\bW_1,\bW_2\in\calC$. Define 
\begin{align*}
d(\bh,\br) &\defeq \tilell(\bv_1,\bW_1,\bh,\br) - \tilell(\bv_2,\bW_2,\bh,\br)\\
& =\frac{1}{2}\norm{\bv_1-\bY_1\bb(\bh,\br)}_2^2 - \frac{1}{2}\norm{\bv_2-\bY_2\bb(\bh,\br)}_2^2
\end{align*}
where $\bY_i = [\bW_i\;\; \bI]$, $i=1,2$ and $\bb(\bh,\br) = [\bh^T\;\;\br^T]^T$. By Lemma~\ref{lem:Lips_b}, we have for all $(\bh_1,\br_1),(\bh_2,\br_2)\in\calH\times\calR$,
\begin{align}
\abs{d(\bh_1,\br_1)-d(\bh_2,\br_2)} &\le \left(c_1\norm{\bv_1-\bv_2}_2 + c_2\norm{\bY_1-\bY_2}_2\right)\norm{\bb(\bh_1,\br_1)-\bb(\bh_2,\br_2)}_2
\end{align}
where $c_1$ and $c_2$ are positive constants. In particular, we have
\begin{equation}
\abs{d(\bh^*_1,\br^*_1)-d(\bh^*_2,\br^*_2)} \le \left(c_1\norm{\bv_1-\bv_2}_2 + c_2\norm{\bY_1-\bY_2}_2\right)\norm{\bb(\bh^*_1,\br^*_1)-\bb(\bh^*_2,\br^*_2)}_2 \label{eq:Lips_upperbd}
\end{equation}
where $\bh^*_i=\bh^*(\bv_i,\bW_i)$ and $\br^*_i=\br^*(\bv_i,\bW_i)$, $i=1,2$. On the other hand, by Assumption~\ref{assum:sc_hr},
\begin{align*}
\abs{d(\bh^*_2,\br^*_2) - d(\bh^*_1,\br^*_1)} &=\abs{\tilell(\bv_1,\bW_1,\bh^*_2,\br^*_2) - \tilell(\bv_2,\bW_2,\bh^*_2,\br^*_2) - \tilell(\bv_1,\bW_1,\bh^*_1,\br^*_1) + \tilell(\bv_2,\bW_2,\bh^*_1,\br^*_1)}\\
&= (\tilell(\bv_1,\bW_1,\bh^*_2,\br^*_2)-\tilell(\bv_1,\bW_1,\bh^*_1,\br^*_1)) + (\tilell(\bv_2,\bW_2,\bh^*_1,\br^*_1) - \tilell(\bv_2,\bW_2,\bh^*_2,\br^*_2))\\
&\ge m_1\norm{\bb(\bh^*_1,\br^*_1)-\bb(\bh^*_2,\br^*_2)}_2^2.\numberthis \label{eq:sc_lowerbd}
\end{align*}
Combining \eqref{eq:Lips_upperbd} and \eqref{eq:sc_lowerbd}, we have
\begin{equation}
\max\{\norm{\bh^*_1-\bh^*_2}_2,\norm{\br^*_1-\br^*_2}_2\} \le \norm{\bb(\bh^*_1,\br^*_1)-\bb(\bh^*_2,\br^*_2)}_2 \le c'_1\norm{\bv_1-\bv_2}_2 + c'_2\norm{\bW_1-\bW_2}_2
\end{equation}
where $c'_i = c_i/m_1$, $i=1,2$. This indeed shows both $\bh^*(\bv,\bW)$ and $\br^*(\bv,\bW)$ are Lipschitz on $\calV\times\calC$ since
\begin{equation}
c'_1\norm{\bv_1-\bv_2}_2 + c'_2\norm{\bW_1-\bW_2}_2 \le 2\max(c'_1,c'_2)\norm{[\bv_1\;\;\bW_1]-[\bv_2 \;\;\bW_2]}_F.
\end{equation}
Hence we complete the proof. 

\section{Proof of Lemma~\ref{lem:rate_W}} \label{sec:rate_W}

By Assumption~\ref{assum:sc_W}, for all $t\ge 1$, we have 
\begin{equation}
\tilf_t(\bW_{t+1}) - \tilf_t(\bW_t) \ge \frac{m_2}{2}\norm{\bW_{t+1}-\bW_t}_F^2, \label{eq:sc_lowerbdW}
\end{equation}
since 
$\bW_t = \argmin_\bW \tilf_t(\bW)$. On the other hand, 
\begin{align*}
\tilf_t(\bW_{t+1}) - \tilf_t(\bW_t) &= \tilf_{t}(\bW_{t+1}) - \tilf_{t+1}(\bW_{t+1}) + \tilf_{t+1}(\bW_{t+1}) - \tilf_{t+1}(\bW_t) + \tilf_{t+1}(\bW_t) - \tilf_t(\bW_t)\\
&\le \tild_t(\bW_{t+1}) - \tild_t(\bW_t)
\end{align*}
where $\tild_t(\bW) \defeq \tilf_{t}(\bW) - \tilf_{t+1}(\bW)$, for all $\bW\in\calC$. We aim to show $\tild_t$ is Lipschitz on $\calC$ with Lipschitz constant only dependent on $t$. For all $\bW\in\calC$, we have
\begin{align*}
\tild_t(\bW) &= \frac{1}{t}\sum_{i=1}^t\frac{1}{2}\norm{\bv_i-\bW\bh_i-\br_i}_2^2+\lambda\norm{\br_i}_1 - \frac{1}{t+1}\sum_{i=1}^{t+1}\frac{1}{2}\norm{\bv_i-\bW\bh_i-\br_i}_2^2+\lambda\norm{\br_i}_1\\
&= \frac{1}{t(t+1)}\left((t+1)\sum_{i=1}^t\frac{1}{2}\norm{\bv_i-\bW\bh_i-\br_i}_2^2+\lambda\norm{\br_i}_1 - t\sum_{i=1}^{t+1}\frac{1}{2}\norm{\bv_i-\bW\bh_i-\br_i}_2^2+\lambda\norm{\br_i}_1\right)\\
&\eqcst \frac{1}{t(t+1)}\left(\sum_{i=1}^t\frac{1}{2}\norm{\bv_i-\bW\bh_i-\br_i}_2^2 - \frac{t}{2}\norm{\bv_{t+1}-\bW\bh_{t+1}-\br_{t+1}}_2^2\right)\\
&= \frac{1}{t(t+1)}\sum_{i=1}^t \left(\frac{1}{2}\norm{\bv_i-\bW\bh_i-\br_i}_2^2 - \frac{1}{2}\norm{\bv_{t+1}-\bW\bh_{t+1}-\br_{t+1}}_2^2\right),
\end{align*}
where $\eqcst$ denotes equality up to an additive constant (independent of $\bW$). By a reasoning similar to the one for Lemma~\ref{lem:Lips_b}, we can show 
there exist some positive  constants $c_1$ and $c_2$ such that 
\begin{align*}
\abs{\tild_t(\bW_{t+1}) - \tild_t(\bW_t)} &\le \frac{1}{t(t+1)}\sum_{i=1}^t \left(c_1\norm{\bh_i-\bh_{t+1}}_2 + c_2\norm{\bv_i-\bv_{t+1}}_2 + c_2\norm{\br_i-\br_{t+1}}_2\right)\norm{\bW_{t+1} - \bW_t}_2\\
&\le \frac{c_3}{t+1}\norm{\bW_{t+1} - \bW_t}_F, \numberthis \label{eq:Lips_upperbdW}
\end{align*}
where $c_3>0$ is a constant since for all $i\ge 1$, $\bv_i$, $\bh_i$ and $\br_i$ are bounded a.s.. Combining \eqref{eq:sc_lowerbdW} and \eqref{eq:Lips_upperbdW}, we have with probability one, 
\begin{equation}
\norm{\bW_{t+1} - \bW_t}_F \le \frac{c'_3}{t+1},
\end{equation}
where $c'_3 = c_3/m_2$. Hence we complete the proof. 


\section{Proof of Theorem~\ref{thm:as_conv}} \label{sec:as_conv}

We prove that $\{\tilf_t(\bW_t)\}_{t\ge 1}$ converges a.s. by the quasi-martingale convergence theorem (see Lemma~\ref{lem:quasi_conv}). Let us first define a filtration $\{\calF_t\}_{t\ge 1}$ where $\calF_t \defeq \sigma \left\{\bv_i,\bW_i,\bh_i,\br_i\right\}_{i\in[t]}$ is 
the minimal $\sigma$-algebra such that $\left\{\bv_i,\bW_i,\bh_i,\br_i\right\}_{i\in[t]}$ are measurable. Also define 
\begin{equation}
u_t \defeq \tilf_t(\bW_t)\quad\mbox{and}\quad
\delta_t \defeq \begin{cases}
1, & \mbox{if} \;\; \bbE[u_{t+1}-u_t|\calF_t] >0 \\
0, & \mbox{otherwise}
\end{cases},
\end{equation}
then it is easy to see $\{u_t\}_{t\ge 1}$ is adapted to $\{\calF_t\}_{t\ge 1}$.
According to the quasi-martingale convergence theorem, it suffices to show $\sum_{t=1}^\infty \bbE[\delta_t(u_{t+1}-u_t)]<\infty$. 
To bound $\bbE[\delta_t(u_{t+1}-u_t)]$, we decompose $u_{t+1}-u_t$ as
\begin{align*}
u_{t+1} -u_t &= \tilf_{t+1}(\bW_{t+1}) - \tilf_{t+1}(\bW_{t}) + \tilf_{t+1}(\bW_{t})- \tilf_{t}(\bW_{t})\\
&= \tilf_{t+1}(\bW_{t+1}) - \tilf_{t+1}(\bW_{t}) + \frac{1}{t+1} \tilell(\bv_{t+1}, \bh_{t+1}, \br_{t+1}, \bW_t) + \frac{t}{t+1} \tilf_t(\bW_t) - \tilf_{t}(\bW_{t})\\
&=  \tilf_{t+1}(\bW_{t+1}) - \tilf_{t+1}(\bW_{t}) + \frac{\ell(\bv_{t+1},\bW_t) - \tilf_t(\bW_t)}{t+1}\numberthis\label{eq:decomp_u0}\\
&= \underbrace{\tilf_{t+1}(\bW_{t+1}) - \tilf_{t+1}(\bW_{t})}_{\lrang{1}} + \frac{\ell(\bv_{t+1},\bW_t)-f_t(\bW_t)}{t+1} + \underbrace{\frac{f_t(\bW_t)- \tilf_t(\bW_t)}{t+1}}_{\lrang{2}}\numberthis \label{eq:decomp_u}.
\end{align*}
By definition, it is easy to see both $\lrang{1},\lrang{2}\le 0$. In \eqref{eq:decomp_u}, we insert the term $f_t(\bW_t)$ in order to invoke Donsker's theorem (see Lemma~\ref{lem:Donsker}). Thus,  
\begin{align*}
\bbE[u_{t+1} -u_t\,|\,\calF_t] \le \frac{\bbE[\ell(\bv_{t+1},\bW_t)-f_t(\bW_t)\,|\,\calF_t]}{t+1} = \frac{f(\bW_t)-f_t(\bW_t)}{t+1}
\end{align*}
and using the definition of $\delta_t$, 
\begin{equation}
\bbE[\delta_t\bbE[u_{t+1} -u_t\,|\,\calF_t]] \le \frac{\bbE\left[\abs{f(\bW_t)-f_t(\bW_t)}\right]}{t+1} \le \frac{\bbE\left[\norm{f-f_t}_\calC\right]}{t+1} = \frac{\bbE\left[\norm{\sqrt{t}(f-f_t)}_\calC\right]}{\sqrt{t}(t+1)},
\end{equation}
where $\norm{\cdot}_\calC$ denotes the uniform norm on $\calC$. By Lemma~\ref{lem:regularity}, we know for all $\bv\in\calV$, $\ell(\bv,\cdot)$ is Lipschitz on $\calC$ with Lipschitz constant independent of $\bv$.  Thus, by Lemma~\ref{lem:suff_Donsker}, the measurable function class $\{\ell(\cdot,\bW):\bW\in\calC\}$ is $\bbP$-Donsker. Consequently $\bbE\left[\norm{\sqrt{t}(f-f_t)}_\calC\right]$ is bounded by a constant $c>0$. Thus, $\bbE[\delta_t\bbE[u_{t+1} -u_t\,|\,\calF_t]]\le c/t^{3/2}$. Since $\bbE[\delta_t(u_{t+1}-u_t)] = \bbE[\bbE[\delta_t(u_{t+1}-u_t)]\,|\,\calF_t]] = \bbE[\delta_t\bbE[u_{t+1}-u_t\,|\,\calF_t]]$,
we have $\sum_{t=1}^\infty \bbE[\delta_t(u_{t+1} -u_t)]<\infty$. Thus $\{\tilf_t(\bW_t)\}_{t\ge 1}$ converges a.s.. Moreover, by Lemma~\ref{lem:quasi_conv}, we also have $\sum_{t=1}^\infty \bbE[\abs{\bbE[u_{t+1}-u_t|\calF_t]}] < \infty$.

Leveraging this result, we proceed to show the almost sure convergence of $\{f(\bW_t)\}_{t\ge 1}$. By Lemma~\ref{lem:suff_Donsker}, the measurable function class $\{\ell(\cdot,\bW):\bW\in\calC\}$ is also $\bbP$-Glivenko-Cantelli. Thus by Glivenko-Cantelli theorem (see Lemma~\ref{lem:GlivenkoCantelli}), we have $\norm{f_t-f}_\calC\convas 0$.  Hence it suffices to show $\{f_t(\bW_t)\}_{t\ge 1}$ converges a.s.. We show this by proving $f_t(\bW_t)-\tilf_t(\bW_t)\convas 0$. First, from \eqref{eq:decomp_u}, we have
\begin{align*}
\frac{\tilf_t(\bW_t)-f_t(\bW_t)}{t+1}
=\;\;& \bbE\left[\frac{\tilf_t(\bW_t)-f_t(\bW_t)}{t+1}\,\bigg|\,\calF_t\right] \\
=\;\;& \bbE[\tilf_{t+1}(\bW_{t+1}) - \tilf_{t+1}(\bW_{t})\,|\,\calF_t] + \bbE\left[\frac{\ell(\bv_{t+1},\bW_t)-f_t(\bW_t)}{t+1}\,\bigg|\,\calF_t\right] - \bbE[u_{t+1} -u_t\,|\,\calF_t]\\
\le\;\;& \frac{\bbE\left[\ell(\bv_{t+1},\bW_t)\,|\,\calF_t\right]-f_t(\bW_t)}{t+1} - \bbE[u_{t+1} -u_t\,|\,\calF_t]\\
\le\;\;& \frac{f(\bW_t)-f_t(\bW_t)}{t+1} - \bbE[u_{t+1} -u_t\,|\,\calF_t]\\
\le\;\;& \frac{\norm{f-f_t}_\calC}{t+1} - \bbE[u_{t+1} -u_t\,|\,\calF_t].
\end{align*}
Since both $\sum_{t=1}^\infty \frac{\norm{f-f_t}_\calC}{t+1}$ and $\sum_{t=1}^\infty  \abs{\bbE[u_{t+1} -u_t\,|\,\calF_t]}$ converge a.s., we conclude 
$\sum_{t=1}^\infty \frac{\tilf_t(\bW_t)-f_t(\bW_t)}{t+1}$
converges a.s.. Define $b_t \defeq \tilf_t(\bW_t)-f_t(\bW_t)$, we show $\abs{b_{t+1}-b_t} = O(1/t)$ a.s. by proving both $|\tilf_{t+1}(\bW_{t+1}) - \tilf_t(\bW_t)| = O(1/t)$ and $|f_{t+1}(\bW_{t+1}) - f_t(\bW_t)| = O(1/t)$ a.s.. First, from \eqref{eq:decomp_u0} and the Lipschitz continuity of $\tilf_t$ on $\calC$,
\begin{align*}
|\tilf_{t+1}(\bW_{t+1}) - \tilf_t(\bW_t)| &\le |\tilf_{t+1}(\bW_{t+1}) - \tilf_{t+1}(\bW_{t})| + \abs{\frac{\ell(\bv_{t+1},\bW_t) - \tilf_t(\bW_t)}{t+1}}\\
&\le c\norm{\bW_{t+1} - \bW_t}_F + \frac{|\ell(\bv_{t+1},\bW_t)| + |\tilf_t(\bW_t)|}{t+1},
\end{align*}
where $c>0$ is a constant independent of $t$. Since both $|\ell(\bv_{t+1},\bW_t)|$ and $|\tilf_t(\bW_t)|$ are bounded on $\calC$ a.s.\ and $\norm{\bW_{t+1} - \bW_t}_F = O(1/t)$ a.s. (by Lemma~\ref{lem:rate_W}), we have $|\tilf_{t+1}(\bW_{t+1}) - \tilf_t(\bW_t)| = O(1/t)$ a.s.. Similarly, by the Lipschitz continuity of $f_t$ on $\calC$, we also have $|f_{t+1}(\bW_{t+1}) - f_t(\bW_t)| = O(1/t)$ a.s.. Now we invoke Lemma~\ref{lem:conv_nonneg} to conclude 
\begin{equation}
f_t(\bW_t)-\tilf_t(\bW_t)\convas 0. \label{eq:as_conv_diff_supp}
\end{equation}
Since $f_t(\bW_t)-f(\bW_t)\convas 0$, both $\{f(\bW_t)\}_{t\ge 1}$ and $\{\tilf_t(\bW_t)\}_{t\ge 1}$ converge to the same almost sure limit. 

\section{Proof of Theorem~\ref{thm:local_min}}\label{sec:local_min}

By \eqref{eq:as_conv_diff_supp}, it suffices to show that for every realization of $\{\bv_t\}_{t\ge 1}$ such that $\tilf_t(\bW_t) - f_t(\bW_t)\to 0$, each subsequential limit of $\{\bW_t\}_{t\ge 1}$ is a stationary point of $f$. We focus on such a realization, then all the variables in the sequel become deterministic. (With a slight abuse of notations we use the same notations to denote the deterministic variables.) 
By the compactness of $\calV$, $\calH$ and $\calR$, both sequences $\{\bA_t\}_{t\ge 1}$ and $\{\bB_t\}_{t\ge 1}$ are bounded. Thus there exist compact sets $\calA$ and $\calB$ such that $\{\bA_t\}_{t\ge 1} \subseteq \calA$ and $\{\bB_t\}_{t\ge 1} \subseteq \calB$. Similar reasoning shows that the sequence $\{\tilf_t(\vecz)\}_{t\ge 1}$ resides in a compact set $\tilcalF$. By the compactness of $\calC$, there exists a convergent subsequence $\{\bW_{t_m}\}_{m\ge 1}$ in $\{\bW_t\}_{t\ge 1}$.  Also, by the compactness of $\calA$, $\calB$ and $\tilcalF$, it is possible to find convergent subsequences $\{\bA_{t_k}\}_{k\ge 1}$, $\{\bB_{t_k}\}_{k\ge 1}$ and $\{\tilf_{t_k}(\vecz)\}_{k\ge 1}$ such that $\{t_k\}_{k\ge 1}\subseteq\{t_m\}_{m\ge 1}$. Thus, we focus on the convergent sequences $\{\bW_{t_k}\}_{k\ge 1}$, $\{\bA_{t_k}\}_{k\ge 1}$, $\{\bB_{t_k}\}_{k\ge 1}$ and $\{\tilf_{t_k}(\vecz)\}_{k\ge 1}$ and drop the subscript $k$ to make notations uncluttered. We denote the limits of the sequences $\{\bW_{t}\}_{t\ge 1}$, $\{\bA_{t}\}_{t\ge 1}$ and $\{\bB_{t}\}_{t\ge 1}$ as $\barbW$, $\barbA$ and $\barbB$ respectively. 

First, we show the sequence of differentiable functions $\{\tilf_t\}_{t\ge 1}$ converges uniformly to a differentiable function $\barf$. Since $\{\tilf_{t}(\vecz)\}_{t\ge 1}$ converges, it suffices to show the sequence $\{\nabla \tilf_t\}_{t\ge 1}$ converges uniformly to a function $\barh$. Since $\nabla \tilf_t(\bW) = \bW\bA_t - \bB_t$, for any $t,t'\ge 1$ and any $\bW\in\calC$, we have
\begin{align*}
\norm{\nabla \tilf_t(\bW) - \nabla \tilf_{t'}(\bW)}_F  &= \norm{\bW\left(\bA_t - \bA_{t'}\right) -  \left(\bB_t - \bB_{t'}\right)}_F \\
&\le \norm{\bW}_F \normt{\bA_t - \bA_{t'}}_F + \normt{\bB_t - \bB_{t'}}_F\\
&\le \sqrt{K}\left(\normt{\bA_t - \bA_{t'}}_F + \normt{\bB_t - \bB_{t'}}_F\right). \numberthis \label{eq:Lips_nablaf_t}
\end{align*}
From \eqref{eq:Lips_nablaf_t}, it is easy to see $\barh(\bW) = \bW\barbA - \barbB$ since
$\sup_{\bW\in\calC} \normt{\nabla \tilf_t(\bW) - \barh(\bW)}_F  \le \sqrt{K}\left(\normt{\bA_t - \barbA}_F + \normt{\bB_t - \barbB}_F\right)$. 

Next, define $g_t \defeq \tilf_t - f_t$, for all $t\ge 1$. By definition, we have $g_t(\bW)\ge 0$, for any $\bW\in \calC$. Since $\tilf_t\convu \barf$ and $f_t \convu f$ (by Glivenko-Cantelli theorem), we have $g_t \convu \barg \defeq \barf - f$. 
Since both $\barf$ and  $f$ are differentiable, $\barg$ is differentiable and $\nabla f = \nabla \barf - \nabla \barg$. To show $\barbW$ is a stationary point of $f$, it suffices to show for any $\bW\in\calC$, the directional derivative $\lrangle{\nabla f(\barbW)}{\bW - \barbW}\ge 0$. We show this by proving  $\lrangle{\nabla \barf(\barbW)}{\bW - \barbW}\ge 0$ and $\lrangle{\nabla \barg(\barbW)}{\bW - \barbW}= 0$ for any $\bW\in\calC$.

By definition, for any $\bW\in\calC$ and $t\ge 1$, we have $\tilf_t(\bW_t)  \le \tilf_t(\bW)$.  First consider 
\begin{align*}
\abs{\tilf_t(\bW_t) - \barf(\barbW)} & = \abs{\tilf_t(\bW_t) - \barf(\bW_t) + \barf(\bW_t) - \barf(\barbW)}\\
&\le \norm{\tilf_t - \barf}_\calC + \abs{\barf(\bW_t) - \barf(\barbW)}.
\end{align*}
Since $\tilf_t\convu \barf$ and $\barf$ is continuous, we have $\tilf_t(\bW_t) \to \barf(\barbW)$ as $t\to\infty$. Thus $\barf(\barbW)\le \barf(\bW)$, for any $\bW\in\calC$. This implies  $\lrangle{\nabla \barf(\barbW)}{\bW - \barbW}\ge 0$. 

Next we show $\lrangle{\nabla \barg(\barbW)}{\bW - \barbW}= 0$ for any $\bW\in\calC$. It suffices to show $\nabla \barg(\barbW) = \vecz$. Since both $\tilf_t$ and $f_t$ are differentiable, $g_t$ is differentiable and $\nabla g_t = \nabla \tilf_t - \nabla f$. First it is easy to see $\nabla g_t$ is Lipschitz  on $\calC$ with constant $L>0$ independent of $t$ since both $\nabla \tilf_t$ and $\nabla f_t$ are Lipschitz on $\calC$ with constants independent of $t$. It is possible to construct another differentiable function $\tilg_t$ with domain $\bbR^{F\times K}$ such that $\tilg_t$ is nonnegative with a $L$-Lipschitz gradient on $\bbR^{F\times K}$ and $\tilg_t (\bW) = g_t(\bW)$ for all $\bW\in\calC$. Thus
\begin{align*}
\frac{1}{2L}\normt{\nabla g_t(\bW_t)}_F^2=\frac{1}{2L}\normt{\nabla \tilg_t(\bW_t)}_F^2 \le \tilg_t(\bW_t) - \tilg_t^* \le g_t(\bW_t) = \tilf_t(\bW_t) - f_t(\bW_t),
\end{align*}
where $\tilg_t^* = \inf_{\bW\in\bbR^{F\times K}} \tilg_t(\bW)\ge 0$. Since  $\tilf_t(\bW_t) - f_t(\bW_t)\to 0$, we have $\nabla g_t(\bW_t)\to \vecz$ as $t\to\infty$. Now consider the first-order Taylor expansion of $g_t$ at $\bW_t$
\begin{equation}
g_t(\bW) = g_t(\bW_t) + \lrangle{\nabla g_t(\bW_t)}{\bW-\bW_t} + o\left(\normt{\bW-\bW_t}_F\right), \;\; \forall\,\bW\in\calC. 
\end{equation}
As $t\to\infty$, we have
\begin{equation}
\barg(\bW) = \barg(\barbW) + o\left(\norm{\bW-\barbW}_F\right), \;\; \forall\,\bW\in\calC. \label{eq:barg_Taylor1}
\end{equation}
On the other hand, we have
\begin{equation}
\barg(\bW) = \barg(\barbW) + \lrangle{\nabla \barg(\barbW)}{\bW-\barbW} + o\left(\norm{\bW-\barbW}_F\right), \;\; \forall\,\bW\in\calC.  \label{eq:barg_Taylor2}
\end{equation}
Comparing \eqref{eq:barg_Taylor1} and \eqref{eq:barg_Taylor2}, we have
\begin{equation}
\lrangle{\nabla \barg(\barbW)}{\bW-\barbW} + o\left(\norm{\bW-\barbW}_F\right) = 0, \;\; \forall\,\bW\in\calC.
\end{equation}
Therefore we conclude $\nabla \barg(\barbW) = \vecz$. 


\section{Proof of Lemma~\ref{lem:Lips_nablaf} } \label{sec:Lips_nablaf}

Since $\bv$, $\bW$, $\bh^*(\bv,\bW)$  and $\br^*(\bv,\bW)$ are all bounded, there exist positive constants $M_1$, $M_2$, $M_3$ and $M_4$ that upper bound $\norm{\bv}$, $\norm{\bW}$, $\norm{\bh^*(\bv,\bW)}$  and $\norm{\br^*(\bv,\bW)}$ respectively. Here the matrix norm is the one induced by the (general) vector norm. Take arbitrary $\bW_1$ and $\bW_2$ in $\calC$, we have 
\begin{align*}
\norm{\nabla f(\bW_1) - \nabla f(\bW_2)} &= \norm{\bbE_\bv[\nabla_{\bW}\ell(\bv,\bW_1) - \nabla_{\bW}\ell(\bv,\bW_2)]}\\
&\le \bbE_\bv \norm{\nabla_{\bW}\ell(\bv,\bW_1) - \nabla_{\bW}\ell(\bv,\bW_2)}.
\end{align*}
Fix an arbitrary $\bv\in\calV$ we have
\begin{align*}
\norm{\nabla_{\bW}\ell(\bv,\bW_1) - \nabla_{\bW}\ell(\bv,\bW_2)} &\le \norm{\bW_1\bh^*(\bv,\bW_1)\bh^*(\bv,\bW_1)^T - \bW_2\bh^*(\bv,\bW_2)\bh^*(\bv,\bW_2)^T} \\
&+ \norm{\br^*(\bv,\bW_1)\bh^*(\bv,\bW_1)^T - \br^*(\bv,\bW_2)\bh^*(\bv,\bW_2)^T} \numberthis \label{eq:norm_nabla_ell}\\
&+ \norm{\bv\bh^*(\bv,\bW_1)^T - \bv\bh^*(\bv,\bW_2)^T}. 
\end{align*}
We bound each term on the RHS of \eqref{eq:norm_nabla_ell} as follows
\begin{flalign*}
&\norm{\bW_1\bh^*(\bv,\bW_1)\bh^*(\bv,\bW_1)^T - \bW_2\bh^*(\bv,\bW_2)\bh^*(\bv,\bW_2)^T}& \\
\le \;&\norm{\bW_1}\norm{\bh^*(\bv,\bW_1)}\norm{\bh^*(\bv,\bW_1)-\bh^*(\bv,\bW_2)} + \norm{\bW_1\bh^*(\bv,\bW_1) - \bW_2\bh^*(\bv,\bW_2)}\norm{\bh^*(\bv,\bW_2)}&\\
\le \;&M_2M_3c_1\norm{\bW_1-\bW_2} + \left(\norm{\bW_1}\norm{\bh^*(\bv,\bW_1) - \bh^*(\bv,\bW_2)} + \norm{\bW_1-\bW_2}\norm{\bh^*(\bv,\bW_2)}\right)\norm{\bh^*(\bv,\bW_2)}&\\
\le \;&M_2M_3c_1\norm{\bW_1-\bW_2} + M_3\left(M_2c_1\norm{\bW_1-\bW_2} + M_3\norm{\bW_1-\bW_2}\right)\\
= \;&(2c_1M_2M_3 + M_3^2)\norm{\bW_1-\bW_2},&
\end{flalign*}
\begin{flalign*}
&\norm{\br^*(\bv,\bW_1)\bh^*(\bv,\bW_1)^T - \br^*(\bv,\bW_2)\bh^*(\bv,\bW_2)^T}&\\
\le\;&\norm{\br^*(\bv,\bW_1)}\norm{\bh^*(\bv,\bW_1)-\bh^*(\bv,\bW_2)} + \norm{\br^*(\bv,\bW_1)-\br^*(\bv,\bW_2)}\norm{\bh^*(\bv,\bW_2)}&\\
\le\;&c_1M_4\norm{\bW_1-\bW_2} + c_2M_3\norm{\bW_1-\bW_2}&\\
\le\;&(c_1M_4 + c_2M_3)\norm{\bW_1-\bW_2},&
\end{flalign*}
and
\begin{flalign*}
&\norm{\bv\bh^*(\bv,\bW_1)^T - \bv\bh^*(\bv,\bW_2)^T}\le c_1M_1\norm{\bW_1-\bW_2}.&
\end{flalign*}
Thus, take $c_3 = 2c_1M_2M_3 + M_3^2 + c_1M_4 + c_2M_3 +  c_1M_1$ and we finish the proof.

\section{Proof of Lemma~\ref{lem:Lips_b}} \label{sec:Lips_b}
It suffices to show $\norm{\nabla g(\bb)}_2 \le c_1\norm{\bz-\bz'}_2 + c_2\norm{\bA-\bA'}_2$ for any $\bb\in\calB$ and some positive constants $c_1$ and $c_2$ (independent of $\bb$). We write  $\norm{\nabla g(\bb)}_2$ as
\begin{align*}
\norm{\nabla g(\bb)}_2 &= \normt{{\bA'}^T\left(\bA'\bb - \bz'\right) - \bA^T\left(\bA\bb-\bz\right)}_2\\
& = \normt{({\bA'}^T\bA'-\bA^T\bA)\bb - ({\bA'}^T\bz' - \bA^T\bz)}_2\\
& \le \underbrace{\normt{{\bA'}^T\bA'-\bA^T\bA}_2\norm{\bb}_2}_{\lrang{1}} + \underbrace{\normt{{\bA'}^T\bz' - \bA^T\bz}_2}_{\lrang{2}}.
\end{align*}
By the  compactness of $\calZ$, $\calA$ and $\calB$, there exist positive  constants $M_1$, $M_2$ and $M_3$ such that $\norm{\bz}_2\le M_1$, $\norm{\bA}_2\le M_2$ and $\norm{\bb}_2\le M_3$, for any $\bz\in\calZ$, $\bA\in\calA$ and $\bb\in\calB$. Thus, 
\begin{align*}
\lrang{1} &\le M_3\normt{\bA^T\bA-\bA^T\bA'+\bA^T\bA'-{\bA'}^T\bA'}_2\\
&\le  M_3\lrpar{\normt{\bA^T(\bA-\bA')}_2+\norm{(\bA-\bA')^T\bA'}_2}\\
&\le 2M_2M_3\norm{\bA-\bA'}_2.
\end{align*}
Similarly for $\lrang{2}$ we have
\begin{align*}
\lrang{2} &= \norm{\bA^T\bz- \bA^T\bz'+\bA^T\bz'-{\bA'}^T\bz'}_2\\
&\le \norm{\bA^T(\bz-\bz')}_2+\norm{(\bA-\bA')^T\bz'}_2\\
&\le M_2\norm{\bz-\bz'}_2 + M_1\norm{\bA-\bA'}_2.
\end{align*}
Hence $\lrang{1}+\lrang{2}\le M_2\norm{\bz-\bz'}_2 + (M_1+2M_2M_3)\norm{\bA-\bA'}_2$. We now take $c_1 = M_2$ and $c_2 = M_1+2M_2M_3$ to complete the proof. 

\section{Technical lemmas}\label{sec:tech_lemma}

\begin{lemma}[{\cite[Lemma 5]{ZhangBox_15}}] \label{lem:L1_box}
Let $I$ be a closed interval in $\bbR$. Define $g_{\tau,I}(t) = \tau\abs{t} + \delta_I(t)$, where $\delta_I$ is the indicator function for the interval $I$. Then the proximal operator for $g_{\tau,I}$ is given by
\begin{equation}
\prox_{g_{\tau,I}} (q) = \Pi_I(\calS_\tau(q)),
\end{equation}
where $q\in\bbR$, $\calS_\tau$ is the soft-thresholding operator with threshold $\tau$ and $\Pi_I$ is the Euclidean projector onto the interval $I$.
\end{lemma}

\begin{lemma}[Projection onto nonnegative $\ell_2$ balls] \label{lem:proj_nonnegL2}
Let $\calC' \defeq \{\bx\in\bbR_+^n\,|\,\norm{\bx}\le 1\}$. Then for all $\by\in\bbR^n$, 
\begin{equation}
\Pi_{\calC'}(\by) = \frac{\by_+}{\max\left\{1,\norm{\by_+}_2\right\}},
\end{equation}
where $(\by_+)_i = \max\{0,y_i\}$, $\forall\,i\in[n]$. 
\end{lemma}
\begin{proof}
The KKT conditions for 
\begin{align*}
&\min \quad \norm{\by-\bx}_2^2\\
&\st \quad  \bx\ge 0, \;\norm{\bx}_2\le 1
\end{align*}
are given by
\begin{align}
&\bx^*\ge 0, \;\norm{\bx^*}_2\le 1, \; \lambda^*\ge 0 \label{eq:KKT_feas1}\\
&(\lambda^*+1)\bx^*\ge \by, \label{eq:KKT_feas2}\\
&\lambda^*(\norm{\bx^*}_2^2-1) = 0, \label{eq:KKT_comp1}\\
&(\lambda^*+1)(x^*_i)^2 =x^*_iy_i, \;\forall\,i\in[n]. \label{eq:KKT_comp2}
\end{align}
We fix an $i\in[n]$. Define $\calI \defeq \{i\in[n]\,|\,y_i>0\}$. For any $\bz\in\bbR^n$, let $\bz_\calI$ to be the subvector of $\bz$ with indices from $\calI$. If $y_i\le 0$, then by \eqref{eq:KKT_comp2} $x^*_i=0$. If $y_i>0$, then by \eqref{eq:KKT_feas2} $x^*_i>0$. Thus by \eqref{eq:KKT_comp2} we have $x^*_i = y_i/(\lambda^*+1)$. If $\lambda^*=0$, then $x^*_i = y_i$. In this case $\norm{\by_\calI}_2 = \norm{\bx^*_\calI}_2  = \norm{\bx^*}_2 \le 1$. If $\lambda^*>0$, then by \eqref{eq:KKT_comp1} $\norm{\bx^*}_2^2=1$. Then $\norm{\bx^*}_2^2 = \norm{\bx^*_\calI}_2^2 = \norm{\by_\calI}_2^2/(\lambda^*+1)^2 = 1$. This means $\lambda^*+1 = \norm{\by_\calI}_2$ so $x^*_i = y_i/\norm{\by_\calI}_2$. Also, we notice in such case $y_i>x^*_i>0$ so $\norm{\by_\calI}_2>1$. Combining both cases where $\lambda^*=0$ and $\lambda^*>0$, we have
$x_i^* = y_i/\max\{1,\norm{\by_\calI}_2\}$, for all $i\in\calI$.
\end{proof}

\begin{lemma}[Danskin's Theorem; {\cite[Theorem 4.1]{Bon_98}}] \label{lem:min_diff}
Let $\calX$ be a metric space and $\calU$ be a normed vector space. Let $f:\calX\times \calU\to\bbR$ have the following properties
\begin{enumerate}
\item $f(x,\cdot)$ is differentiable on $\calU$, for any $x\in\calX$.
\item $f(x,u)$ and $\nabla_u f(x,u)$ are continuous on $\calX\times\calU$.
\end{enumerate}
Let $\Phi$ be a compact set in $\calX$. Define $v(u) = \inf_{x\in\Phi}f(x,u)$ and $S(u) = \argmin_{x\in\Phi}f(x,u)$, then $v(u)$ is directionally differentiable and its directional derivative along $d\in\calU$, $v'(u,d)$ is given by
\begin{equation}
v'(u,d) = \min_{x\in S(u)} \nabla_uf(x,u)^Td.
\end{equation}
In particular, if for some $u_0\in\calU$, $S(u_0) = \{x_0\}$, then $v$ is differentiable at $u=u_0$ and $\nabla v(u_0) = \nabla_uf(x_0,u_0)$.
\end{lemma}

\begin{lemma}[The Maximum Theorem; {\cite[Theorem 14.2.1 \& Example 2]{Syd_05}}]\label{lem:min_cont}
Let $\calP$ and $\calX$ be two metric spaces. Consider a maximization problem
\begin{equation}
\max_{x\in B(p)} f(p,x), \label{eq:max_problem}
\end{equation}
where $B:\calP \twoheadrightarrow \calX$ is a correspondence and  $f:\calP\times \calX\to\bbR$ is a function. If $B$ is compact-valued and continuous on $\calP$ and $f$ is continuous on $\calP\times \calX$, then the correspondence $S(p) = \argmax_{x\in B(p)} f(p,x)$ is compact-valued and upper hemicontinuous, for any $p\in\calP$.
In particular, if for some $p_0\in\calP$, $S(p_0)=\{s(p_0)\}$, where $s:\calP\to\calX$ is a function, then $s$ is continuous at $p=p_0$. Moreover, we have the same conclusions if the maximization in \eqref{eq:max_problem} is replaced by minimization. 
\end{lemma}

\begin{lemma}[Leibniz Integral Rule] \label{lem:Leib_int}
Let $\calX$ be an open set in $\bbR^n$ and let $(\Omega,\calA,\mu)$ be a measure space. If $f:\calX\times\Omega\to\bbR$ satisfies
\begin{enumerate}
\item For all $x\in\calX$, the mapping $\omega\mapsto f(x,\omega)$ is Lebesgue integrable. 
\item For all $\omega\in\Omega$, $\nabla_xf(x,\omega)$ exists on $\calX$.
\item For all $x\in\calX$,  the mapping $\omega\mapsto \nabla_x f(x,\omega)$ is Lebesgue integrable. 
\end{enumerate}
Then $\int_\Omega f(x,\omega) \, d\mu(\omega)$ is differentiable and  
\begin{equation}
\nabla_x \int_\Omega f(x,\omega) \, d\mu(\omega) = \int_\Omega \nabla_x f(x,\omega) \, d\mu(\omega).
\end{equation}
\end{lemma}
\begin{remark}
This is a simplified version of the Leibniz Integral Rule. See \cite[Theorem 16.8]{Bill_86} for weaker conditions on $f$.
\end{remark}

\begin{definition}[Quasi-martingale; {\cite[Definition 1.4]{Fisk_65}}]\label{def:quasi_mart}
Let $\{X_t\}_{t\in\calT}$ be a stochastic process and 
let $\{\calF_t\}_{t\in\calT}$ be the filtration to which 
$\{X_t\}_{t\in\calT}$ is adapted, where $\calT$ is a subset of the real line. We call $\{X_t\}_{t\in\calT}$ a quasi-martingale if there exist two stochastic processes $\{Y_t\}_{t\in\calT}$ and $\{Z_t\}_{t\in\calT}$ such that \\
\hspace*{.25cm} i) both $\{Y_t\}_{t\in\calT}$ and $\{Z_t\}_{t\in\calT}$ are adapted to $\{\calF_t\}_{t\in\calT}$, \\
\hspace*{.25cm} ii) $\{Y_t\}_{t\in\calT}$ is a martingale and $\{Z_t\}_{t\in\calT}$ has bounded variations on $\calT$ a.s.,\\
\hspace*{.25cm} iii) $X_t = Y_t + Z_t$, for all $t\in\calT$ a.s.. 
\end{definition}

\begin{lemma}[The Quasi-martingale Convergence Theorem; {\cite[Theorem 9.4 \& Proposition 9.5]{Met_82}}] \label{lem:quasi_conv}
Let $(u_t)_{t\ge 1}$ be a nonnegative discrete-time stochastic process on a probability space $(\Omega,\calA,\bbP)$, i.e., $u_t\ge 0 \;a.s.$, for all $t\ge 1$. Let $\{\calF_t\}_{t\ge 1}$ be a filtration to which $(u_t)_{t\ge 1}$ is adapted. Define another binary stochastic process $(\delta_t)_{t\ge 1}$ as
\begin{equation}
\delta_t = \begin{cases}
1, & \mbox{if} \;\; \bbE[u_{t+1}-u_t|\calF_t] >0 \\
0, & \mbox{otherwise}
\end{cases}.
\end{equation}
If $\sum_{t=1}^\infty \bbE[\delta_t(u_{t+1}-u_t)]<\infty$, then $u_t\convas u$, where $u$ is integrable on $(\Omega,\calF,\bbP)$ and nonnegative a.s.. Furthermore, $(u_t)_{t\ge 1}$ is a quasi-martingale and
\begin{equation}
\sum_{t=1}^\infty \bbE[\abs{\bbE[u_{t+1}-u_t|\calF_t]}] < \infty.
\end{equation}
\end{lemma} 

\begin{lemma}[Donsker's Theorem; {\cite[Section 19.2]{Vaart_00}}]\label{lem:Donsker}
Let $X_1,\ldots,X_n$ be \iid generated from a distribution $\bbP$. Define the empirical distribution $\bbP_n \defeq \frac{1}{n}\sum_{i=1}^n\delta_{X_i}$. For a measurable function $f$, define $\bbP_n f$ and $\bbP f$ as the expectations of $f$ under the distributions $\bbP_n$ and $\bbP$ respectively. Define an empirical process $G_n(f) \defeq \sqrt{n}(\bbP_n f - \bbP f)$, $f\in\calF$, where $\calF$ is a class of measurable functions. $\calF$ is $\bbP$-Donsker if and only if the sequence of empirical processes $\{G_n\}_{n\ge 1}$ converges in distribution to a zero-mean Gaussian process $G$ tight in $\ell^\infty(\calF)$, where $\ell^\infty(\calF)$ is the space of all real-valued and bounded functionals defined on $\calF$ equipped with the uniform norm on $\calF$, denoted as $\norm{\cdot}_\calF$. 
Moreover, in such case, we have $\bbE\norm{G_n}_\calF\to\bbE\norm{G}_\calF$. 
\end{lemma}

\begin{lemma}[Glivenko-Cantelli theorem; {\cite[Section 19.2]{Vaart_00}}]\label{lem:GlivenkoCantelli}
Let $\bbP_n$ and $\bbP$ be the distributions defined as in Lemma~\ref{lem:Donsker}. A class of measurable functions $\calF$ is $\bbP$-Glivenko-Cantelli if and only if $\;\sup_{f\in\calF}\abs{\bbP_n f - \bbP f}\convas 0$. 
\end{lemma}

\begin{lemma}[A sufficient condition for $\bbP$-Glivenko-Cantelli and $\bbP$-Donsker classes; {\cite[Example 19.7]{Vaart_00}}]\label{lem:suff_Donsker}
Define a probability space $(\calX,\calA,\bbP)$. Let $\calF = \{f_\theta:\calX\to\bbR\,|\,\theta\in\Theta\}$ be a class of measurable functions, where 
$\Theta$ is a bounded subset in $\bbR^d$. If there exists a universal constant $K>0$ such that 
\begin{equation}
\abs{f_{\theta_1}(x) - f_{\theta_2}(x)} \le K\norm{\theta_1 - \theta_2}, \;\forall\,\theta_1,\theta_2\in\Theta, \;\forall\,x\in\calX,
\end{equation}
where $\norm{\cdot}$ is a general vector norm in $\bbR^d$, then $\calF$ is both $\bbP$-Glivenko-Cantelli and $\bbP$-Donsker. 
\end{lemma}

\begin{lemma}[{\cite[Lemma 8]{Mairal_10}}]\label{lem:conv_nonneg}
Let $(a_n)$, $(b_n)$ be two nonnegative sequences. Suppose $\sum_{n=1}^\infty a_n = \infty$ and $\sum_{n=1}^\infty a_nb_n <\infty$, and $\exists\,N\in\bbN$ and $K>0$ such that for all $n\ge N$, $|b_{n+1}-b_n|\le Ka_n$. Then $(b_n)$ converges and $\lim_{n\to\infty} b_n = 0$.
\end{lemma}

\section{Additional experiment results}\label{sec:exp_res}
This section consists of two parts. In the first part, we show the convergence speeds of all the online and batch algorithms on the (contaminated) CBCL face dataset for different values of the mini-batch size $\tau$, the latent dimension $K$, the penalty parameter $\rho$ in the ADMM-based algorithms, the step-size parameter $\kappa$ in the PGD-based algorithms and the (salt and pepper) noise density parameters $\nu$ and $\tnu$ in Figure~\ref{fig:canon2} to \ref{fig:tnu2}. As mentioned in Section~\ref{sec:conv_spd}, all the convergence results on the CBCL face dataset agree with those on the synthetic dataset. In the second part, we show the quality of the denoised images of all the algorithms on the CBCL face dataset for $K=25$ and $K=100$ in Table~\ref{tab:PSNRs} (a) and (b) respectively. The corresponding running times of all the algorithms for $K=25$ and $K=100$ are shown in Table~\ref{tab:times} (a) and (b) respectively. As stated in Section~\ref{sec:img_denois}, the results when $K=25$ and $K=100$ are similar to those when $K=49$. 

\begin{figure}[p!]\centering
\subfloat[]{\includegraphics[width=.4\columnwidth,height=.32\columnwidth]{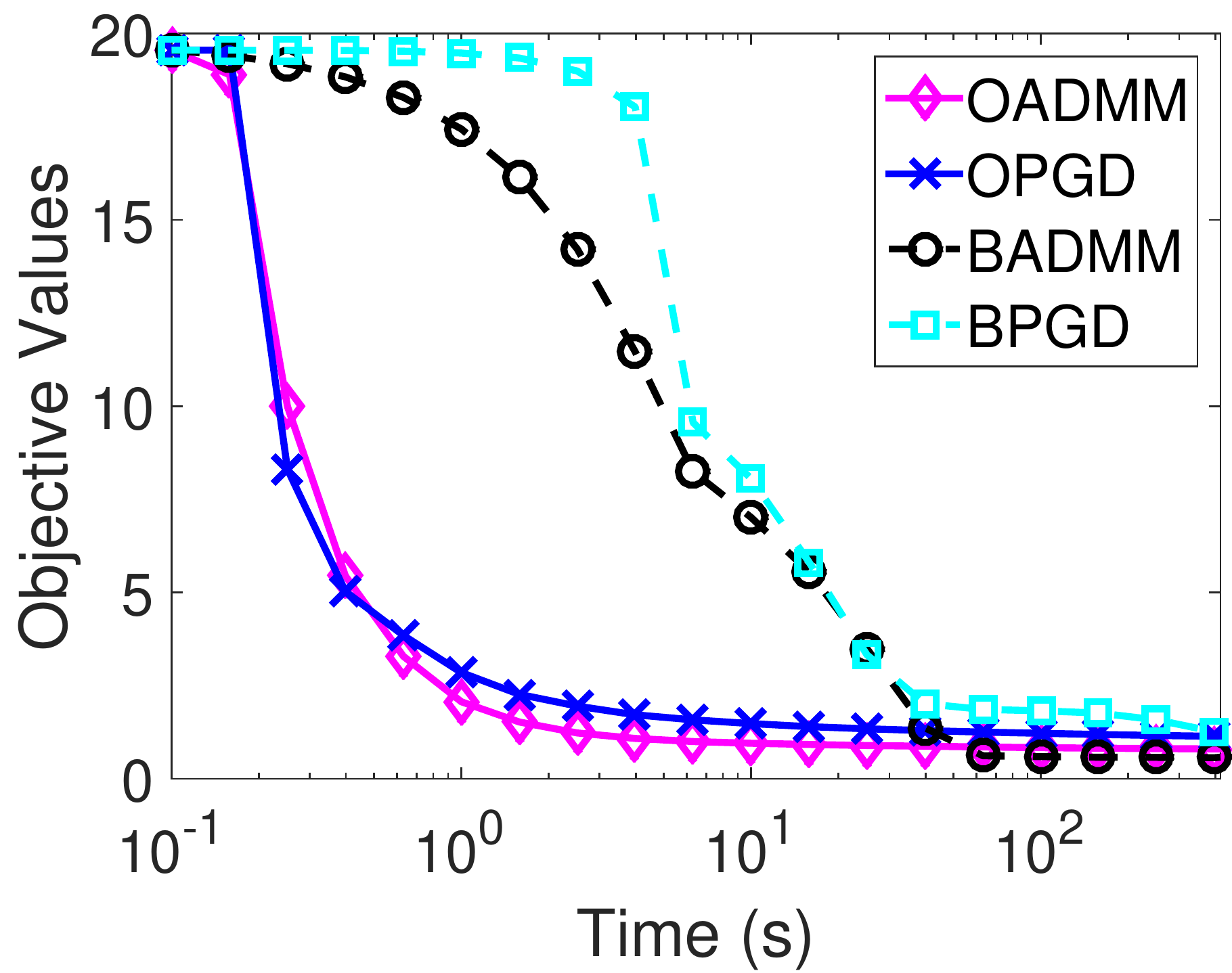}}\hspace{1cm}
\subfloat[]{\includegraphics[width=.4\columnwidth,height=.31\columnwidth]{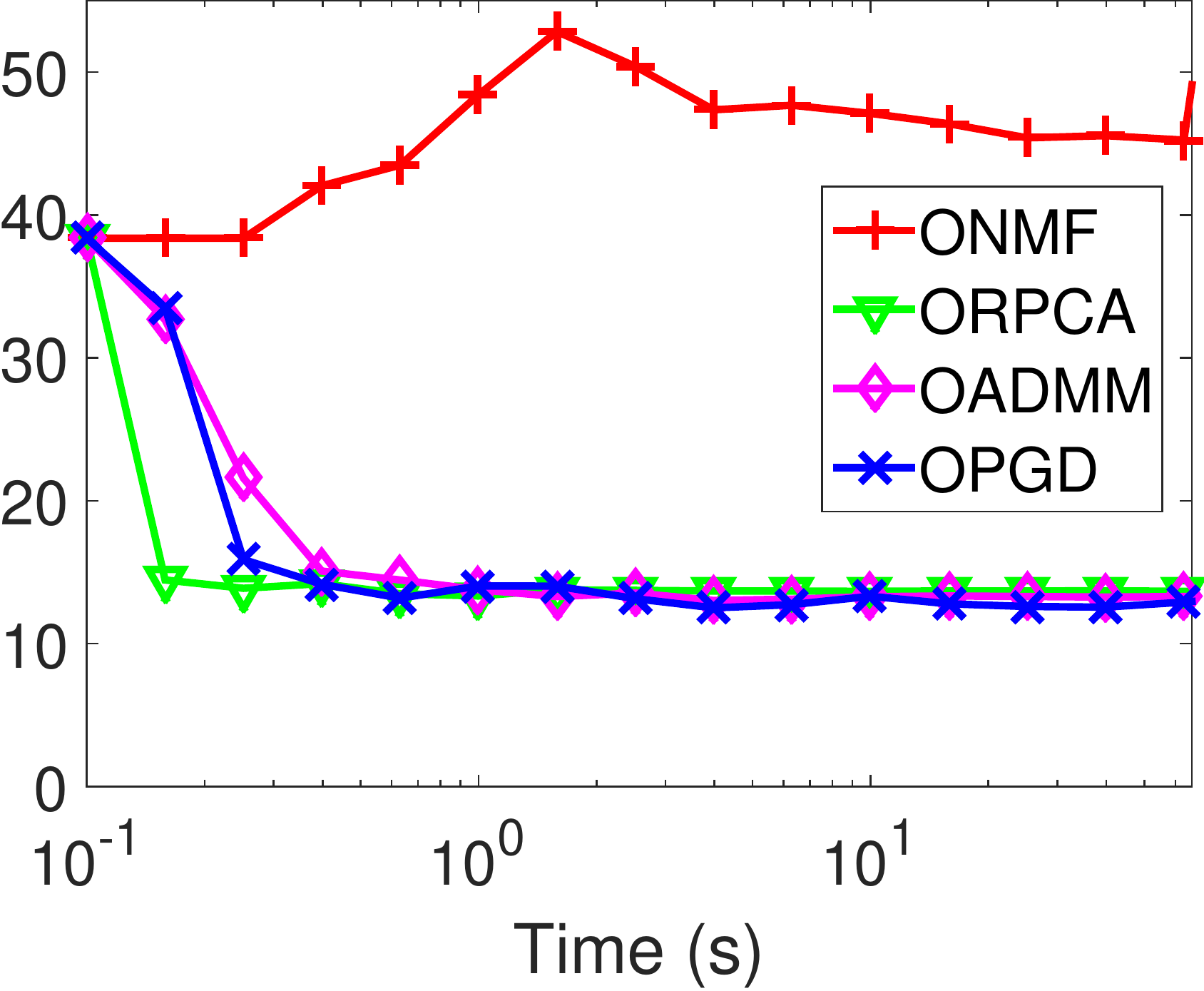}}
\caption{The objective values (as a function of time) of (a) our online algorithms and their batch counterparts (b) our online algorithms and other online algorithms on the CBCL face dataset. The parameters are set according to the canonical setting.}\label{fig:canon2}
\end{figure}

\begin{figure}[p!]\centering
\subfloat[OPGD]{\includegraphics[width=.4\columnwidth,height=.32\columnwidth]{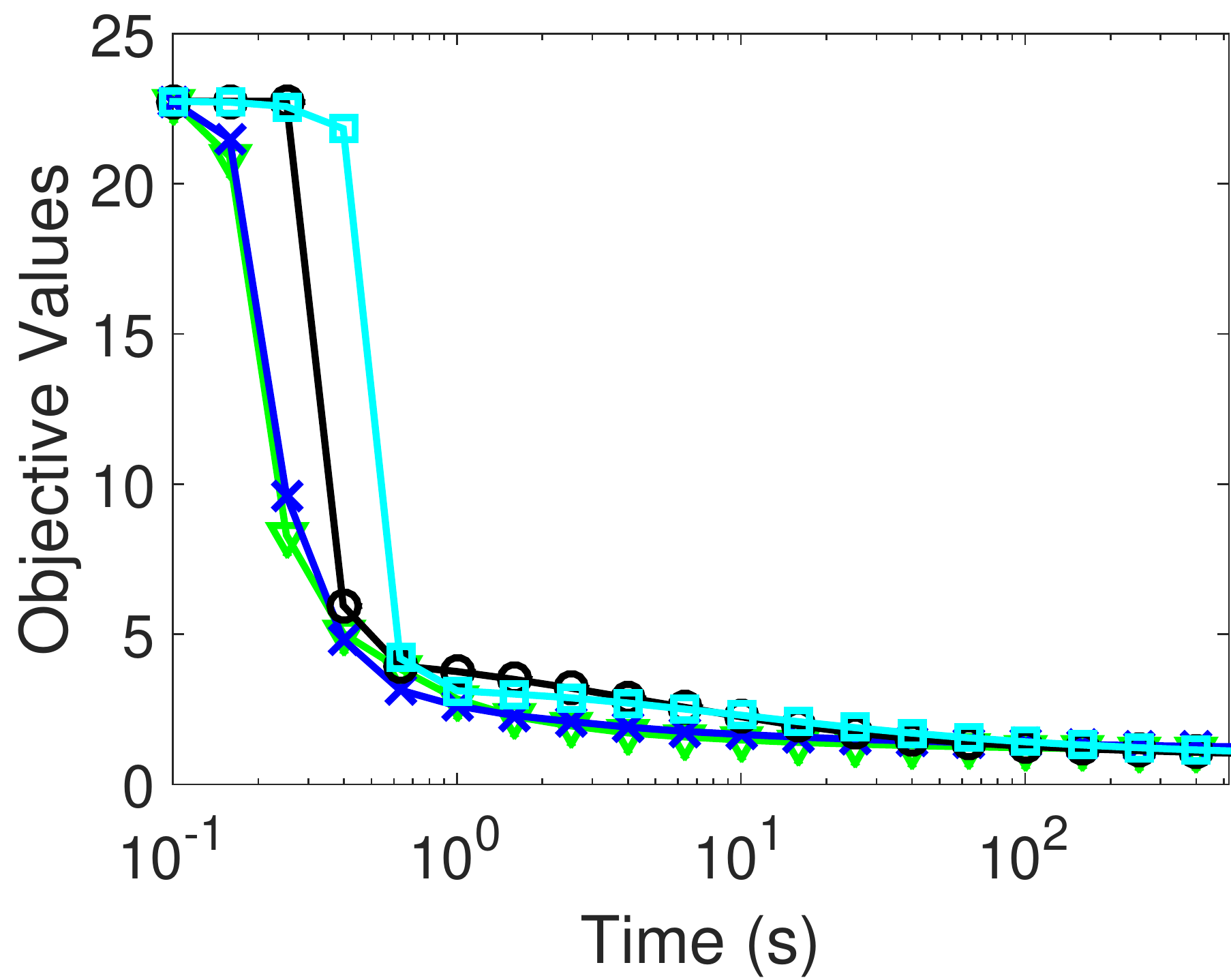}}\hspace{1cm}
\subfloat[OADMM]{\includegraphics[width=.4\columnwidth,height=.32\columnwidth]{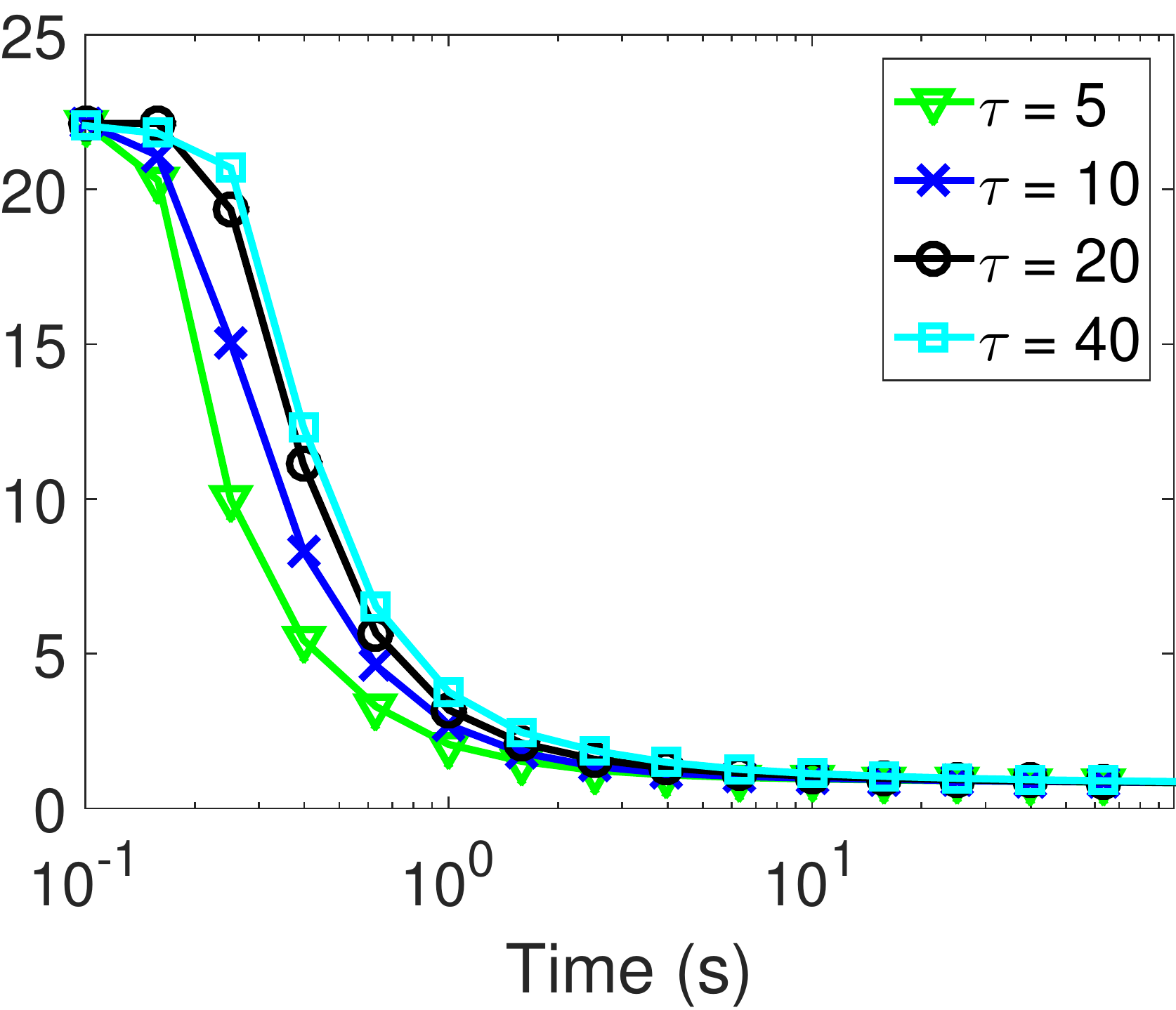}}
\caption{The objective values (as a function of time) of (a) OPGD and (b) OADMM for different values of $\tau$ on the CBCL face dataset. All the other parameters are set according to the canonical setting.}\label{fig:tau2}
\end{figure}

\begin{figure}[p!]\centering
\subfloat[ $K=25$ ]{\includegraphics[width=.4\columnwidth,height=.32\columnwidth]{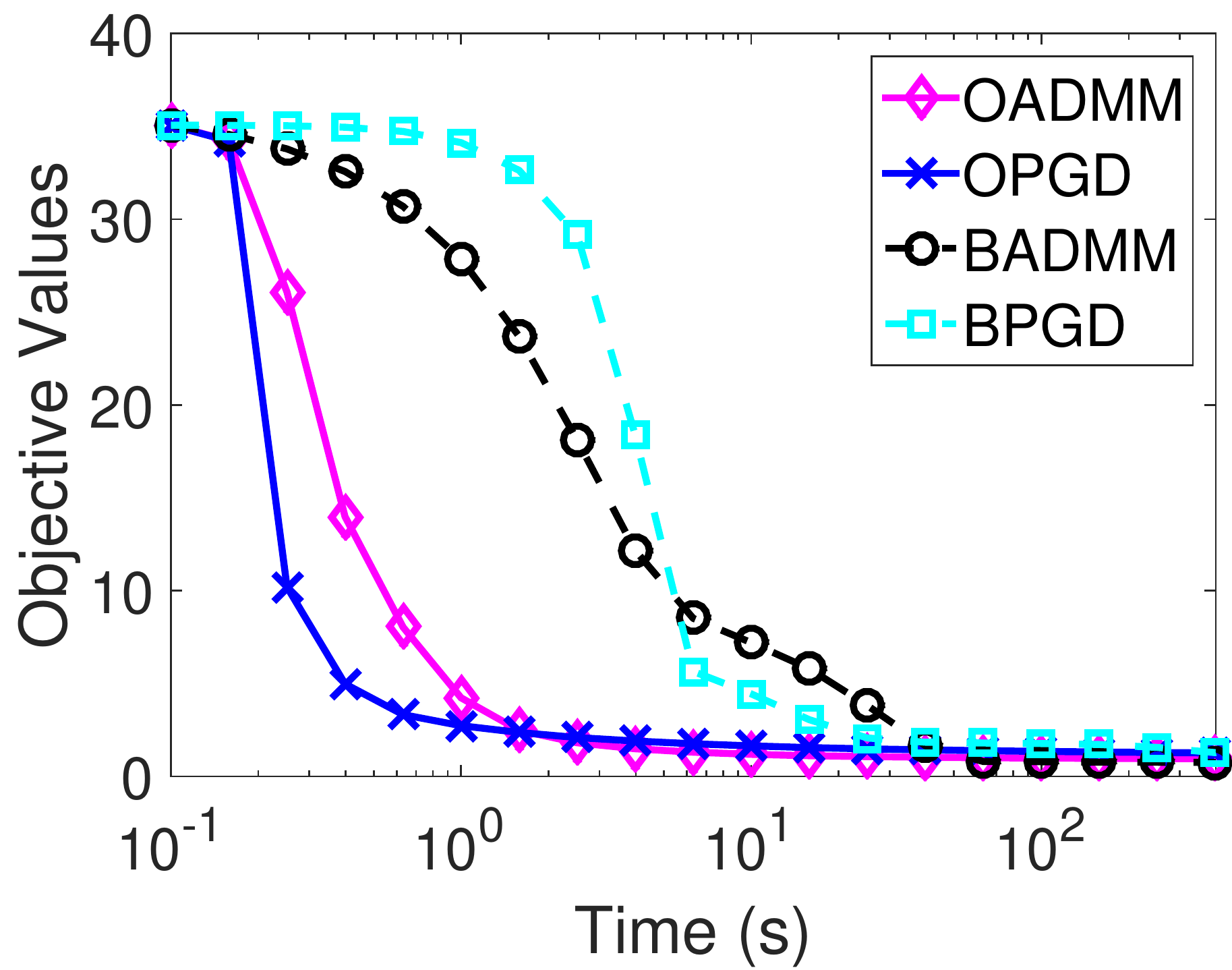}}\hspace{1cm}
\subfloat[$K=25$ ]{\includegraphics[width=.4\columnwidth,height=.31\columnwidth]{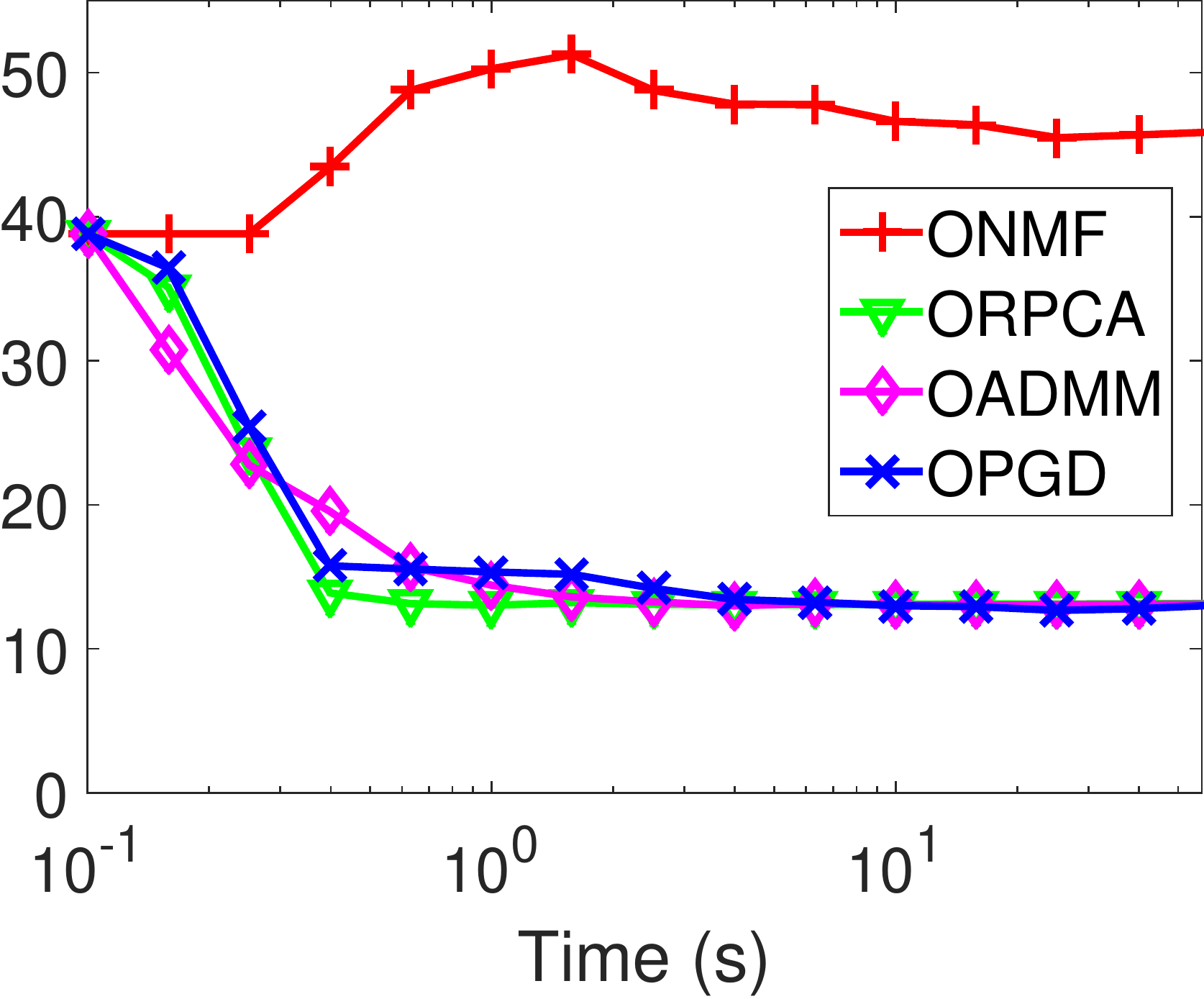}}\\
\subfloat[$K=100$ ]{\includegraphics[width=.4\columnwidth,height=.32\columnwidth]{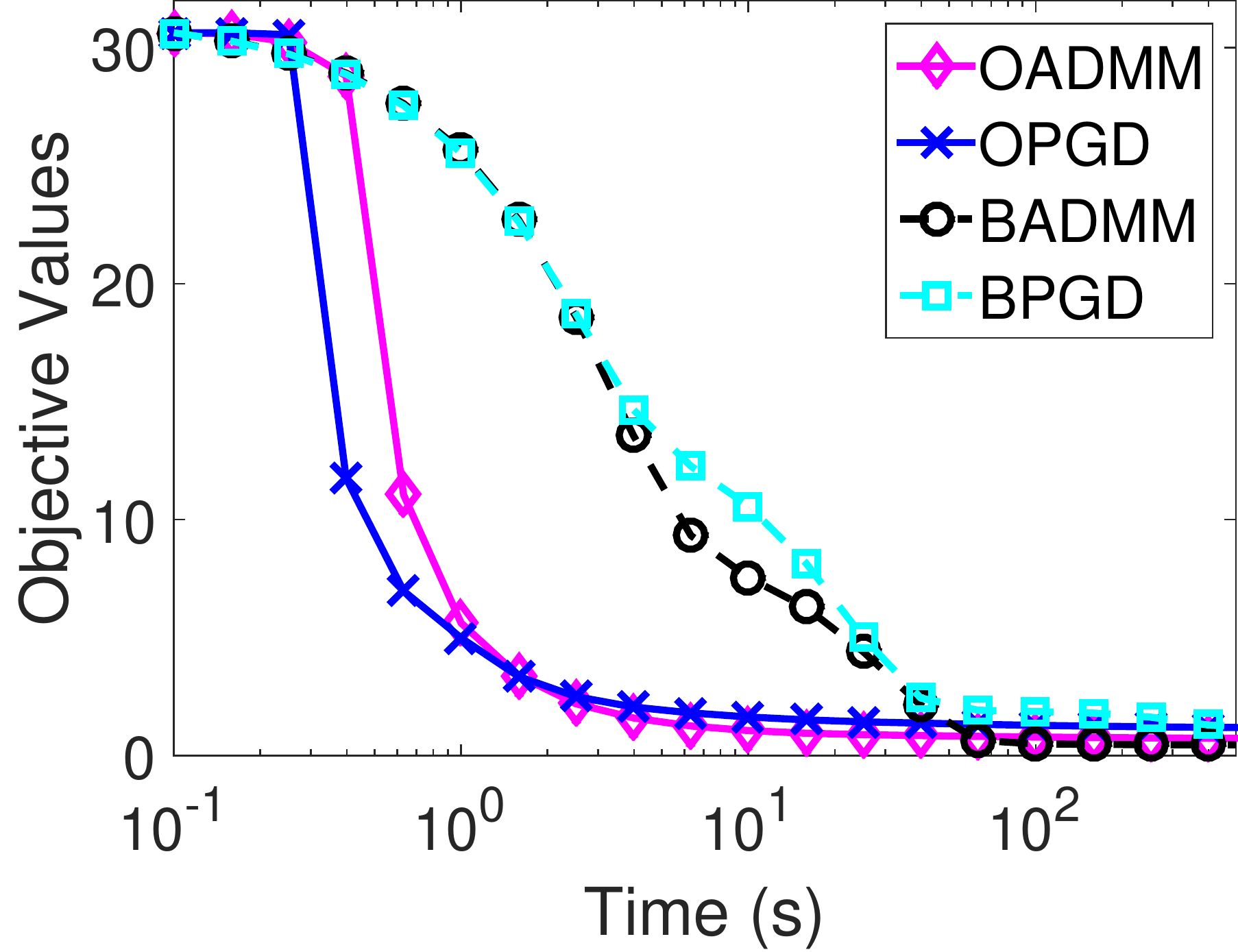}}\hspace{1cm}
\subfloat[$K=100$ ]{\includegraphics[width=.4\columnwidth,height=.32\columnwidth]{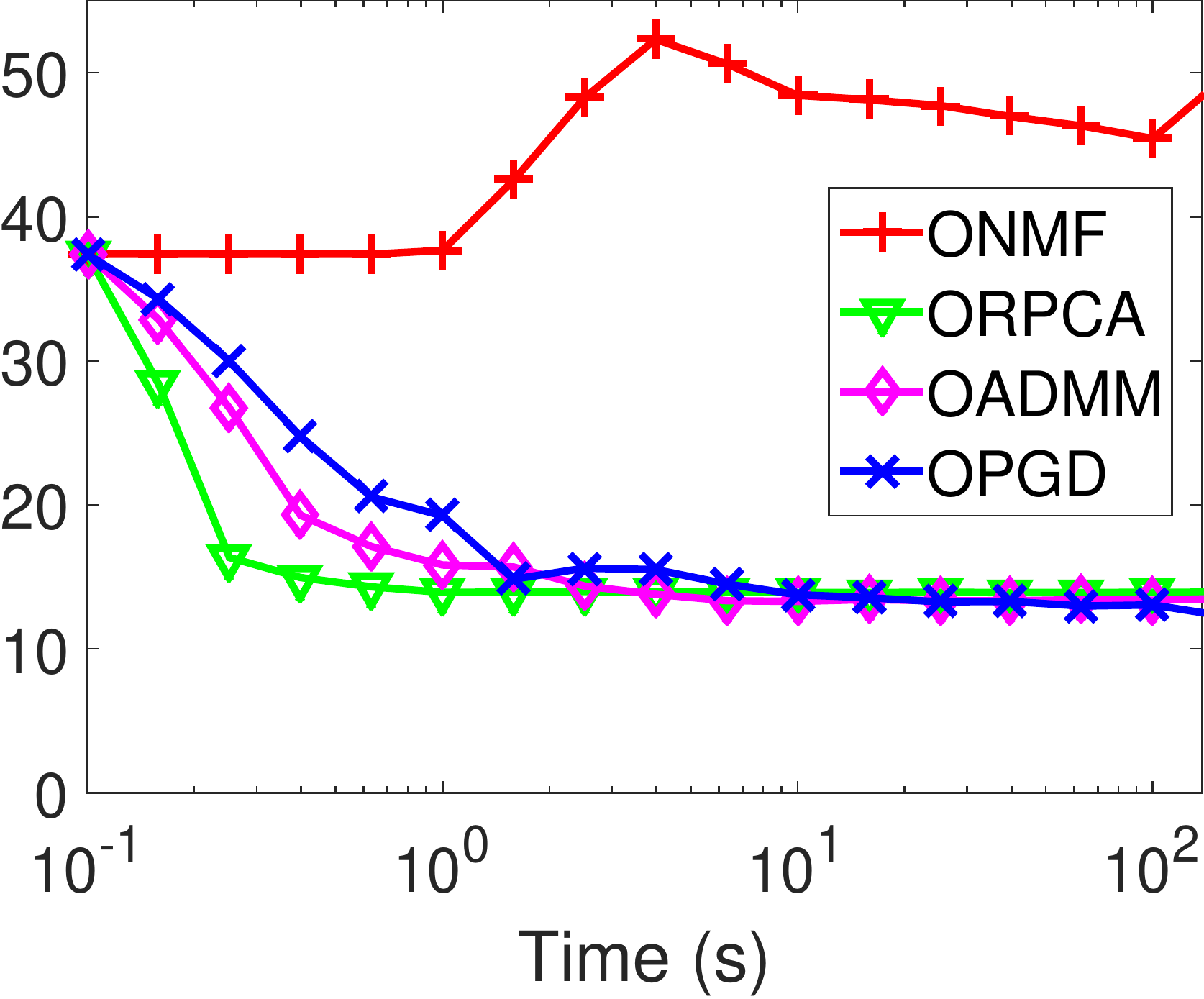}}
\caption{The objective values (as a function of time) of all the algorithms for different values of $K$ on the CBCL face dataset. In (a) and (b), $K=25$. In (c) and (d), $K=100$. All the other parameters are set according to the canonical setting.}\label{fig:K2}
\end{figure}

\begin{figure}[p!]\centering
\subfloat[BADMM]{\includegraphics[width=.4\columnwidth,height=.31\columnwidth]{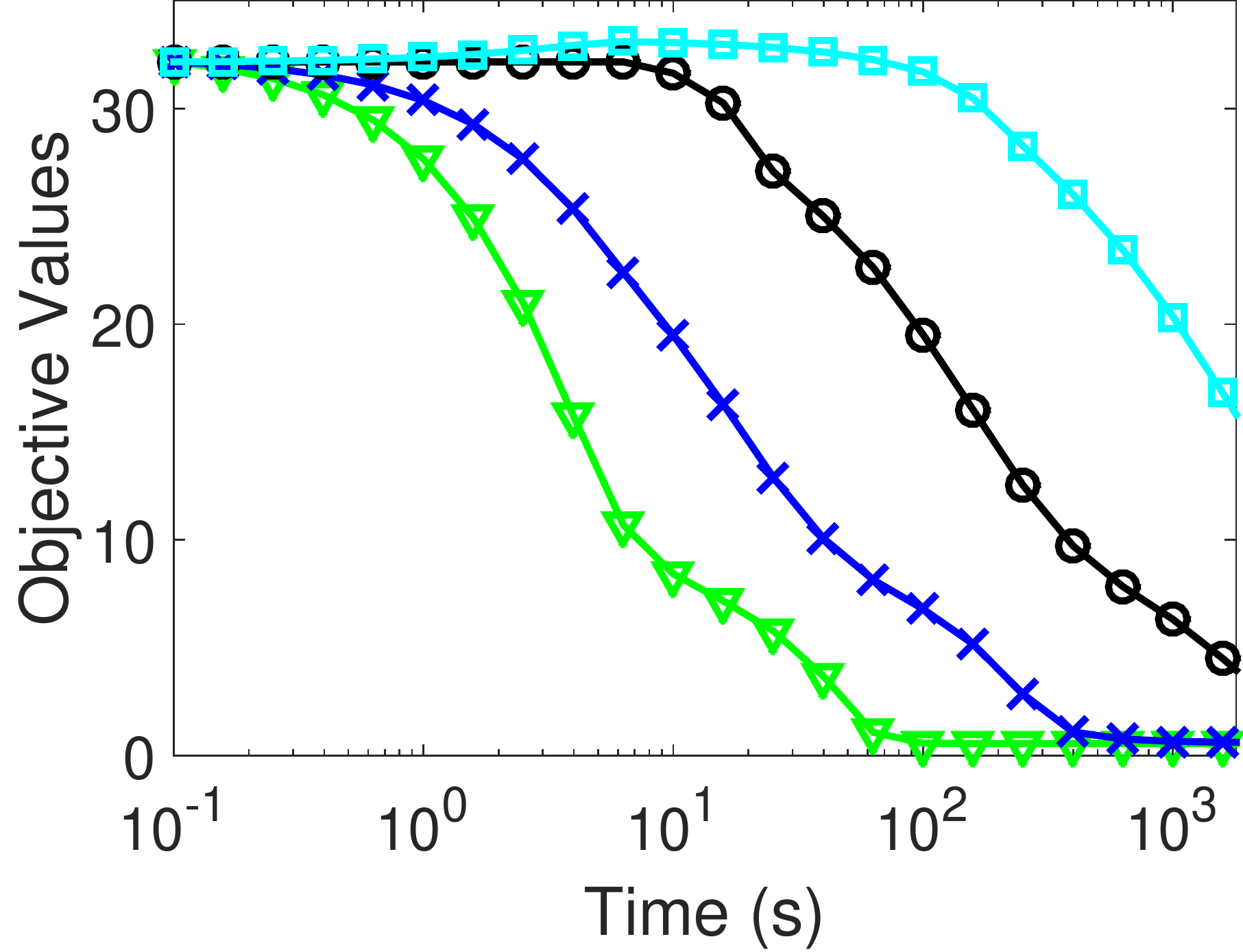}}\hspace{1cm}
\subfloat[OADMM]{\includegraphics[width=.4\columnwidth,height=.32\columnwidth]{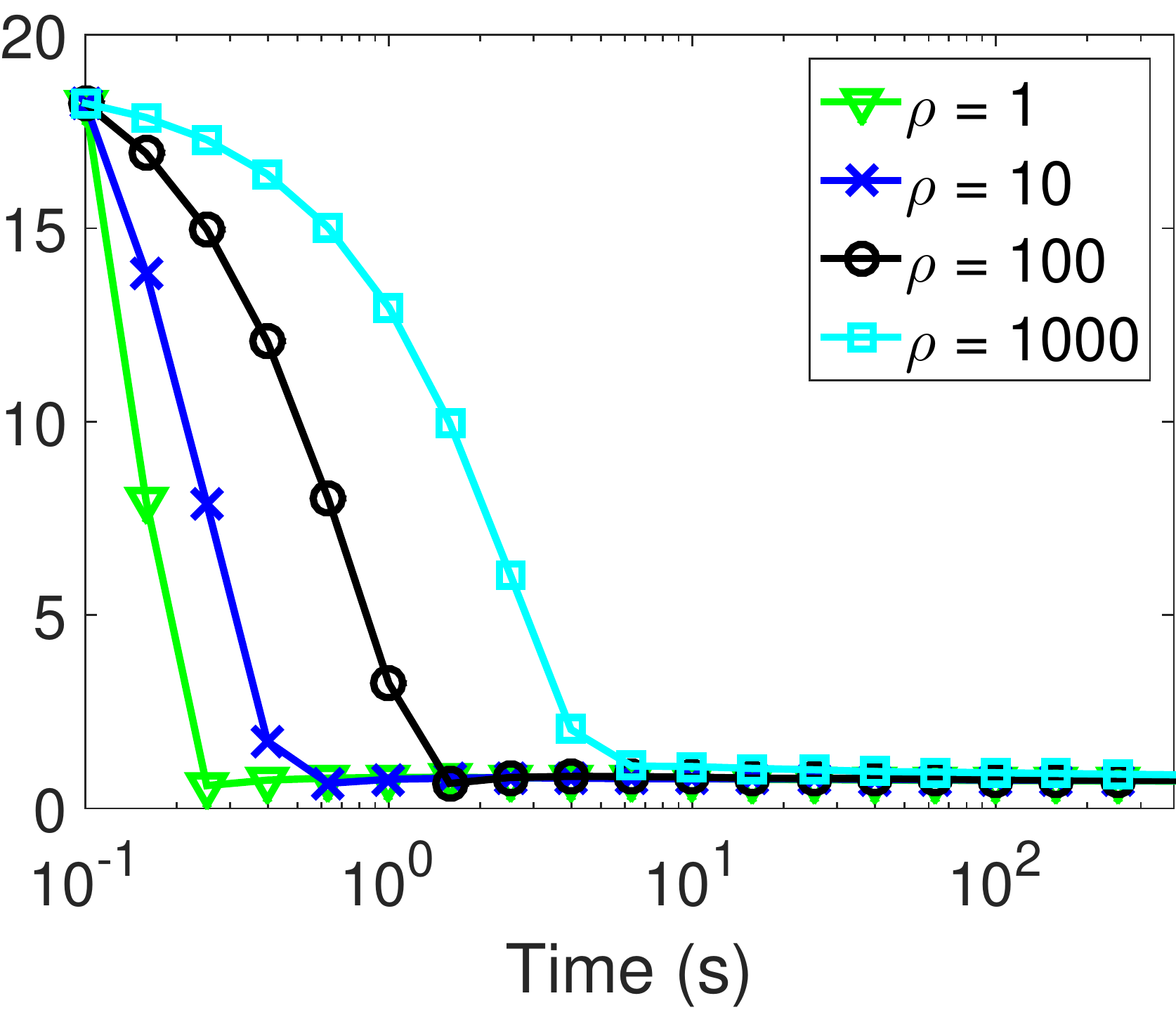}}
\caption{The objective values (as a function of time) of (a) BADMM and (b) OADMM for different values of $\rho$ on the CBCL face dataset. All the other parameters are set according to the canonical setting.}\label{fig:rho2}
\end{figure}
 
\begin{figure}[p!]\centering
\subfloat[BPGD]{\includegraphics[width=.4\columnwidth,height=.32\columnwidth]{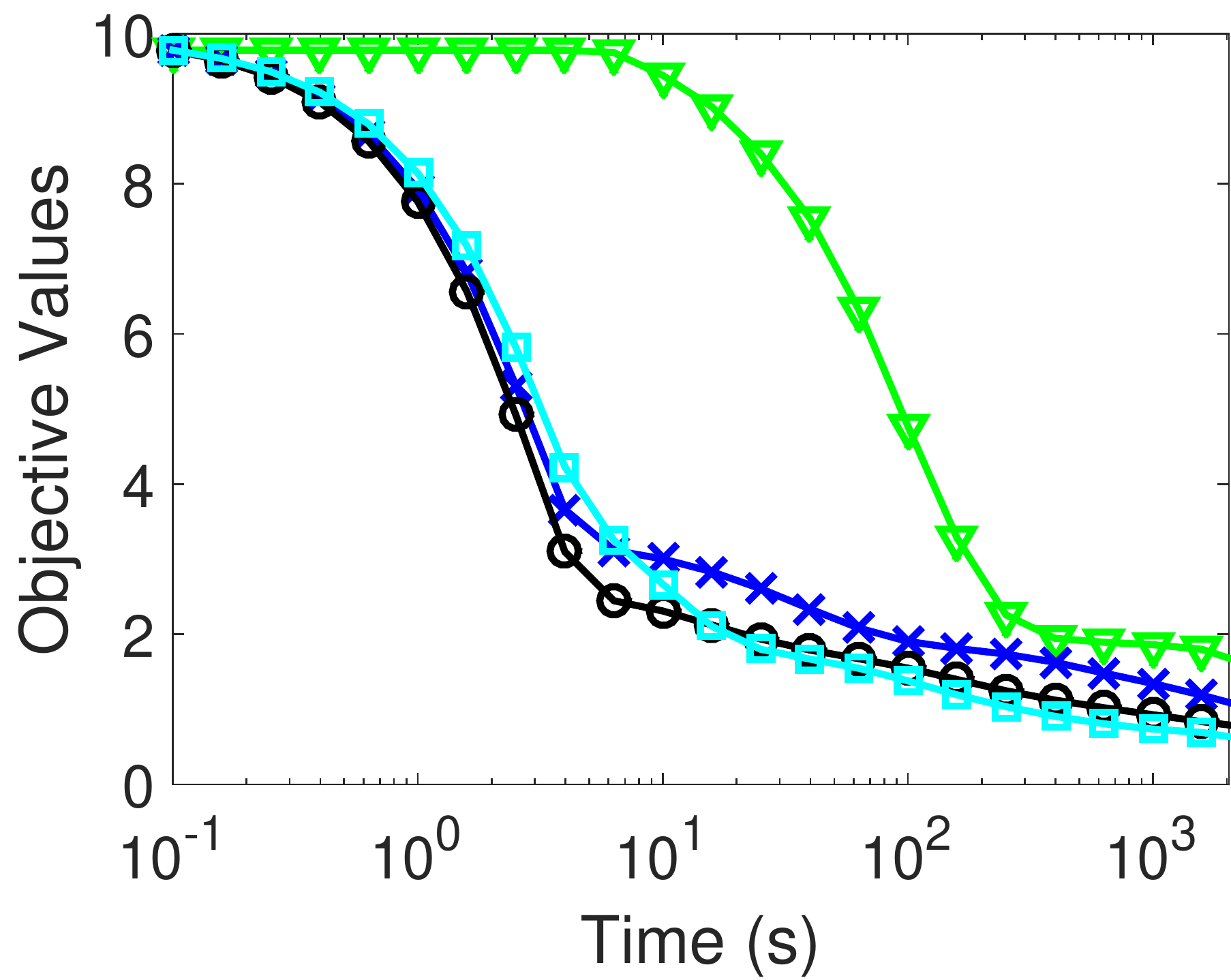}}\hspace{1cm}
\subfloat[OPGD]{\includegraphics[width=.4\columnwidth,height=.32\columnwidth]{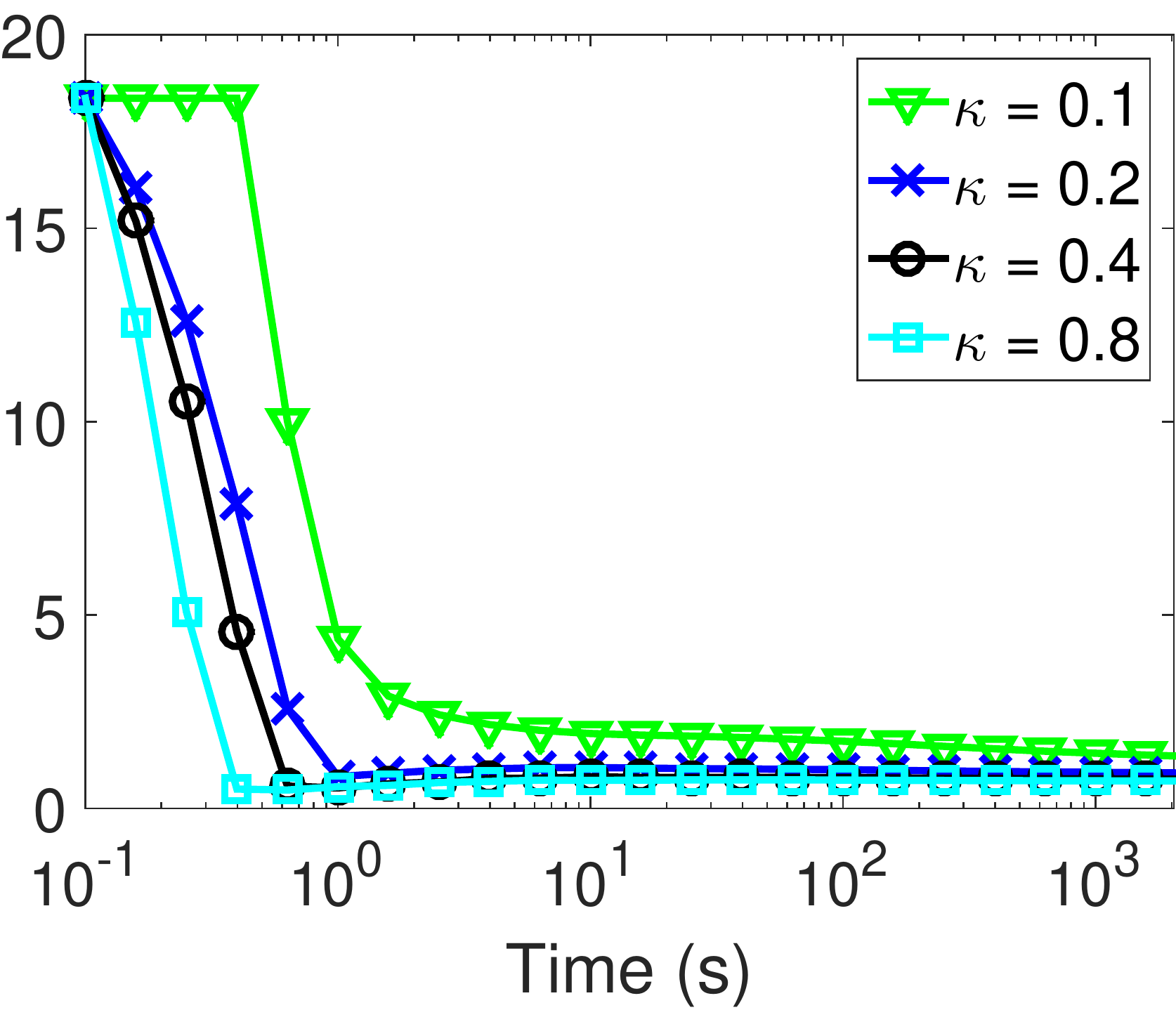}}
\caption{The objective values (as a function of time) of (a) BPGD and (b) OPGD for different values of $\kappa$ on the CBCL face dataset. All the other parameters are set according to the canonical setting.}\label{fig:kappa2}
\end{figure}

\begin{figure}[p!]\centering
\subfloat[$\nu = 0.8$, $\tnu = 0.2$]{\includegraphics[width=.4\columnwidth,height=.32\columnwidth]{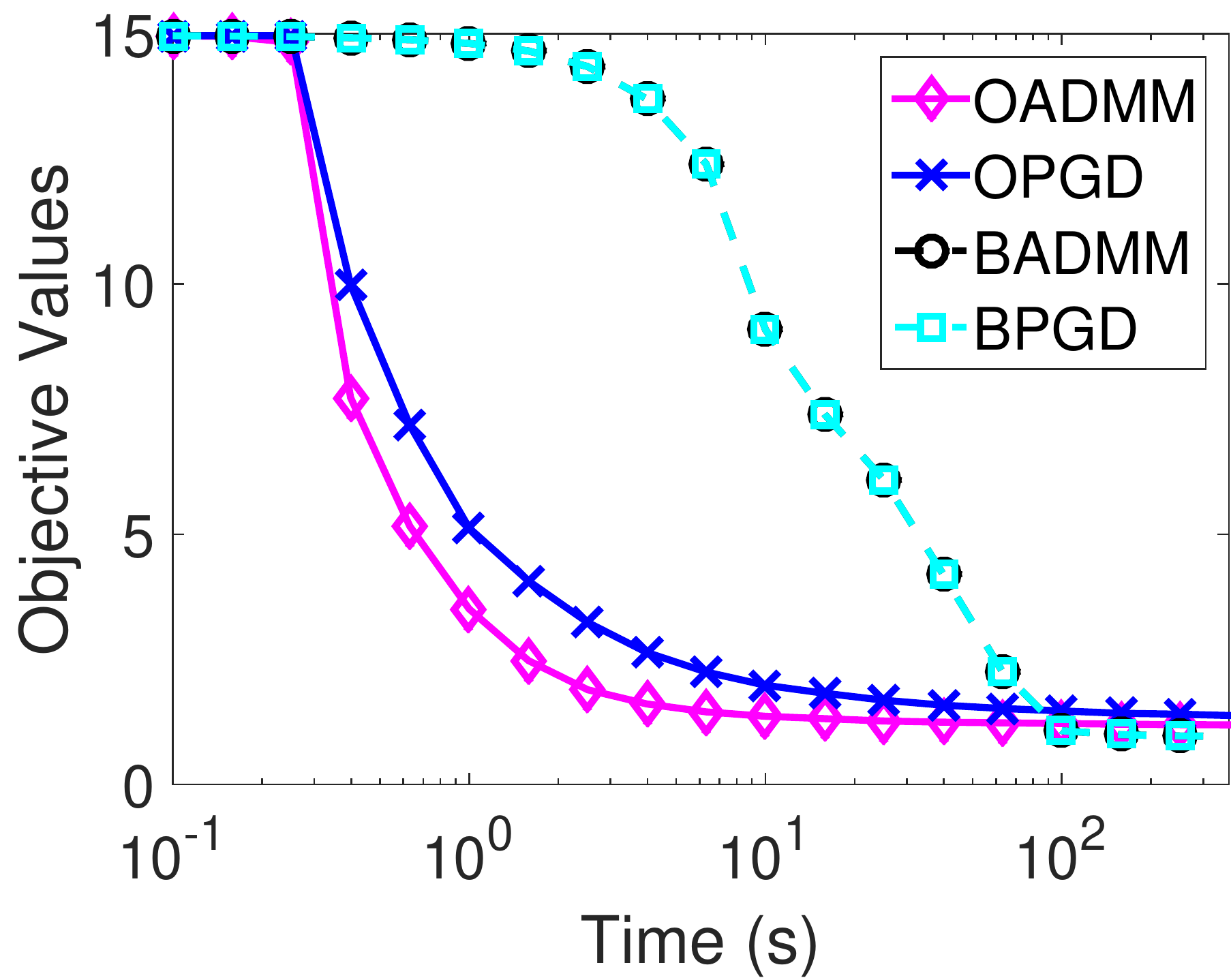}}\hspace{1cm}
\subfloat[$\nu = 0.9$, $\tnu = 0.3$]{\includegraphics[width=.4\columnwidth,height=.32\columnwidth]{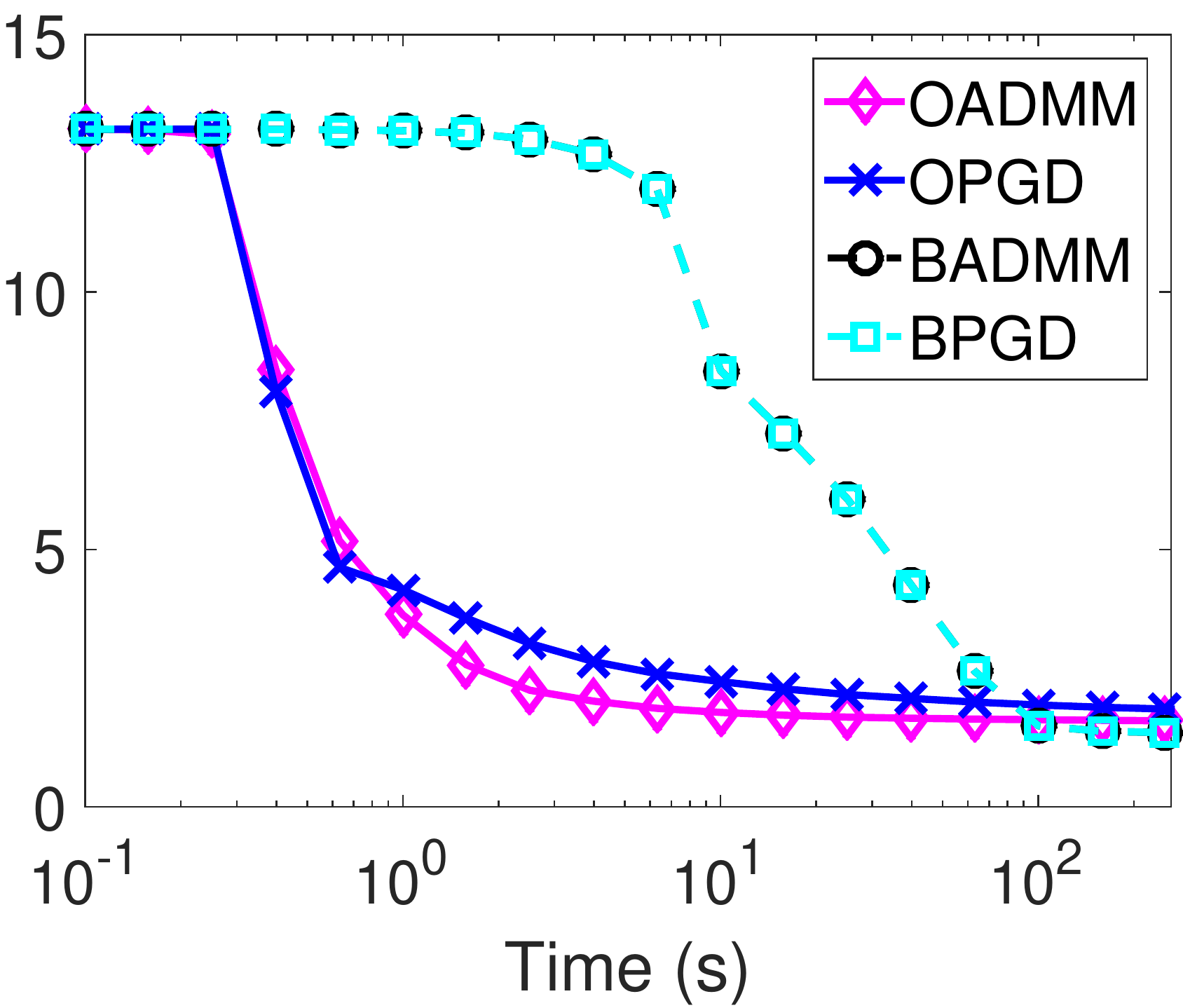}}
\caption{The objective values (as a function of time) of our online and batch algorithms on the CBCL face dataset with a larger proportion of outliers.}\label{fig:tnu2}
\end{figure}

\begin{table}[p!]\centering
\caption{PSNRs (in dB) of all the algorithms on the CBCL face dataset with different noise density.}\label{tab:PSNRs}
\begin{subfloat}[$K=25$]{
\begin{tabular}{|c|c|c|c|c|c|c|c|c|c|}
\hline
\quad&Setting 1 & Setting 2 & Setting 3\\\hline
OADMM &  11.39 $\pm$ 0.16 &  11.37 $\pm$ 0.12 &  11.35 $\pm$ 0.18\\ \hline
OPGD &  11.49 $\pm$ 0.05 &  11.43 $\pm$ 0.09 &  11.38 $\pm$ 0.06\\ \hline
BADMM &  11.51 $\pm$ 0.19 &  11.46 $\pm$ 0.07 &  11.41 $\pm$ 0.15\\ \hline
BPGD &  11.54 $\pm$ 0.07 &  11.46 $\pm$ 0.17 &  11.42 $\pm$ 0.15\\ \hline
ONMF &   5.99 $\pm$ 0.04 &   5.97 $\pm$ 0.12 &   5.97 $\pm$ 0.08\\ \hline
ORPCA &  11.26 $\pm$ 0.05 &  11.24 $\pm$ 0.11 &  11.22 $\pm$ 0.11\\ \hline
\end{tabular}}
\end{subfloat}\hspace{1cm}
\begin{subfloat}[$K=100$]{
\begin{tabular}{|c|c|c|c|c|c|c|c|c|c|}
\hline
\quad&Setting 1 & Setting 2 & Setting 3\\\hline
OADMM &  11.39 $\pm$ 0.02 &  11.35 $\pm$ 0.11 &  11.35 $\pm$ 0.06\\ \hline
OPGD &  11.51 $\pm$ 0.01 &  11.45 $\pm$ 0.09 &  11.44 $\pm$ 0.11\\ \hline
BADMM &  11.51 $\pm$ 0.11 &  11.47 $\pm$ 0.00 &  11.45 $\pm$ 0.03\\ \hline
BPGD &  11.52 $\pm$ 0.16 &  11.46 $\pm$ 0.07 &  11.45 $\pm$ 0.12\\ \hline
ONMF &   5.99 $\pm$ 0.19 &   5.97 $\pm$ 0.03 &   5.95 $\pm$ 0.05\\ \hline
ORPCA &  11.26 $\pm$ 0.03 &  11.24 $\pm$ 0.16 &  11.20 $\pm$ 0.13\\ \hline
\end{tabular}
}
\end{subfloat}
\end{table}

\begin{table}[p!]\centering
\caption{Running times (in seconds) of all the algorithms on the CBCL face dataset with different noise density.}\label{tab:times}
\begin{subfloat}[$K=25$]{
\begin{tabular}{|c|c|c|c|c|c|c|c|c|c|}
\hline
\quad&Setting 1 & Setting 2 & Setting 3\\\hline
OADMM & 420.58 $\pm$ 2.59 & 427.66 $\pm$ 4.80 & 430.02 $\pm$ 2.93\\ \hline
OPGD & 431.66 $\pm$ 2.49 & 455.15 $\pm$ 1.70 & 463.67 $\pm$ 1.12\\ \hline
BADMM & 1009.45 $\pm$ 11.27 & 1184.29 $\pm$ 10.49 & 1240.91 $\pm$ 8.21\\ \hline
BPGD & 1125.58 $\pm$ 12.83 & 1185.64 $\pm$ 13.36 & 1279.07 $\pm$ 9.08\\ \hline
ONMF & 2384.70 $\pm$ 9.59 & 2588.29 $\pm$ 14.39 & 2698.57 $\pm$ 10.24\\ \hline
ORPCA & 365.98 $\pm$ 5.29 & 382.49 $\pm$ 4.20 & 393.10 $\pm$ 4.27\\ \hline
\end{tabular}}
\end{subfloat}\hspace{1cm}
\begin{subfloat}[$K=100$]{
\begin{tabular}{|c|c|c|c|c|c|c|c|c|c|}
\hline
\quad&Setting 1 & Setting 2 & Setting 3\\\hline
OADMM & 422.97 $\pm$ 2.38 & 424.67 $\pm$ 2.65 & 434.17 $\pm$ 4.67\\ \hline
OPGD & 430.19 $\pm$ 2.27 & 448.72 $\pm$ 3.90 & 454.30 $\pm$ 4.65\\ \hline
BADMM & 1009.04 $\pm$ 8.53 & 1187.58 $\pm$ 5.06 & 1250.53 $\pm$ 4.67\\ \hline
BPGD & 1131.89 $\pm$ 7.04 & 1192.46 $\pm$ 7.43 & 1280.55 $\pm$ 7.93\\ \hline
ONMF & 2379.86 $\pm$ 15.18 & 2591.09 $\pm$ 11.91 & 2693.08 $\pm$ 12.48\\ \hline
ORPCA & 363.37 $\pm$ 3.01 & 390.58 $\pm$ 5.31 & 401.21 $\pm$ 3.27\\ \hline
\end{tabular}
}
\end{subfloat}
\end{table}

\end{document}